\documentclass[11pt]{article}

\usepackage[numbers,compress]{natbib}

\usepackage{fullpage}
\usepackage[utf8]{inputenc} 
\usepackage[T1]{fontenc}    
\usepackage{hyperref}       
\usepackage{url}            
\usepackage{booktabs}       
\usepackage{amsfonts}       
\usepackage{nicefrac}       
\usepackage{microtype}      
\usepackage{color}
\usepackage{xcolor}
\usepackage{amsmath,amsthm,amssymb,url}
\usepackage{mathtools}
\usepackage{algorithm,algcompatible}
\usepackage{enumerate}

\usepackage[labelfont=bf]{subcaption}
\usepackage[labelfont=bf]{caption}
\def \basefigwidth{0.49}
\def \basethreefigwidth{0.32}

\newtheorem{theorem}{Theorem}
\numberwithin{theorem}{section}
\newtheorem{lemma}[theorem]{Lemma}
\newtheorem{proposition}[theorem]{Proposition}
\newtheorem{coro}[theorem]{Corollary}
\newtheorem{assumption}[theorem]{Assumption}

\numberwithin{equation}{section}
\newcommand{\minimize}{\mbox{minimize}}

\newcommand{\tf}[1]{\mathbf{#1}}
\newcommand{\Ah}{\widehat{A}}
\newcommand{\Ahat}{\Ah}
\newcommand{\Bhat}{\Bh}
\newcommand{\Bh}{\widehat{B}}

\newcommand{\Phixh}{\hat{\tf \Phi}_x}
\newcommand{\Phiuh}{\hat{\tf \Phi}_u}
\newcommand{\Dh}{\widehat{\tf{\Delta}}}
\newcommand{\trueA}{A_\star}
\newcommand{\trueB}{B_\star}
\newcommand{\trueK}{K_\star}
\newcommand{\const}{\calO(1)}
\newcommand{\RHinf}{\mathcal{RH}_\infty}
\newcommand{\Avg}{\mathsf{Avg}}

\DeclareMathOperator*{\Tr}{\mathbf{Tr}}

\newcommand{\norm}[1]{\lVert #1 \rVert}
\newcommand{\bignorm}[1]{\left\lVert #1 \right\rVert}
\newcommand{\twonorm}[1]{\lVert #1 \rVert_{2}}

\newcommand{\E}{\mathbb{E}}

\renewcommand{\Pr}{\mathbb{P}}
\newcommand{\T}{\top}
\newcommand{\Res}[1]{\mathfrak{R}_{#1}}
\newcommand{\res}[1]{\mathfrak{R}_{#1}}

\newcommand{\statedim}{n}
\newcommand{\inputdim}{p}
\newcommand{\hinf}{\mathcal{H}_\infty}
\newcommand{\htwo}{\mathcal{H}_2}

\newcommand{\ltwonorm}[1]{\| #1 \|_2}
\newcommand{\infnorm}[1]{\| #1 \|_\infty}
\newcommand{\hinfnorm}[1]{\| #1 \|_{\hinf}}
\newcommand{\htwonorm}[1]{\| #1 \|_{\htwo}}
\newcommand{\lonenorm}[1]{\| #1 \|_{\mathcal{L}_1}}

\newcommand{\bightwonorm}[1]{\bignorm{#1}_{\mathcal{H}_2}}
\newcommand{\bighinfnorm}[1]{\bignorm{#1}_{\mathcal{H}_\infty}}

\newcommand{\iid}{\stackrel{\mathclap{\text{\scriptsize{ \tiny i.i.d.}}}}{\sim}}

\newcommand{\cvectwo}[2]{\begin{bmatrix} #1 \\ #2 \end{bmatrix}}
\newcommand{\rvectwo}[2]{\begin{bmatrix} #1 & #2 \end{bmatrix}}
\newcommand{\bmattwo}[4]{\begin{bmatrix} #1 & #2 \\ #3 & #4 \end{bmatrix}}

\newcommand{\bvct}[1]{\mathbf{#1}}
\newcommand{\bmtx}[1]{\mathbf{#1}}

\newcommand{\R}{\mathbb{R}}
\newcommand{\C}{\mathbb{C}}

\newcommand{\calE}{\mathcal{E}}
\newcommand{\calN}{\mathcal{N}}

\newcommand{\calO}{\mathcal{O}}
\newcommand{\calF}{\mathcal{F}}
\newcommand{\calR}{\mathcal{R}}
\newcommand{\calQ}{\mathcal{Q}}
\newcommand{\calS}{\mathcal{S}}

\newcommand{\Otilde}{\widetilde{\calO}}
\newcommand{\logg}[1]{\log\left(#1\right)}
\newcommand{\Omegatilde}{\widetilde{\Omega}}

\title{Regret Bounds for Robust Adaptive Control of the Linear Quadratic Regulator}
\author{Sarah Dean, Horia Mania, Nikolai Matni, Benjamin Recht, and Stephen Tu
\vspace{0.0625in}
\\
University of California, Berkeley}

\date{\today}

\begin{document}

\maketitle


\begin{abstract}
We consider adaptive control of the Linear Quadratic Regulator (LQR), where an unknown linear system is controlled subject to quadratic costs.
Leveraging recent developments in the estimation of linear systems and in robust controller synthesis,
we present the first provably polynomial time algorithm that provides high probability guarantees of sub-linear regret on this problem.
We further study the interplay between regret minimization and parameter estimation by proving a lower bound {on the expected regret in terms of the exploration schedule used by any algorithm.}
Finally, we conduct a numerical study comparing our robust adaptive algorithm to other methods from the adaptive LQR literature,
and demonstrate the flexibility of our proposed method by extending it to a demand forecasting problem subject to state constraints.
\end{abstract}


\section{Introduction}
 
The problem of adaptively controlling an unknown dynamical system has a rich history, with classical asymptotic results of convergence and stability dating back decades \cite{ioannou1996robust,krstic1995nonlinear}.  Of late, there has been a renewed interest in the study of a particular instance of such problems, namely the adaptive Linear Quadratic Regulator (LQR), with an emphasis on \emph{non-asymptotic} guarantees of stability and performance.  Initiated by {\citet{abbasi2011regret}}, there have since been several works analyzing the regret suffered by various adaptive algorithms on LQR-- here the regret incurred by an algorithm is thought of as a measure of deviations in performance from optimality over time.  These results can be broadly divided into two categories: those providing high-probability guarantees for a single execution of the algorithm~\cite{abbasi2011regret,abeille17,faradonbeh17b,ibrahimi12}, and those providing bounds on the expected \emph{Bayesian} regret incurred over a family of possible systems~\cite{abbasi15,ouyang17}.  As we discuss in more detail, these methods all suffer from one or several of the following limitations: restrictive and unverifiable assumptions, limited applicability, and computationally intractable subroutines.  In this paper, we provide, to the best of our knowledge, the first polynomial-time algorithm for the adaptive LQR problem that provides high probability guarantees of sub-linear regret, and that does not require unverifiable or unrealistic assumptions.

\paragraph{Related Work.}
There is a rich body of work on the estimation of linear systems as well as on the robust and adaptive control of unknown systems.
We target our discussion to works on non-asymptotic guarantees for the LQR control of an unknown system, 
broadly divided into three categories.

\emph{Offline estimation and control synthesis:}  
In a non-adaptive setting, i.e., when system identification can be done offline prior to controller synthesis and implementation, the first work to provide end-to-end guarantees for the LQR optimal control problem is that of \citet{fiechter1997pac}, who shows that the \emph{discounted} LQR
problem is PAC-learnable. Dean et al~\citet{dean17} improve on this result, and provide the first end-to-end sample complexity guarantees for the infinite horizon average cost LQR problem.

\emph{Optimism in the Face of Uncertainty (OFU):}  
{\citet{abbasi2011regret}, \citet{faradonbeh17b}, and \citet{ibrahimi12} employ} the \emph{Optimism in the Face of Uncertainty} (OFU) principle \cite{bittanti06}, which optimistically selects model parameters from {a} confidence set by choosing those that lead to the \emph{best} closed-loop (infinite horizon) control
performance, and then plays the corresponding optimal controller, repeating this process online as the confidence set shrinks.  While OFU in the LQR setting has been shown to achieve optimal {regret $\Otilde(\sqrt{T})$}, its implementation requires
solving a non-convex optimization problem to precision $\Otilde(T^{-1/2})$, for which no provably 
efficient implementation exists.  

\emph{Thompson Sampling (TS):} 
 {To circumvent the} computational roadblock of OFU, {recent works} replace the intractable OFU subroutine with a random draw from the model uncertainty set, resulting in \emph{Thompson Sampling} (TS) based policies \cite{abbasi15,abeille17,ouyang17}. {\citet{abeille17} show that such a method achieves $\Otilde(T^{2/3})$ regret with high-probability for scalar systems.} However, {their proof} does not extend to the non-scalar setting. {\citet{abbasi15} and \citet{ouyang17}} consider expected regret in a Bayesian setting, and {provide TS methods which achieve $\Otilde(\sqrt{T})$ regret.}  Although not directly comparable to our result, {we remark on the computational challenges of these algorithms.}  Whereas the proof {of \citet{abbasi15}} was shown to be incorrect~\cite{osband16}, \citet{ouyang17} make the restrictive assumption that there exists a (known) initial compact set $\Theta$ describing the uncertainty in the system parameters, such that for any system $\theta_1 \in \Theta$, the optimal controller $K(\theta_1)$ is stabilizing when applied to any other system $\theta_2 \in \Theta$.  No means of constructing such a set {are} provided, and {there is no known tractable algorithm to verify if a given set satisfies this property.}  {Also, it is implicitly  assumed that projecting onto this set can be done efficiently.}

\paragraph{Contributions.}
{To develop the first polynomial-time algorithm that provides high probability guarantees of sub-linear regret}, we leverage recent results from the estimation of linear systems~\cite{simchowitz18}, robust controller synthesis~\cite{virtual,SysLevelSyn1}, and coarse-ID control~\cite{dean17}. {We show} that our robust adaptive control algorithm:
(i) guarantees stability and near-optimal performance at all times; (ii) achieves a regret up to time $T$ bounded by $\Otilde(T^{2/3})$; and (iii) is based on finite-dimensional semidefinite programs of size logarithmic in $T$.

Furthermore, {our method estimates the system parameters at $\Otilde(T^{-1/3})$ rate in operator norm.} 
Although system parameter identification is not necessary for optimal control performance, 
an accurate system model is often desirable in practice.
Motivated by this, we study the interplay between regret minimization and parameter estimation, and identify fundamental limits connecting the two.  We show that the expected regret of our algorithm is lower bounded by $\Omega(T^{2/3})$, proving that our analysis is sharp up to logarithmic factors. 
{Moreover, our lower bound suggests that the estimation rate achievable by any algorithm with $\calO(T^{\alpha})$ regret is $\Omega(T^{-\alpha/2})$.}

Finally, we conduct a numerical study of the adaptive LQR problem, in which we implement our algorithm, and compare its performance to heuristic implementations of OFU and TS based methods.
We show on
several examples that the regret incurred by our algorithm is comparable 
to {that of the} OFU and TS based methods. {Furthermore}, the infinite horizon cost
achieved by our algorithm at any given time on the true system is
consistently lower than that attained by OFU and TS based algorithms. 
Finally, {we use a demand forecasting example to show how our algorithm naturally generalizes to incorporate environmental uncertainty and safety constraints.}


\section{Problem Statement and Preliminaries}
\label{sec:problem}

\label{sec:problem_statement}

In this work we consider adaptive control of the following discrete-time linear system
\begin{align}
  x_{k+1} = \trueA x_{k} + \trueB u_k + w_k \:, \:\: w_k \iid \calN(0, \sigma_w^2 I) \:,
  \label{eq:dynamics}
\end{align}
where $x_k \in \R^\statedim$ is the state, $u_k \in \R^\inputdim$ is the control input, and $w_k \in \R^n$ is the process noise.
{We assume that the state variables are observed exactly and, for simplicity, that $x_0 = 0$.}
We consider the {\emph{Linear Quadratic Regulator} optimal control problem, given by cost matrices $Q \succeq 0$ and $R \succ 0$,}
\begin{align*}
  J_\star = \min_{u} \lim_{T \to \infty} \frac{1}{T} \E\left[ \sum_{k=1}^{T} x_k^\T Q x_k + u_k^\T R u_k \right]  \ \text{s.t. dynamics \eqref{eq:dynamics}}\:,
\end{align*}
where the minimum is taken over measurable functions $u = \{ u_k(\cdot) \}_{k \geq 1}$, with
each $u_k$ adapted to the history $x_k$, $x_{k - 1}$, \ldots, $x_1$, and possibe additional randomness independent of future states.
Given knowledge of $(\trueA, \trueB)$, the optimal policy is a static state-feedback law
$u_k = \trueK x_k$, where $\trueK$ is derived from the solution to a discrete algebraic
Riccati equation.

We{ are interested in  algorithms
which} operate without knowledge of the true system transition matrices $(\trueA, \trueB)$. We measure the performance
of such algorithms via their regret, {defined as}
\begin{align*}
  \mathsf{Regret}(T) := \sum_{{k=1}}^{T} (x_k^\T Q x_k + u_k^\T R u_k - J_\star) \:.
\end{align*}
The regret of any algorithm is lower-bounded by $\Omega(\sqrt{T})$, a bound matched by OFU up to logarithmic factors \cite{faradonbeh17b}. However, after each epoch, OFU requires
optimizing a non-convex objective to $\calO(T^{-1/2})$ precision.
Instead, our method uses a subroutine based on convex optimization
and robust control.

\subsection{Preliminaries: System Level Synthesis}

{
We briefly describe the necessary background on robust control and
System Level Synthesis~\cite{SysLevelSyn1} (SLS).
These tools were recently used by \citet{dean17} to provide
non-asymptotic bounds for LQR in the offline ``estimate-and-then-control'' setting.
In Appendix~\ref{sec:appendix:sls}, we expand on these preliminaries.}

Consider the dynamics \eqref{eq:dynamics}, and fix a static state-feedback control policy $K$, i.e., let $u_k = Kx_k$.  Then, the closed {loop map from} the disturbance process $\{w_0, w_1, \dots\}$ to the state $x_k$ and control input $u_k$ at time $k$ is given by
\begin{equation}
\begin{array}{rcl}
x_k &=& \sum_{t=1}^{k} (\trueA + \trueB K)^{k-t}w_{t-1} \:, \\
u_k &=& \sum_{t=1}^k K(\trueA + \trueB K)^{k-t}w_{t-1} \:.
\end{array}
\label{eq:impulse-response}
\end{equation}
Letting $\Phi_x(k) := (\trueA + \trueB K)^{k-1}$ and $\Phi_u(k) := K(\trueA + \trueB K)^{k-1}$, we can rewrite {Eq.~\eqref{eq:impulse-response}} as
\begin{equation}
\begin{bmatrix} x_k \\ u_k \end{bmatrix} =
\sum_{t=1}^k \begin{bmatrix}\Phi_x(k-t+1) \\ \Phi_u(k-t+1) \end{bmatrix}w_{t-1} \:,
\label{eq:phis}
\end{equation}
where {$\{\Phi_x(k),\Phi_u(k)\}$ are called the \emph{closed loop system response elements} induced by the controller $K$. The SLS framework shows that for any elements $\{\Phi_x(k),\Phi_u(k)\}$ constrained to obey
\begin{equation}
\Phi_x(k+1) = \trueA \Phi_x(k) + \trueB \Phi_u(k) \:, \:\: \Phi_x(1) = I \:, \:\: \forall k \geq 1 \:,
\label{eq:time-achievability}
\end{equation}
there exists some controller that achieves the desired system responses~\eqref{eq:phis}. Theorem \ref{thm:param} formalizes this observation: the SLS framework thereore allows for any optimal control problem over linear systems to be cast as an optimization problem over elements $\{\Phi_x(k),\Phi_u(k)\}$, constrained to satisfy the affine equations~\eqref{eq:time-achievability}. Comparing equations \eqref{eq:impulse-response} and \eqref{eq:phis}, we see that the former is non-convex in the controller $K$, whereas the latter is affine in the elements $\{\Phi_x(k),\Phi_u(k)\}$, enabling solutions to previously difficult optimal control problems. }

As we work with infinite horizon problems, it is notationally more convenient to work with \emph{transfer function} representations of the above objects, which can be obtained by taking a $z$-transform of their time-domain representations. The frequency domain variable $z$ can be informally thought of as the time-shift operator, i.e., $z\{x_k,x_{k+1},\dots\} = \{x_{k+1},x_{k+2},\dots\}$, allowing for a compact representation of LTI dynamics. {We use boldface letters to denote such transfer functions, e.g., $\tf{\Phi}_x(z) = \sum_{k = 1}^\infty\Phi_x(k) z^{-k}$. Then, the constraints \eqref{eq:time-achievability} can be rewritten as }
\begin{equation*}
\begin{bmatrix} zI - \trueA & - \trueB \end{bmatrix} \begin{bmatrix} \tf \Phi_x \\ \tf \Phi_u \end{bmatrix} = I \:,
\end{equation*}
{and the corresponding (not necessarily static) control law $\tf u = \tf K \tf x$ is given by $\tf K = \tf \Phi_u \tf \Phi^{-1}_x$.}

Although other approaches to optimal controller design exists, we argue now that the SLS parameterization
 has some appealing properties when applied to the control of uncertain
systems.  In particular, suppose that rather than having access to the true
system transition matrices $(\trueA, \trueB)$, we instead only have access to
estimates $(\Ah,\Bh)$.  The SLS framework allows us to characterize
the system responses achieved by a controller, computed
using only the estimates $(\Ah,\Bh)$, on the true system $(\trueA,\trueB)$.
Specifically, if we denote $\Dh:= (\Ah - \trueA)\tf \Phi_x + (\Bh - \trueB)\tf \Phi_u$, simple algebra shows that
\begin{align*}
    \begin{bmatrix} zI - \Ah & - \Bh \end{bmatrix} \begin{bmatrix} \tf \Phi_x \\ \tf \Phi_u \end{bmatrix} = I ~~~~~\text{if and only if}~~~~~
    \begin{bmatrix} zI - \trueA & - \trueB \end{bmatrix} \begin{bmatrix} \tf \Phi_x \\ \tf \Phi_u \end{bmatrix} = I + \Dh \:.
\end{align*}
Theorem \ref{thm:robust} shows that if $(I+\Dh)^{-1}$ exists, then the controller $\tf K = \tf \Phi_u \tf \Phi_x^{-1}$, computed using only the estimates $(\Ah,\Bh)$, achieves the following response on the true system $(\trueA, \trueB)$:
\[
\begin{bmatrix} \tf x \\ \tf u \end{bmatrix} =  \begin{bmatrix} \tf \Phi_x \\ \tf \Phi_u \end{bmatrix}(I+\Dh)^{-1} \tf w \:.
\]
{Further, if $\tf K$ stabilizes the system} $(\Ah,\Bh)$, and $(I+\Dh)^{-1}$ is stable (simple sufficient conditions can be derived to ensure this, see \cite{dean17}), then $\tf K$ is also stabilizing for the true system.  It is this transparency between system uncertainty and controller {performance that} we exploit in our algorithm.

We end this discussion with the definition of a function space that we use extensively throughout:
\begin{align*}
  \calS(C, \rho) &:= \left\{ \tf M = \sum_{k=0}^\infty M(k) z^{-k} \:|\: {\norm{ M(k)} \leq C \rho^k} \:, \:\: k = 0, 1, 2, ... \right\} \:.
\end{align*}
The space $\calS(C,\rho)$ consists of stable transfer functions that satisfy a certain decay rate in the spectral norm of their impulse response elements.  We denote the restriction of $\calS(C,\rho)$ to the space of $F$-length finite impulse response (FIR) filters by $\calS_F(C,\rho)$,{ i.e., $\tf M \in  \calS_F(C,\rho)$} if $\tf M \in \calS(C,\rho)$, and $M(k) = 0$ for all $k > F$.  Further note that we write $\tf M \in \frac{1}{z}\calS(C,\rho)$ to mean that $z \tf M \in \calS(C,\rho)$, i.e., that $M(0) = 0$.

We equip $\calS(C,\rho)$ with the $\hinf$ and $\htwo$ norms, which are infinite horizon analogs of the spectral and Frobenius norms of a matrix, respectively:
$\hinfnorm{\tf M} = \sup_{\twonorm{\tf w}=1} \: \twonorm{\tf M\tf w}$
and $\htwonorm{\tf M} = {\sqrt{\sum_{k=0}^\infty \norm{ M(k)}_F^2 }}$.
The $\hinf$ and $\htwo$ norm {have distinct interpretations.} The $\hinf$ norm
of a system $\tf M$ is equal to its $\ell_2 \mapsto \ell_2$ operator norm, and
can be used to measure the robustness of a system to unmodelled
dynamics~\cite{ZDGbook}. The $\htwo$ norm has a direct interpretation as the
energy transferred to the system by a white noise process, and is hence closely
related to the LQR optimal control problem.  Unsurprisingly, the $\htwo$ norm appears in the objective function of our
optimization problem, whereas the $\hinf$ norm appears in the constraints to
ensure robust stability and performance.

\section{Algorithm and Guarantees}

Our proposed robust adaptive control algorithm for LQR is shown in Algorithm~\ref{alg:adaptive}.
We note that while Line~\ref{lst:line:sls} of Algorithm~\ref{alg:adaptive}
is written as an infinite-dimensional optimization problem,
because of the FIR nature of the decision variables, it can be equivalently
written as a finite-dimensional semidefinite program. We describe
this transformation in Section~\ref{sec:app:sls_fir} of the Appendix.

\begin{center}
    \begin{algorithm}[h]
    \caption{Robust Adaptive Control Algorithm}
    \begin{algorithmic}[1]
      \REQUIRE{Stabilizing controller $\bmtx{K}^{(0)}$, failure probability $\delta \in (0, 1)$, and
        constants $(C_\star, \rho_\star, \norm{\trueK})$.}
        \STATE Set $C_x \gets \frac{\const C_\star}{(1-\rho_\star)^3}$, $C_u \gets \norm{\trueK} C_x$, and $\rho \gets .999 + .001 \rho_\star$.
	      \STATE Set $C_T \gets \Otilde\left( (n+p) \frac{C_\star^4(1+\norm{\trueK})^4}{(1-\rho_\star)^8}\right)$.
        \FOR{$i = 0, 1, 2, ...$}
            \STATE{Set $T_i \gets C_T 2^{i}$ and $\sigma_{\eta,i}^2 \gets \sigma_w^2 (T_i/C_T)^{-1/3}$.}
            \STATE{$D_i = \{(x_k^{(i)}, u_k^{(i)}\}_{k=1}^{T_i} \gets$ evolve system forward $T_i$ steps using feedback $\bvct{u} = \bmtx{K}^{(i)} \bvct{x} + \bmtx{\eta}_i$, where each entry of $\bmtx{\eta}_i$ is drawn i.i.d.\ from $\calN(0, \sigma_{\eta,i}^2 I_p)$. }
            \STATE{$(\Ah_i, \Bh_i) \gets \arg\min_{A, B} \sum_{k=1}^{T_i-1} \frac{1}{2}\ltwonorm{x_{k+1}^{(i)} - A x_k^{(i)} - B u_k^{(i)}}^2$. \label{lst:line:ls}}
            \STATE Set $\varepsilon_i \gets \Otilde\left( \frac{\sigma_w \norm{\trueK} C_\star}{\sigma_{\eta,i} (1-\rho_\star)^3} \sqrt{\frac{n+p}{T_i}} \right)$ and $F_i \gets \frac{\Otilde(1)(i+1)}{1-\rho_\star}$.
            \STATE{Set $\bmtx{K}^{(i+1)} = \tf \Phi_u \tf \Phi_x^{-1}$, where $(\tf \Phi_x, \tf \Phi_u)$ are the solution to
\begin{align*}
  \mathrm{minimize}_{\gamma \in [0, 1)} &\frac{1}{1-\gamma} \min_{\tf \Phi_x, \tf \Phi_u, V} \bightwonorm{ \begin{bmatrix} Q^{1/2} & 0 \\ 0 & R^{1/2} \end{bmatrix} \begin{bmatrix} \tf \Phi_x \\ \tf \Phi_u \end{bmatrix} }  \\
    \qquad \text{s.t.} &\rvectwo{zI - \Ah_i}{-\Bh_i} \cvectwo{\tf \Phi_x}{\tf \Phi_u} = I + \frac{1}{z^{F_i}}V \:, \:\: \frac{\sqrt{2}\varepsilon_i}{1-C_x\rho^{F_i+1}}\bighinfnorm{ \cvectwo{ \bmtx{\Phi_x}}{ \bmtx{\Phi_u}}} \leq  \gamma \:, \\
    & \norm{V} \leq C_x\rho^{F_i+1} \:, \:\: \tf \Phi_x \in \frac{1}{z}\calS_{F_i}(C_x, \rho) \:, \:\: \tf \Phi_u \in \frac{1}{z} \calS_{F_i}(C_u, \rho) \:.
\end{align*} \label{lst:line:sls} }
        \ENDFOR
    \end{algorithmic}
    \label{alg:adaptive}
    \end{algorithm}
\end{center}

Some remarks on practice are in order.
First, in Line~\ref{lst:line:ls}, only the trajectory data collected during
the $i$-th epoch is used for the least squares estimate.
Second, the epoch lengths we use grow exponentially in the epoch index.
These settings are chosen primarily
to simplify the analysis; in practice all the data collected should be used,
and it may be preferable to use a slower growing epoch schedule
(such as $T_i = C_T (i+1)$).
Finally, for storage considerations, instead of performing a batch
least squares update of the model, a recursive least squares (RLS)
estimator rule can be used to update the parameters in an online manner.

\subsection{Regret Upper Bounds}
\label{sec:upper_bound}
Our guarantees for Algorithm~\ref{alg:adaptive} are stated in terms of certain system specific constants,
which we define here.
We let $\trueK$ denote the static feedback solution
to the LQR problem for $(\trueA, \trueB, Q, R)$.
Next, we define $(C_\star, \rho_\star)$
such that the closed loop system $\trueA + \trueB \trueK$ belongs to
$\calS(C_\star, \rho_\star)$.
Our main assumption is stated as follows.
\begin{assumption}
\label{assumption:stablizing_controller}
We are given a controller $\tf K^{(0)}$ that stabilizes the true system $(\trueA, \trueB)$.
Furthermore, letting $( \tf \Phi_x, \tf \Phi_u )$ denote the response of $\tf K^{(0)}$ on $(\trueA, \trueB)$,
we assume that $\tf \Phi_x \in \calS(C_x, \rho)$ and $\tf \Phi_u \in \calS(C_u, \rho)$, where
the constants $C_x, C_u, \rho$ are defined in Algorithm~\ref{alg:adaptive}.
\end{assumption}
The requirement of an initial stabilizing controller $\tf K^{(0)}$ is not restrictive;
Dean et al.~\cite{dean17} provide an offline strategy for finding such a controller.
Furthermore, in practice Algorithm~\ref{alg:adaptive} can be initialized with
no controller, with random inputs applied instead to the system in the first
epoch to estimate $(\trueA, \trueB)$ within an initial confidence set for which the
synthesis problem becomes feasible.

Our first guarantee is on the rate of estimation of $(\trueA, \trueB)$ as the algorithm progresses
through time.
This result builds on recent progress~\cite{simchowitz18} for
estimation along trajectories of a linear dynamical system.
For what follows, the notation $\Otilde(\cdot)$ hides absolute constants
and
$\mathrm{polylog}\left( T, \frac{1}{\delta}, C_\star, \frac{1}{1-\rho_\star}, n, p, \norm{\trueB}, \norm{\trueK}\right)$
factors.
\begin{theorem}
\label{thm:estimation_main}
Fix a $\delta \in (0, 1)$ and suppose that Assumption~\ref{assumption:stablizing_controller} holds.
With probability at least $1-\delta$ the following statement holds.
Suppose that $T$ is at an epoch boundary.
Let $(\Ah(T), \Bh(T))$ denote the current estimate of $(\trueA, \trueB)$ computed by Algorithm~\ref{alg:adaptive} at the end of time $T$.
Then, this estimate satisfies the guarantee
\begin{align*}
  \max\{\norm{\Ah(T) - \trueA}, \norm{\Bh(T) - \trueB}\} \leq \Otilde\left(\frac{C_\star \norm{\trueK}}{(1-\rho_\star)^3} \frac{\sqrt{n+p}}{T^{1/3}} \right) \:.
\end{align*}
\end{theorem}
Theorem~\ref{thm:estimation_main} shows that Algorithm~\ref{alg:adaptive} achieves a consistent estimate
of the true dynamics $(\trueA, \trueB)$, and learns at a rate of $\Otilde(T^{-1/3})$.
We note that consistency of parameter estimates is not a guarantee provided by OFU or TS based approaches.

Next, we state an upper bound on the regret incurred by Algorithm~\ref{alg:adaptive}.
\begin{theorem}
\label{thm:regret}
Fix a $\delta \in (0, 1)$ and suppose that Assumption~\ref{assumption:stablizing_controller} holds.
With probability at least $1-\delta$ the following statement holds.
For all $T \geq 0$ we have that Algorithm~\ref{alg:adaptive} satisfies
\begin{align*}
  \mathsf{Regret}(T) \leq
  \Otilde\left( (n+p) \frac{C_\star^4 (1 + \norm{\trueK})^4 (1 + \norm{\trueB})^2 J_\star}{(1-\rho_\star)^{16}} T^{2/3} \right) \:.
\end{align*}
Here, the notation $\Otilde(\cdot)$ also hides $o(T^{2/3})$ terms.
\end{theorem}

The intuition behind our proof is transparent. We use SLS to
show that the cost during epoch $i$ is bounded by $T_i (1 + \calO(\sigma_{\eta,i}^2/\sigma_w^2))(1 + \calO(\varepsilon_{i-1})) J_\star$,
where the $(1+\calO(\varepsilon_{i-1}))$ factor is the performance degredation incurred by model uncertainty,
and the $(1 + \calO(\sigma_{\eta,i}^2/\sigma_w^2))$ factor is the additional cost incurred from injecting exploration noise.
Hence, the regret incurred during this epoch is $\calO( T_i ( \sigma_{\eta,i}^2/\sigma_w^2 + \varepsilon_{i-1} ) J_\star )$,
We then bound our estimation error by $\varepsilon_i = \Otilde( (\sigma_w/\sigma_{\eta,i}) T_i^{-1/2} )$.
Setting $\sigma_{\eta,i}^2 = \sigma_w^2 T_i^{-\alpha}$, we have the per epoch bound
$\Otilde( T_i^{1-\alpha} + T_i^{1 - (1-\alpha)/2} )$. Choosing $\alpha = 1/3$ to balance these competing powers of $T_i$
and summing over logarithmic number of epochs, we obtain a final regret of $\Otilde(T^{2/3})$.

The main difficulty in the proof is ensuring that the transient behavior of the
resulting controllers is uniformly bounded when applied to the true system. Prior works sidestep
this issue by assuming that the true dynamics lie within a (known) compact set for which
the Heine-Borel theorem asserts the existence of finite constants that capture this behavior.
We go a step further and work through the perturbation analysis which allows us to
give a regret bound that depends only on simple quantities of the true system $(\trueA, \trueB)$.
The full proof is given in the appendix.

Finally, we remark that the dependence on $1/(1-\rho_\star)$ in our results is
an artifact of our perturbation analysis, and we leave sharpening this dependence to future work.


\subsection{Regret Lower Bounds and Parameter Estimation Rates}
\label{sec:lower_bound}
 
We saw that Algorithm~\ref{alg:adaptive} achieves $\Otilde(T^{2/3})$ regret with high probability. Now we provide a matching algorithmic lower bound on the expected regret, showing that the analysis presented in Section~\ref{sec:upper_bound} is sharp as a function of $T$. Moreover, our lower bound characterizes how much regret must be accrued in order to achieve a specified estimation rate for the system parameters $(\trueA, \trueB)$.

\begin{theorem}
\label{thm:lower_bound}
Let the initial state $x_0$ be distributed according to the steady state distribution $\calN(0, P_\infty)$ of the optimal closed loop system, and let $\{u_t\}_{t \geq 0}$ be any sequence of inputs as in Section~\ref{sec:problem_statement}. Furthermore, let $f\colon \R \rightarrow \R$ be any function such that with probability $1 - \delta$ we have
\begin{align} 
\label{eq:min_eig_design}
\lambda_{\min}\left(\sum_{k = 0}^{T - 1} \begin{bmatrix}
x_k \\
u_k
\end{bmatrix} \begin{bmatrix}
x_k^\T &
u_k^\T
\end{bmatrix}
\right) \geq f(T) \:.
\end{align}
Then, there exist positive values $T_0$ and $C_0$ such that for all $T \geq T_0$ we have
\begin{align*}
\sum_{k = 0}^{T} \E \left[x_k^\T Q x_k + u_k^\T R u_k - J_\star\right] \geq \frac{1}{2}(1 - \delta) \lambda_{\min}(R) (1 + \sigma_{\min}(\trueK)^2) f(T - T_0) - C_0 \:,
\end{align*}
where $T_0$ and $C_0$ are functions of $\trueA$, $\trueB$, $Q$, $R$, $\sigma_w^2$, and $\statedim$, detailed in Appendix~\ref{app:lower_bound}.
\end{theorem}

The proof of the estimation error Theorem~\ref{thm:estimation_main} shows that
Algorithm~\ref{alg:adaptive} satisfies
Eq.~\eqref{eq:min_eig_design} with $f(T) = \Otilde(T \sigma_{\eta, \Theta(\log_2(T))}^2)$.
Since the exploration variance $\sigma_{\eta,i}^2$ used by Algorithm~\ref{alg:adaptive} during the $i$-th epoch
is given by $\sigma_{\eta,i}^2 = \calO(\sigma_w^2 T^{-i/3})$, we obtain the following corollary
which demonstrates the sharpness of our regret analysis with respect to the scaling of $T$.
\begin{coro}
For $T > C_1(\statedim, \delta, \sigma_w^2, \trueA, \trueB, Q, R)$ the expected regret of Algorithm~\ref{alg:adaptive} satisfies
\begin{align*}
  \sum_{k = 1}^{T} \E \left[x_k^\T Q x_k + u_k^\T R u_k - J_\star\right] \geq \Omegatilde(\lambda_{\min}(R) (1 + \sigma_{\min}(\trueK)^2) T^{2/3} ) \:.
\end{align*}
\end{coro}

A natural question to ask is how much regret does any algorithm accrue in order
to achieve estimation error $\norm{\Ahat - \trueA} \leq \varepsilon$ and
$\norm{\Bhat - \trueB} \leq \varepsilon$. From Theorem~\ref{thm:estimation_main} we
know that Algorithm~\ref{alg:adaptive} estimates $(\trueA, \trueB)$ at rate
$\Otilde(T^{-1/3})$.  Therefore, in order to achieve $\varepsilon$ estimation
error, $T$ must be $\Omegatilde(\varepsilon^{-3})$. Hence,
Theorem~\ref{thm:regret} implies that the regret of
Algorithm~\ref{alg:adaptive} to achieve $\varepsilon$ estimation error is
$\Otilde(\varepsilon^{-2})$.

Interestingly, let us consider any other
Algorithm achieving $\calO(T^{\alpha})$ regret for some $0 < \alpha <
1$. Then, Theorem~\ref{thm:lower_bound} suggests that the best  rate
achievable by such an algorithm is $\calO(T^{-\alpha/2})$, since
the minimum eigenvalue condition Eq.~\eqref{eq:min_eig_design} governs the signal-to-noise ratio.
In the case of linear-regression with independent data it is known that the minimax estimation rate is lower bounded
by square root of the inverse of the minimum eigenvalue~\eqref{eq:min_eig_design}. We conjecture that the same
results holds in our case.
Therefore,
to achieve $\varepsilon$ estimation error, any Algorithm would
likely require $\Omega(\varepsilon^{-2})$ regret, showing that
Algorithm~\ref{alg:adaptive} is optimal up to logarithmic factors in this sense.
Finally, we note that
while Algorithm~\ref{alg:adaptive} estimates $(\trueA, \trueB)$ at a rate
$\Otilde(T^{-1/3})$, Theorem~\ref{thm:lower_bound} suggests that any
algorithm achieving the  $\calO(\sqrt{T})$ regret would estimate $(\trueA, \trueB)$ at a
rate $\Omega(T^{-1/4})$.
%



\section{Experiments}
\label{sec:experiments}

\paragraph{Regret Comparison.}
We illustrate the performance of several adaptive schemes empirically.
We compare the proposed robust adaptive method with non-Bayesian Thompson sampling (TS) as in~\citet{abeille17} and a heuristic projected gradient descent (PGD) implementation of OFU. As a simple baseline,
we use the nominal control method, which synthesizes the optimal infinite-horizon LQR controller for the estimated system
and injects noise with the same schedule as the robust approach. Implementation details and computational considerations for all adaptive methods are in Appendix~\ref{sec:appendix:implementation}.

The comparison experiments are carried out on the following LQR problem:
\begin{align} \label{eq:exampledynamics_laplacian}
\trueA  = \begin{bmatrix} 1.01 & 0.01 & 0\\
0.01 & 1.01 & 0.01\\
0 & 0.01 & 1.01\end{bmatrix}, ~~ \trueB = I, ~~ Q = 10 I, ~~ R =I, ~~ \sigma_w = 1 \:.
\end{align}
This system corresponds to a marginally unstable Laplacian system where adjacent nodes are weakly connected; these dynamics were also studied by~\cite{abbasi18,dean17,tu2017least}.
The cost is such that input size is penalized relatively less than state. This problem setting is amenable to robust methods due to both the cost ratio and the marginal instability, which are factors that may hurt optimistic methods.  In Appendix~\ref{sec:app:laplacian_exp}, we show similar results for an unstable system with large transients.

To standardize the initialization of the various adaptive methods, we use a
rollout of length $T_0=100$ where the input is a stabilizing controller plus
Gaussian noise with fixed variance $\sigma_u=1$. This trajectory is not counted
towards the regret, but the recorded states and inputs are used to initialize
parameter estimates.
In each experiment, the system starts from $x_0=0$
to reduce variance over runs. For all methods, the actual errors
$\Ahat_t-\trueA$ and $\Bhat_t-\trueB$ are used rather than bounds or
bootstrapped estimates. The effect of this choice on regret is small, as
examined empirically in Appendix~\ref{sec:app:error_scaling}.

\begin{figure} 
\centering
\begin{subfigure}[b]{\basefigwidth\textwidth}
\caption{\small Regret}
\centerline{\includegraphics[width=\columnwidth]{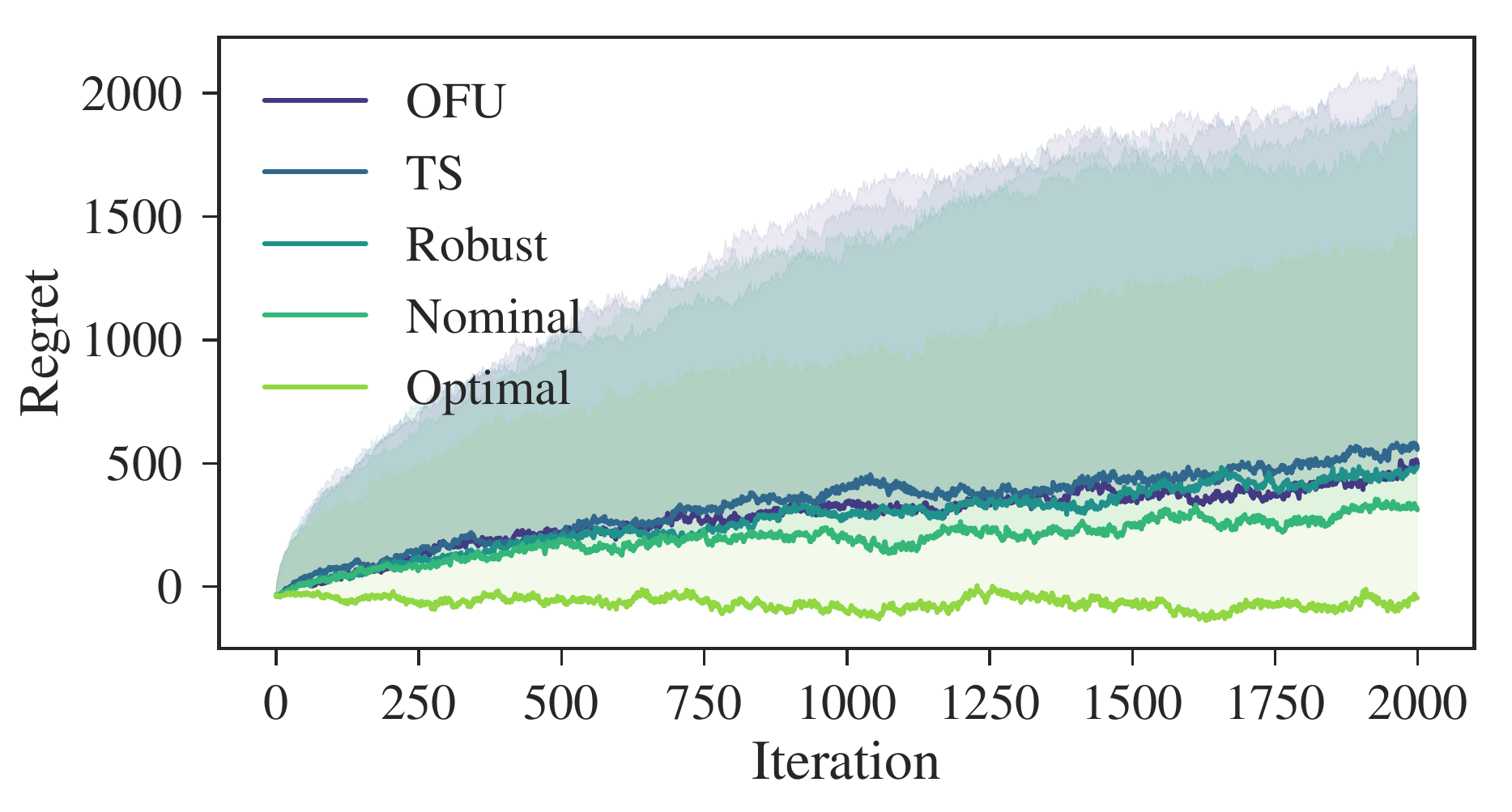}}
\end{subfigure}
\begin{subfigure}[b]{\basefigwidth\textwidth}
\caption{\small Infinite Horizon LQR Cost}
\centerline{\includegraphics[width=\columnwidth]{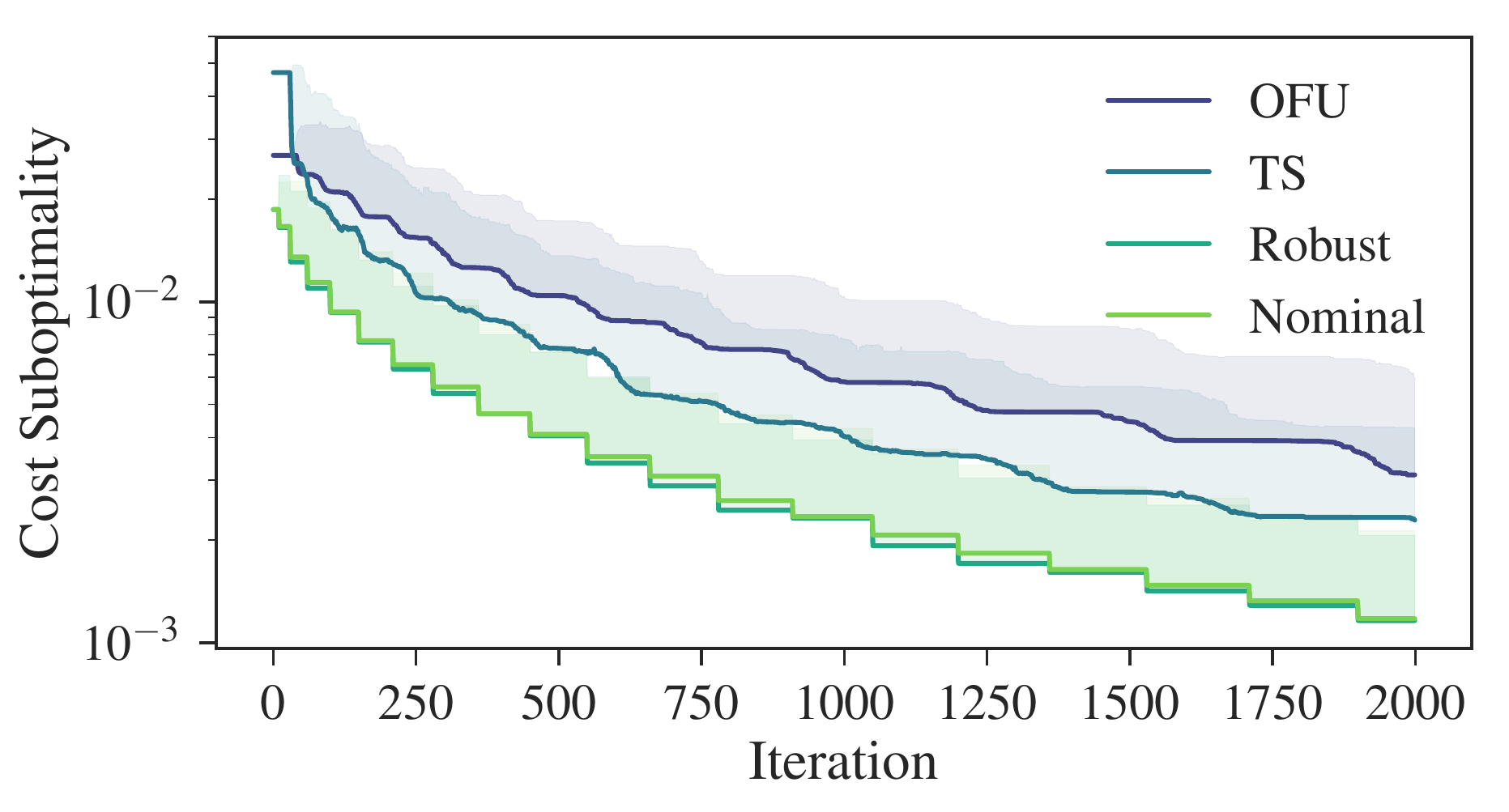}}
\end{subfigure}
\caption{\small A comparison of different adaptive methods on 500 experiments of the marginally unstable Laplacian example in~\ref{eq:exampledynamics_laplacian}. In (a), the median and 90th percentile regret is plotted over time. In (b), the median and 90th percentile infinite-horizon LQR cost of the epoch's controller.}
\label{fig:regrets}
\end{figure}

The performance of the various adaptive methods is compared in Figure~\ref{fig:regrets}. The median and 90th percentile regret over 500 instances is displayed in  Figure~\ref{fig:regrets}a, which gives an idea of both typical and worst-case behavior. The regret of the optimal LQR controller for the true system is displayed as a baseline.
Overall, the methods have very similar performance. One benefit of robustness is the guaranteed stability and bounded infinite-horizon cost at every point during operation.
In  Figure~\ref{fig:regrets}b, this infinite-horizon LQR cost is plotted for the controllers played during each epoch.
This value measures the cost of using each epoch's controller indefinitely, rather than continuing to update its parameters.
The robust adaptive method performs relatively better than other adaptive algorithms, indicating that it is more amenable to early stopping, i.e., to turning off the adaptive component of the algorithm and playing the current controller indefinitely.

\paragraph{Extension to Uncertain Environment with State Constraints.}

\begin{figure} 
\centering
\begin{subfigure}[b]{\basefigwidth\textwidth}
\caption{\small Demand Forecasting}
\label{fig:demand_forcasting}
\centerline{\includegraphics[width=\columnwidth]{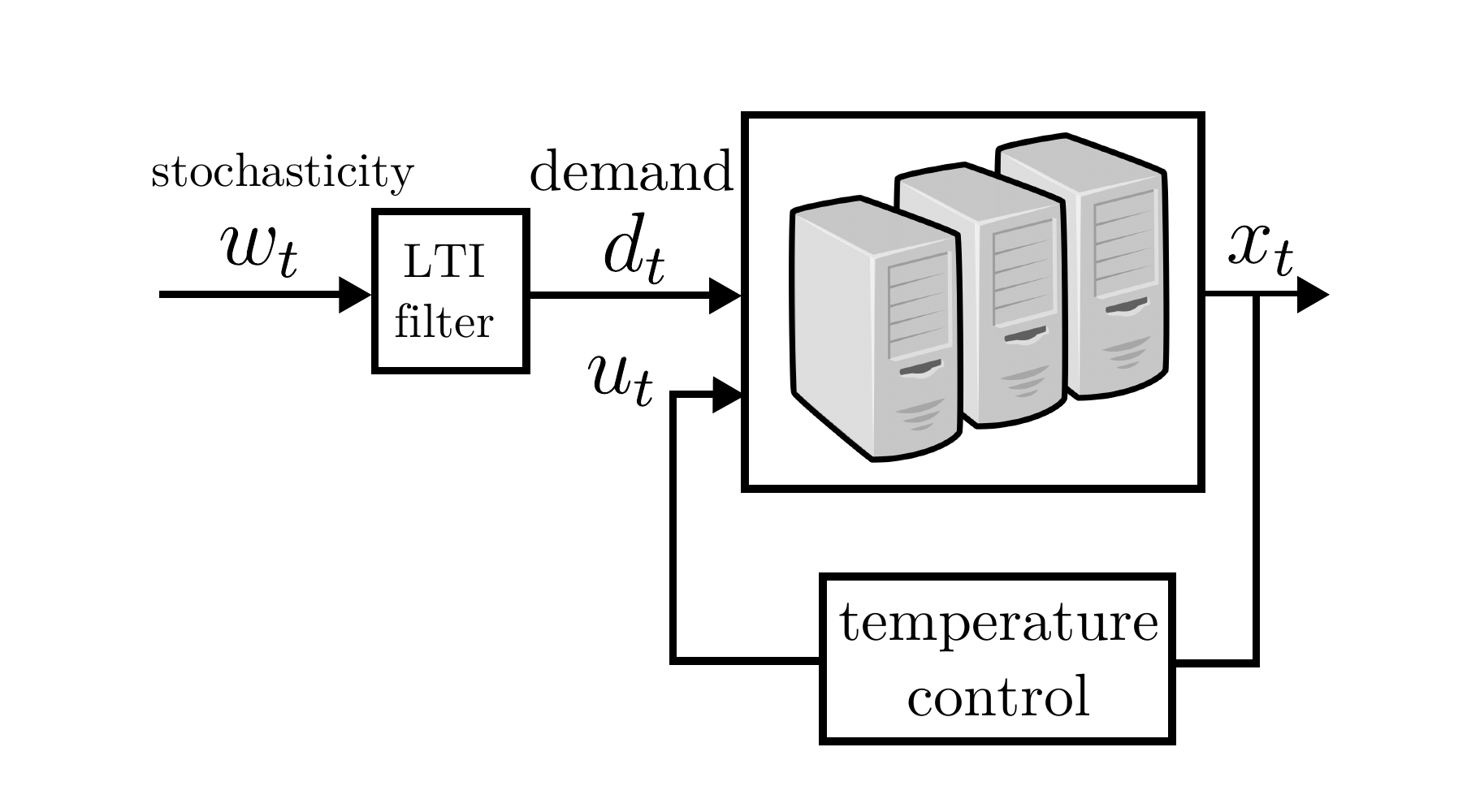}}
\end{subfigure}
\begin{subfigure}[b]{\basefigwidth\textwidth}
\caption{\small Constraint Satisfaction}
\label{fig:constraint_satisfaction}
\centerline{\includegraphics[width=\columnwidth]{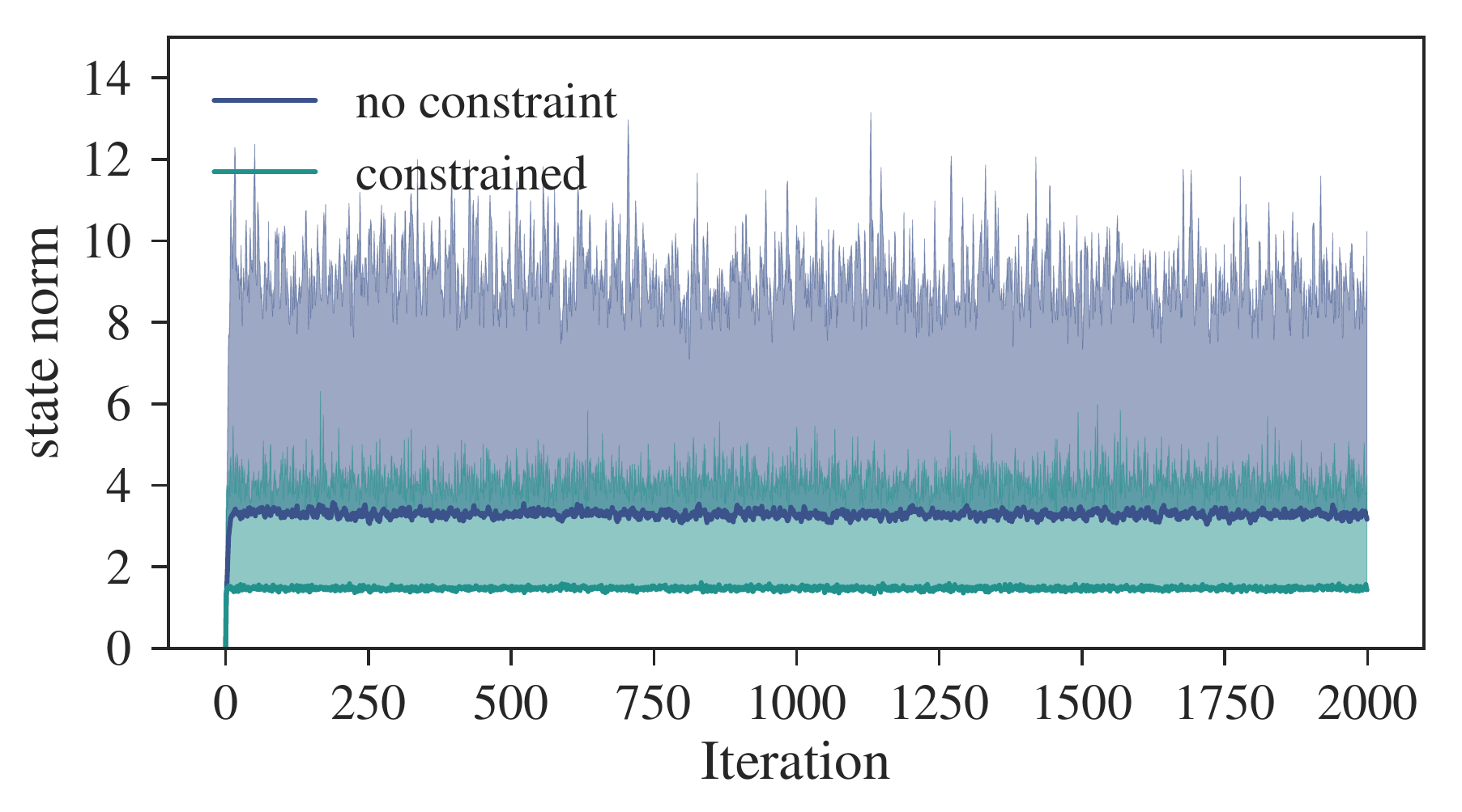}}
\end{subfigure}
\caption{\small The addition of constraints in the robust synthesis problem can guarantee the safe execution of adaptive systems. We consider an example inspired by demand forecasting, as illustrated in (a), where the left hand side of the diagram represents unknown dynamics. The median and maximum values of $\|x_t\|_\infty$ over 500 trials are plotted for both the unconstrained and constrained synthesis problems in (b).}
\label{fig:constraints}
\end{figure}

The proposed robust adaptive method naturally generalizes beyond the standard LQR problem. We consider
a disturbance forecasting example which incorporates environmental uncertainty and safety constraints.
Consider a system with known dynamics driven by stochastic disturbances that are now correlated in time. We model the disturbance process as the output of an unknown autonomous LTI system, as illustrated in Figure~\ref{fig:demand_forcasting}. This setting can be interpreted as a demand forecasting problem, where, for example, the system is a server farm and the disturbances represent changes in the amount of incoming jobs. If the dynamics of the correlated disturbance process are known, this knowledge can be used for more cost-effective temperature control.

We let the system $(\trueA, \trueB)$ with known dynamics be described by the graph Laplacian dynamics as in Eq.~\eqref{eq:exampledynamics_laplacian}. The disturbance dynamics are unknown and are governed by a stable system transition matrix $A_d$, resulting in the following dynamics for the full system:
\[\begin{bmatrix} x_{t+1} \\ d_{t+1} \end{bmatrix} =
\begin{bmatrix} \trueA & I \\ 0  & A_d \end{bmatrix} \begin{bmatrix} z_{t} \\ d_{t} \end{bmatrix} +
\begin{bmatrix} \trueB \\ 0   \end{bmatrix}  u_{t}
+ \begin{bmatrix} 0 \\ I \end{bmatrix} w_t\:, \quad A_d = \begin{bmatrix} 0.5 & 0.1 & 0\\
0 & 0.5 & 0.1\\
0 & 0 & 0.5\end{bmatrix}.\]
The costs are set to model expensive inputs, with $Q=I$ and $R=1\times10^{3}I$. The controller synthesis problem in Line~\ref{lst:line:sls} of Algorithm~\ref{alg:adaptive} is modified to reflect the problem structure, and crucially, we add a constraint on the system response $\tf \Phi_x$.
Further details of the formulation are explained in Appendix~\ref{sec:app:demand_ex}.
Figure~\ref{fig:constraint_satisfaction} illustrates the effect.  While the unconstrained synthesis results in trajectories with large state values, the constrained synthesis results in much more moderate behavior.



\section{Conclusions and Future Work}

We presented a polynomial-time algorithm for the adaptive LQR problem that provides high probability guarantees of sub-linear regret. 
In contrast to other approaches to this problem, our robust adaptive method guarantees stability, robust performance, and parameter estimation.  We also explored the interplay between regret minimization and parameter estimation, identifying fundamental limits connecting the two. 

Several questions remain to be answered.
It is an open question whether a polynomial-time algorithm can achieve a regret of  $\Otilde(\sqrt{T})$. 
In our implementation of OFU, we observed that PGD performed quite effectively.
Interesting future work is to see if the techniques of Fazel et al.~\cite{Fazel18}
for policy gradient optimization on LQR can be applied to prove convergence of PGD on the OFU subroutine,
which would provide an optimal polynomial-time algorithm.
Moreover, we observed that OFU and TS methods in practice gave estimates of system parameters that 
were comparable with our method which explicitly adds excitation noise. 
It seems that the switching of control policies at epoch boundaries provides more excitation
for system identification than is currently understood by the theory.
Furthermore, practical issues that remain to be addressed include satisfying safety constraints and dealing with nonlinear dynamics; in both settings, finite-sample parameter estimation/system identification and adaptive control remain an open problem.

\subsubsection*{Acknowledgments}
SD is supported by an NSF Graduate Research Fellowship.
As part of the RISE lab, HM is generally supported in part by NSF CISE Expeditions
Award CCF-1730628, DHS Award HSHQDC-16-3-00083, and gifts from Alibaba, Amazon
Web Services, Ant Financial, CapitalOne, Ericsson, GE, Google, Huawei, Intel, IBM, Microsoft,
Scotiabank, Splunk and VMware.
BR is generously supported in part by NSF award CCF-1359814, ONR awards
N00014-17-1-2191, N00014-17-1-2401, and N00014-17-1-2502, the DARPA Fundamental
Limits of Learning (Fun LoL) and Lagrange Programs, and an Amazon AWS AI
Research Award.

\bibliographystyle{plainnat}
\bibliography{lqr}

\clearpage
\appendix

\section{Background on System Level Synthesis}
\label{sec:appendix:sls}

We begin by defining
two function spaces
which we use extensively throughout:
\begin{align}
  \mathcal{RH}_\infty &= \{ \tf M : \C \longrightarrow \C^{n \times p} \:|\: \tf M(z) \text{ is rational} \:, \:\: \tf M(z) \text{ is analytic on } \mathbb{D}^c \} \:, \\
  \mathcal{RH}_\infty(C, \rho) &= \{ \tf M \in \mathcal{RH}_\infty \:|\: \norm{\tf M[k]} \leq C \rho^k \:, \:\: k = 1, 2, ... \} \:.
\end{align}
Note that we use $\calS(C,\rho)$ to denote $\mathcal{RH}_\infty(C, \rho)$ in the main body of the text.

Recall that our main object of interest is the system
\begin{align*}
  x_{k+1} = A x_k + B u_k + w_k \:,
\end{align*}
and our goal is to design a LTI feedback control policy $\tf u = \tf K \tf x$ such that
the resulting closed loop system is stable.
For a given $\tf K$, we refer to the closed loop transfer functions
from $\tf w \mapsto \tf x$ and $\tf w \mapsto \tf u$ as the \emph{system response}.
Symbolically, we denote these maps as $\tf \Phi_x$ and $\tf \Phi_u$.
Simple algebra shows that given $\tf K$, these maps take on the form
\begin{align}
  \tf \Phi_x = (zI - A - B\tf K)^{-1} \:, \:\:
  \tf \Phi_u = \tf K (z I - A - B \tf K)^{-1} \:. \label{eq:response}
\end{align}

We then have the following theorem parameterizing the set of such stable closed-loop transfer functions that are achievable by a stabilizing controller $\tf K$.
\begin{theorem}[State-Feedback Parameterization~\cite{SysLevelSyn1}]
The following are true:
\begin{itemize}
\item The affine subspace defined by
\begin{equation}
\begin{bmatrix} zI - A & - B \end{bmatrix} \begin{bmatrix} \tf \Phi_x \\ \tf \Phi_u \end{bmatrix} = I, \ \tf \Phi_x, \tf \Phi_u \in \frac{1}{z}\mathcal{RH}_\infty
\label{eq:achievable}
\end{equation}
parameterizes all system responses \eqref{eq:response} from $\tf w$ to $(\tf x, \tf u)$, achievable by an internally stabilizing state-feedback controller $\tf K$.
\item For any transfer matrices $\{\tf \Phi_x, \tf \Phi_u\}$ satisfying \eqref{eq:achievable}, the controller $\tf K = \tf \Phi_u \tf \Phi_x^{-1}$ is internally stabilizing and achieves the desired system response \eqref{eq:response}.
\end{itemize}
\label{thm:param}
\end{theorem}

If $\tf K$ stabilizes $(A, B)$, then the LQR cost of $\tf K$ on $(A, B)$
can be written by Parseval's identity as
\begin{align}
    J(A, B, \tf K; \sigma_w^2 I) := \lim_{T \to \infty} \frac{1}{T} \E\left[ \sum_{k=1}^{T} x_k^\T Q x_k + u_k^\T R u_k \right] = \sigma^2_w \bightwonorm{ \bmattwo{Q^{1/2}}{0}{0}{R^{1/2}} \cvectwo{\tf \Phi_x}{\tf \Phi_u}  }^2 \:.
\end{align}
More generally, we will define $J(A, B, \tf K; \Sigma)$ to be the LQR cost
when the process noise is driven by $w \iid \calN(0, \Sigma)$.
When we omit the last argument, we mean $\sigma_w^2=1$, i.e.
$J(A, B, \tf K) = J(A, B, \tf K; I)$.

In \cite{dean17}, the authors use SLS to study how uncertainty in
the true parameters $(\trueA, \trueB)$ affect the LQR objective cost.
Our analysis relies on these tools, which we briefly describe below.

The starting point for the theory is a characterization of all \emph{robustly}
stabilizing controllers.
\begin{theorem}[\cite{virtual}]
Suppose that the transfer matrices $\{\tf\Phi_x, \tf \Phi_u\} \in \frac{1}{z}\mathcal{RH}_\infty$ satisfy
\begin{equation}
\begin{bmatrix} zI - A & - B \end{bmatrix} \begin{bmatrix} \tf \Phi_x \\ \tf \Phi_u \end{bmatrix} = I + \tf \Delta.
\end{equation}
Then the controller $\tf K = \tf \Phi_u \tf \Phi_x^{-1}$ stabilizes the system described by $(A,B)$ if and only if $(I+\tf \Delta)^{-1} \in \mathcal{RH}_\infty$.  Furthermore, the resulting system response is given by
\begin{equation}
\begin{bmatrix} \tf x \\ \tf u \end{bmatrix} = \begin{bmatrix} \tf \Phi_x \\ \tf \Phi_u \end{bmatrix}(I+\tf \Delta)^{-1} \tf w.
\end{equation}
\label{thm:robust}
\end{theorem}

This robustness result is used to
derive a cost perturbation result for LQR.
\begin{lemma}[\cite{dean17}]\label{lemma:robust-sls}
Let the controller $\tf K$ stabilize $(\Ah, \Bh)$ and $(\tf\Phi_x,\tf\Phi_u)$ be its corresponding system response on system $(\Ah,\Bh)$.  Then if $\tf K$ stabilizes $(A,B)$, it achieves the following LQR cost
\begin{equation}
    \sqrt{J(A,B,\tf K)} = \left\|\begin{bmatrix} Q^\frac{1}{2} & 0 \\ 0 & R^\frac{1}{2}\end{bmatrix}\begin{bmatrix} \tf\Phi_x \\  \tf\Phi_u \end{bmatrix}\left(I+\begin{bmatrix}\Delta_A& \Delta_B\end{bmatrix}\begin{bmatrix} \tf\Phi_x \\  \tf\Phi_u \end{bmatrix}\right)^{-1}\right\|_{\htwo}\:.
\end{equation}

Furthermore, letting
\begin{equation}\label{eq:deltahat}
\Dh := \begin{bmatrix}\Delta_A& \Delta_B\end{bmatrix}\begin{bmatrix} \tf\Phi_x \\  \tf\Phi_u \end{bmatrix} \:.
\end{equation}
a sufficient condition for $\tf K$ to stabilize $(A,B)$ is that $\hinfnorm{\Dh} <1$.
An upper bound on $\hinfnorm{\Dh}$
is given by, for any $\alpha \in (0, 1)$,
\begin{align}
  \hinfnorm{\Dh} \leq \bighinfnorm{ \cvectwo{\frac{\varepsilon_A}{\sqrt{\alpha}} \tf \Phi_x}{\frac{\varepsilon_B}{\sqrt{1-\alpha}} \tf \Phi_u} } \:,
\end{align}
where we assume that $ \twonorm{A - \Ah}\leq \varepsilon_A$ and $\twonorm{B - \Bh}\leq \varepsilon_B $.
\end{lemma}

\section{Synthesis Results}

We first study the following infinite-dimensional synthesis problem.
\begin{align}\label{eq:adaptiveLQRbnd}
\begin{split}
 \minimize_{\gamma\in[0,1)}\frac{1}{1 - \gamma}&\min_{\tf\Phi_x, \tf \Phi_u} \left\|\begin{bmatrix} Q^\frac{1}{2} & 0 \\ 0 & R^\frac{1}{2}\end{bmatrix}\begin{bmatrix} \tf\Phi_x \\  \tf\Phi_u \end{bmatrix}\right\|_{\htwo}\\
& \text{s.t.} \begin{bmatrix}zI-\Ah&-\Bh\end{bmatrix}\begin{bmatrix} \tf\Phi_x \\  \tf\Phi_u \end{bmatrix} = I,~~\left\|\begin{bmatrix}  \tf \Phi_x \\ \tf\Phi_u \end{bmatrix} \right\|_{\hinf}\leq \tfrac{\gamma}{\sqrt{2}\varepsilon}\\
&\qquad \tf\Phi_x \in\frac{1}{z}\RHinf(C_x,\rho), \, \tf \Phi_u  \in\frac{1}{z}\RHinf(C_u,\rho).
 \end{split}
\end{align}

We will conduct our analysis assuming that this infinite-dimensional problem is
solvable.  Later on, we will show how to relax this problem to a
finite-dimension one via FIR truncation, and show the minor modifications
needed to the analysis for the guarantees to hold.

We now prove a sub-optimality guarantee on the solution to \eqref{eq:adaptiveLQRbnd}
which holds for certain choices of $\varepsilon$ and the coefficients
$(C_x, \rho_x)$ and $(C_u, \rho_u)$. This result also establishes
an important technical consideration, which is when the problemmmm
\eqref{eq:adaptiveLQRbnd} is feasible.
\begin{theorem}
\label{thm:lqr_feasibility_and_subopt}
Let $J_\star$ denote the minimal LQR cost achievable by any controller for the dynamical system with transition matrices $(\trueA, \trueB)$, and let $\trueK$ denote its optimal static feedback contoller.
Suppose that $\res{\trueA + \trueB\trueK} \in \RHinf(C_\star, \rho_\star)$ and that (wlog) $\rho_\star \geq 1/e$.
Suppose furthermore that $\varepsilon$ is small enough
to satisfy the following conditions:
\begin{align*}
    \varepsilon(1 + \norm{\trueK})\hinfnorm{\res{\trueA + \trueB\trueK}} &\leq 1/5 \:, \\
    \varepsilon(1 + \norm{\trueK}) C_\star &\leq 1 - \rho_\star \:.
\end{align*}
Let $(\Ah,\Bh)$ be any estimates of the transition matrices such that $\max\{\norm{\Delta_A},\norm{\Delta_B}\} \leq \varepsilon$. Then, if $(C_x,\rho)$ and $(C_u,\rho)$ are set as,
\begin{align*}
    C_x &= \frac{\const C_\star}{1-\rho_\star}  \:, \\
    C_u &= \frac{\const \norm{K_\star} C_\star}{1-\rho_\star}  \:, \\
    \rho &= (1/4) \rho_\star + 3/4 \:, \\
\end{align*}
we have that (a) the program \eqref{eq:adaptiveLQRbnd} is feasible,
(b) letting $\tf K$ denote an optimal solution to \eqref{eq:adaptiveLQRbnd},
the relative error in the LQR cost is
\begin{align} \label{eq:lqr_bound}
J(\trueA, \trueB, \tf K) \leq (1+5 \varepsilon(1 + \norm{\trueK})\hinfnorm{\res{\trueA + \trueB\trueK}})^2 J_\star \:,
\end{align}
and (c) if furthermore $\varepsilon(C_x + C_u) \leq 2 (1 - \rho_\star)$,
the response $\{ \tf{\widehat{\Phi}}_x, \tf{\widehat{\Phi}}_u \}$ of $\tf K$
on the true system $(\trueA, \trueB)$ satisfies
\begin{align*}
    \tf{\widehat{\Phi}}_x \in \RHinf\left( \frac{\const C_\star}{(1-\rho_\star)^2}, 7/8 + (1/8) \rho_\star\right) \:, \\
    \tf{\widehat{\Phi}}_u \in \RHinf\left( \frac{\const \norm{\trueK} C_\star}{(1-\rho_\star)^2}, 7/8 + (1/8) \rho_\star\right) \:.
\end{align*}
\end{theorem}
\begin{proof}
The proof of (a) and (b) is nearly identical to that given in \cite{dean17},
which works by showing that $\tf \Phi_x = \res{\Ah + \Bh \trueK}$
and $\tf \Phi_u = \trueK \res{\Ah + \Bh \trueK}$ is a feasible response
which gives the desired sub-optimality guarantee.
The only modification is that we need to find constants $C_x, C_u, \rho$
for which $\res{\Ah + \Bh \trueK} \in \frac{1}{z} \RHinf(C_x, \rho)$ and
$\trueK  \res{\Ah + \Bh \trueK} \in \frac{1}{z} \RHinf(C_u, \rho)$.
We do this by writing
\begin{align*}
    \res{\Ah + \Bh \trueK} = \res{\trueA + \trueB \trueK} ( I - \tf \Delta)^{-1}  \:, \:\:
    \tf\Delta = (\Delta_A + \Delta_B \trueK) \res{\trueA + \trueB \trueK} \:.
\end{align*}
By the definition of $\tf \Delta$ and our assumptions, we have that
\begin{align*}
    \tf \Delta \in \RHinf( \varepsilon(1 + \norm{\trueK}) C_\star ,  \rho_\star) \:, \:\: \hinfnorm{\tf \Delta} < 1 \:.
\end{align*}
This places us in a position to apply Lemma~\ref{lem:inverse_rhinf_direct}, from which we conclude that
\begin{align*}
    (I - \tf\Delta)^{-1} \in \RHinf\left( \const, \Avg(\rho_\star, 1) \right) \:.
\end{align*}
Now applying Lemma~\ref{lem:G1G2} to $\res{\trueA + \trueB \trueK} ( I - \tf \Delta)^{-1}$,
we conclude that
\begin{align*}
    \res{\Ah + \Bh \trueK} &\in \RHinf\left( \frac{\const C_\star}{1-\rho_\star} , (1/4) \rho_\star + 3/4 \right) \:.
\end{align*}
The claims of (a) and (b) now follows.

Now for the proof of (c).
Let $\{ \tf \Phi_x, \tf \Phi_u \}$ be the solution to \eqref{eq:adaptiveLQRbnd}.
We have that
\begin{align*}
    \cvectwo{\tf{\widehat{\Phi}}_x}{\tf{\widehat{\Phi}}_u} = \cvectwo{\tf \Phi_x}{\tf \Phi_u} (I + \tf{\widehat{\Delta}})^{-1} \:, \:\:
    \tf{\widehat{\Delta}} = \rvectwo{\Delta_A}{\Delta_B} \cvectwo{\tf \Phi_x}{\tf \Phi_u} \:.
\end{align*}
We know that $\hinfnorm{\tf{\widehat{\Delta}}} < 1$ by the constraints of the optimization problem
\eqref{eq:adaptiveLQRbnd} and furthermore,
\begin{align*}
    \tf{\widehat{\Delta}} \in \RHinf( \varepsilon(C_x + C_u), \rho ) \:.
\end{align*}
By assumption we have $\varepsilon(C_x + C_u) \leq 2$, from which we conclude
using Lemma~\ref{lem:inverse_rhinf_direct} that
\begin{align*}
    (I + \tf{\widehat{\Delta}})^{-1} \in \RHinf\left(\const, \Avg(\rho, 1) \right) \:.
\end{align*}
Furthermore, from Lemma~\ref{lem:G1G2}, we conclude that
\begin{align*}
    \tf \Phi_x (I + \tf{\widehat{\Delta}})^{-1} &\in \RHinf\left( \frac{C_x}{1-\rho} , 3/4 + (1/4) \rho \right) \:, \\
    \tf \Phi_u (I + \tf{\widehat{\Delta}})^{-1} &\in \RHinf\left( \frac{C_u}{1-\rho} , 3/4 + (1/4) \rho \right) \:.
\end{align*}
The claim now follows by plugging in the values of $C_x$, $C_u$, and $\rho$.
\end{proof}

\newcommand{\Phii}{\begin{bmatrix}\tf \Phi_x \\ \tf \Phi_u \end{bmatrix}}
\newcommand{\Cost}{\begin{bmatrix} Q^\frac{1}{2} & 0 \\ 0 & R^\frac{1}{2}\end{bmatrix}}
\newcommand{\Phis}{\begin{bmatrix}\boldsymbol{\Phi^\star_x} \\ \boldsymbol{\Phi^\star_u} \end{bmatrix}}
\newcommand{\Phixs}{\tf \Phi^\star_x}
\newcommand{\Phius}{\tf \Phi^\star_u}
\newcommand{\Phix}{\boldsymbol{\Phi}_x}
\newcommand{\Phiu}{\boldsymbol{\Phi}_u}
\newcommand{\Phih}{\begin{bmatrix}\Phixh \\ \Phiuh \end{bmatrix}}
\newcommand{\IP}[1]{\left\langle #1\right\rangle_{\mathcal{H}_2}}
\newcommand{\A}{\mathcal{A}}

\subsection{Suboptimality bounds for FIR truncated SLS} \label{sec:fir_subopt}
Optimization problem \eqref{eq:adaptiveLQRbnd} is convex but infinite dimensional, and as far as we are aware does not admit an efficient solution.  In Algorithm \ref{alg:adaptive}, we instead propose solving the following FIR approximation to problem \eqref{eq:adaptiveLQRbnd}:
\begin{align}
  \mathrm{minimize}_{\gamma \in [0, 1)} &\frac{1}{1-\gamma} \min_{\bmtx{\Phi_x}, \bmtx{\Phi_u}, V} \bightwonorm{ \begin{bmatrix} Q^{1/2} & 0 \\ 0 & R^{1/2} \end{bmatrix} \begin{bmatrix} \bmtx{\Phi_x} \\ \bmtx{\Phi_u} \end{bmatrix} }  \nonumber \\
    \qquad \text{s.t.} &\rvectwo{zI - \Ah}{-\Bh} \cvectwo{\bmtx{\Phi_x}}{\bmtx{\Phi_u}} = I + \frac{1}{z^F}V \:, \frac{\sqrt{2}\varepsilon}{1-C_x\rho^{F+1}}\bighinfnorm{ \cvectwo{ \bmtx{\Phi_x}}{ \bmtx{\Phi_u}}} \leq  \gamma \label{eq:adaptiveFIR}\\
    & \norm{V} \leq C_x\rho^{F+1} \:, \: \bmtx{\Phi_x} \in \frac{1}{z}\mathcal{RH}_\infty^{F}(C_x, \rho), \bmtx{\Phi_u} \in \frac{1}{z} \mathcal{RH}_\infty^{F}(C_u, \rho)\nonumber \:.
\end{align}
where here $F$ denotes the FIR truncation length used.  This optimization problem can be posed as a finite dimensional semidefinite program (see Section \ref{sec:app:sls_fir}).  Let $\tf K(F)$ denote the resulting controller.  We begin with a lemma identifying conditions under which optimization problem \eqref{eq:adaptiveFIR} is feasible.  to ease notation going forward, we let $\zeta := \varepsilon(1+\norm{\trueK})\hinfnorm{\res{\trueA+\trueB\trueK}}$.

\begin{lemma}\label{lem:FIR_feasibility}
Let the assumptions of Theorem~\ref{thm:lqr_feasibility_and_subopt} hold, and further assume that
\[
F_0 \geq \frac{\log(2C_x)}{\log(1/\rho)} - 1 \:.
\]
Then optimization problem \eqref{eq:adaptiveFIR} is feasible for any $F\geq F_0$.
\end{lemma}
\begin{proof}
We construct a feasible solution as follows.  Let $\Phix = \res{\Ah+\Bh\trueK}(1:F)$, $\Phiu = \trueK\res{\Ah+\Bh\trueK}(1:F)$, $V=\res{\Ah+\Bh\trueK}(F+1)$, and $\gamma = \frac{2\sqrt{2}\zeta}{1-\zeta}$.  First, the proposed $(\Phix,\Phiu)$ are FIR of length $F$, and hence, using the same arguments as in the proof of Theorem~\ref{thm:lqr_feasibility_and_subopt}, $\Phix\in\RHinf^F(C_x,\rho)$ and $\Phiu\in\RHinf^F(C_u,\rho)$.  It then also follows immediately that $\norm{V} = \norm{\res{\Ah+\Bh\trueK}(F+1)} \leq C_x\rho^{F+1}$.

Note that the affine constraint
\begin{equation}
\rvectwo{zI - \Ah}{-\Bh} \cvectwo{\bmtx{\Phi_x}}{\bmtx{\Phi_u}} = I + \frac{1}{z^F}V
\label{eq:affine}
\end{equation}
is equivalent to
\[
\Phi_x(t+1) = \Ah \Phi_x(t) + \Bh\Phi_u(t), \: \Phi_x(1) = I,
\]
for $1 \leq t < F$.  We have by construction that the proposed $\Phix$ and $\Phiu$ satisfy this constraint.  Further, the combination of the FIR constraints and the affine constraint \eqref{eq:affine} impose that
\[
\Phi_x(F+1) = \Ah \Phi_x(F) + \Bh\Phi_u(F) - V = 0.
\]
Now notice that for the proposed $(\Phix,\Phiu)$, we have that $\Ah\Phi_x(F) + \Bh\Phi_u(F) = (\Ah+\Bh\trueK)\res{\Ah+\Bh\trueK}(F) = \res{\Ah+\Bh\trueK}(F+1)$, where the last equality follows from the fact that $ \res{\Ah+\Bh\trueK}(t+1) = (\Ah+\Bh\trueK)^t$.  It follows that $\Phi_x(F+1) = 0$, as desired.

It remains to prove that
\[
\frac{\sqrt{2}\varepsilon}{1-C_x\rho^{F+1}}\bighinfnorm{ \cvectwo{ \bmtx{\Phi_x}}{ \bmtx{\Phi_u}}} \leq  \frac{2\sqrt{2}\zeta}{1-\zeta} < 1.
\]

The final inequality follows immediately from the assumption that $\zeta \leq 1/5$.  Further, note that
\[
\frac{\sqrt{2}\varepsilon}{1-C_x\rho^{F+1}}\bighinfnorm{ \cvectwo{ \bmtx{\Phi_x}}{ \bmtx{\Phi_u}}} \leq 2\sqrt{2}\varepsilon\bighinfnorm{ \cvectwo{ \res{\Ah+\Bh\trueK}}{ \trueK\res{\Ah+\Bh\trueK}}} \leq \frac{2\sqrt{2}\zeta}{1-\zeta},
\]
where the first inequality follows from the assumption on on $F_0$ and that the proposed $\Phix$ is a truncation of $\res{\Ah+\Bh\trueK}$and that the proposed $\Phiu$ is a truncation of $\trueK\res{\Ah+\Bh\trueK}$, and final inequality follows by applying the triangle inequality and the definition of $\zeta$.  This proves the result.
\end{proof}

Next, we use this to bound the suboptimality gap of the performance achieved by the controller implemented using the solutions of optimization problem \eqref{eq:adaptiveFIR}.

\begin{lemma}\label{lem:FIR_subopt}
Let the assumptions of Lemma \ref{lem:FIR_feasibility} hold.  Fix any $C_J > 0$, and further let
\[
  F \geq \frac{\log( (1 + C_J^{-1}) C_x )}{\log(1/\rho)} - 1 \:.
\]
Denote by $(\Phix(F),\Phiu(F),V(F),\gamma(F))$ the optimal solution to optimization problem \eqref{eq:adaptiveFIR}, and let $\tf K(F) = \Phiu(F)\Phix^{-1}(F)$.  Then
\begin{equation}
J(\trueA,\trueB,\tf K(F)) \leq (1+C_J)^2(1 + \const\varepsilon(1 + \norm{\trueK})\hinfnorm{\res{\trueA+\trueB\trueK}})^2J_\star.
\label{eq:FIR_subopt}
\end{equation}
\end{lemma}
\begin{proof}
Let
\[
\Dh := \rvectwo{\Delta_A}{\Delta_B}\cvectwo{\Phix(F)}{\Phiu(F)}\left(I+\frac{1}{z^F}V(F)\right)^{-1}.
\]
Further note that using a similar argument to that in the proof of Lemma 4.2 of \cite{dean17}, one can verify that
\[
\hinfnorm{\Dh} \leq \frac{\sqrt{2}\varepsilon}{1-C_x\rho^{F+1}}\bighinfnorm{\cvectwo{\Phix(F)}{\Phiu(F)}} \leq \gamma(F),
\]
where we have exploited that $(\Phix(F),\Phiu(F),V(F),\gamma(F))$ form a feasible solution to optimization problem \eqref{eq:adaptiveFIR}.

Then, repeated application of Theorem \ref{thm:robust} tells us that the performance achieved by $K(F)$ on the true system is given by
\begin{align*}
\sqrt{J(\trueA,\trueB,\tf K(F))} &= \bightwonorm{\Cost \cvectwo{\Phix(F)}{\Phiu(F)}\left(I+\frac{1}{z^F}V(F)\right)^{-1}(I+\Dh)^{-1}} \\
& \leq \frac{1}{1-C_x\rho^{F+1}} \frac{1}{1-\gamma(F)}\bightwonorm{\Cost \cvectwo{\Phix(F)}{\Phiu(F)}},
\end{align*}
where the inequality follows from $\hinfnorm{\Dh} \leq \gamma(F)<1$, and $\norm{V(F)}\leq 1/2$ (by the assumption of $F\geq F_0$.

Denote by $(\Phix,\Phiu,V,\gamma_0)$ the feasible solution constructed in the proof of Lemma \ref{lem:FIR_feasibility}.  Then,
\begin{align*}
 \frac{1}{1-C_x\rho^{F+1}} \frac{1}{1-\gamma(F)}\bightwonorm{\Cost \cvectwo{\Phix(F)}{\Phiu(F)}} &\leq  \frac{1}{1-C_x\rho^{F+1}} \frac{1}{1-\gamma_0}\bightwonorm{\Cost \cvectwo{\Phix}{\Phiu}} \\
 &= \frac{1}{1-C_x\rho^{F+1}} \frac{1}{1-\gamma_0} \sqrt{J_F(\Ah,\Bh,\trueK)} \\
 & \leq \frac{1}{1-C_x\rho^{F+1}} \frac{1}{1-\gamma_0} \sqrt{J(\Ah,\Bh,\trueK)} \\
 & \leq \frac{1}{1-C_x\rho^{F+1}} \frac{1}{1-\gamma_0} \frac{1}{1-\zeta}\sqrt{J_\star},
\end{align*}
where the first inequality follows from the optimality of  $(\Phix(F),\Phiu(F),V(F),\gamma(F))$, the equality and second inequality from the fact that $(\Phix,\Phiu)$ are truncations of the response of $\trueK$ on $(\Ah,\Bh)$ to the first $F$ time steps, and the final inequality by following similar arguments to the proof of Theorem 4.1 in \cite{dean17} in applying Theorem \ref{thm:robust} and noting that
\[
\bighinfnorm{\rvectwo{\Delta_A}{\Delta_B}\cvectwo{\res{\Ah+\Bh\trueK}}{\trueK\res{\Ah+\Bh\trueK}}} \leq \zeta < 1.
\]

We therefore have that
\[
\sqrt{J(\trueA,\trueB,\tf K(F))} \leq \frac{1}{1-C_x\rho^{F+1}} \frac{1}{1-\gamma_0} \frac{1}{1-\zeta}\sqrt{J_\star} \leq
(1+C_J)  \frac{1}{1-\gamma_0} \frac{1}{1-\zeta}\sqrt{J_\star},
\]
where the last inequality follows from the assumptions on $F$ stated in the Lemma.  Finally, by assumption $\zeta \leq 1/5 <.8(1+2\sqrt{2})^{-1}$, from which it follows that $(1-\gamma_0)^{-1}(1-\zeta)^{-1} \leq 1 + 20\zeta$, leading to the bound
\[
\sqrt{J(\trueA,\trueB,\tf K(F))} \leq (1+C_J)(1+20\zeta)\sqrt{J_\star}.
\]
Squaring both sides proves the result.
\end{proof}

The following Theorem is then immediate.

\begin{theorem}
\label{thm:lqr_feasibility_and_subopt_fir}
Let $J_\star$ denote the minimal LQR cost achievable by any controller
for $(\trueA, \trueB)$. Let $\trueK$ denote the optimal controller and
suppose that $\Res{\trueA + \trueB \trueK} \in \RHinf(C_\star, \rho_\star)$.  Fix a $C_J>0$, and suppose that $F_0$ and $\varepsilon$ satisfy the assumptions of Lemmas \ref{lem:FIR_feasibility} and \ref{lem:FIR_subopt}.
Let $(\Ah,\Bh)$ be any estimates of the transition matrices such that $\max\{\norm{\Delta_A},\norm{\Delta_B}\} \leq \varepsilon$. Then, if $(C_x,\rho)$ and $(C_u,\rho)$ are set as in Lemma \ref{lem:FIR_feasibility},
we have that (a) the program \eqref{eq:adaptiveFIR} is feasible for any truncation length $F\geq F_0$,
(b) letting $\tf K(F)$ denote an optimal solution to \eqref{eq:adaptiveFIR} for truncation length $F$,
the relative error in the LQR cost is
\begin{align} \label{eq:lqr_bound_fir}
J(\trueA,\trueB,\tf K(F)) \leq (1+C_J)^2(1 + \const\varepsilon(1 + \norm{\trueK})\hinfnorm{\res{\trueA+\trueB\trueK}})^2J_\star \, ,
\end{align}

and (c) if furthermore $\varepsilon(C_x + C_u) \leq \const (1 - \rho_\star)^2$,
the response $\{ \tf{\widehat{\Phi}}_x, \tf{\widehat{\Phi}}_u \}$ of $\tf K$
on the true system $(\trueA, \trueB)$ satisfies
\begin{align*}
    \tf{\widehat{\Phi}}_x \in \RHinf\left(\frac{\const C_\star}{(1-\rho_\star)^3}, .999 + .001 \rho_\star\right) \:, \\
    \tf{\widehat{\Phi}}_u \in \RHinf\left(\frac{\const \norm{\trueK} C_\star}{(1-\rho_\star)^3}, .999 + .001 \rho_\star\right) \:.
\end{align*}
\end{theorem}
\begin{proof}
Claims (a) and (b) follow immediately from Lemmas \ref{lem:FIR_feasibility} and \ref{lem:FIR_subopt}.

Now for the proof of (c).
Let $\{ \tf \Phi_x(F), \tf \Phi_u(F) \}$ be the solution to \eqref{eq:adaptiveFIR}.  Then as argued in the proof of
Lemma \ref{lem:FIR_subopt}, the response achieved on the true system $(\trueA,\trueB)$ is given by
\[
\cvectwo{\Phix(F)}{\Phiu(F)}\left(I+\frac{1}{z^F}V(F)\right)^{-1}(I+\Dh)^{-1},
\]
where $\Dh$ is defined as in the proof of Lemma~\ref{lem:FIR_subopt}.

We start by noting that $\Phix(F) \in \RHinf(C_x,\rho)$, and by the assumption on $F\geq F_0$, it holds that $z^{-F}V(F) \in \RHinf(2,\rho^{1/2})$.  This allows us to apply Lemma~\ref{lem:inverse_rhinf_direct} to conclude that $(I+z^{-F}V(F))^{-1} \in \RHinf\left(\const(1-\rho^{1/2})^{-1},\mathsf{Avg}(\rho^{1/2},1)\right)$.  Thus, applying Lemma~\ref{lem:G1G2} we conclude that
\[
\Phix(F)\left(I+\frac{1}{z^F}V(F)\right)^{-1} \in \RHinf\left(\frac{\const C_x}{1-\rho^{1/2}}, \mathsf{Avg}(\mathsf{Avg}(\rho^{1/2},1),1)\right).
\]
A similar argument yields
\[
\Phiu(F)\left(I+\frac{1}{z^F}V(F)\right)^{-1} \in \RHinf\left(\frac{\const C_u}{1-\rho^{1/2}}, \mathsf{Avg}(\mathsf{Avg}(\rho^{1/2},1),1)\right).
\]

Now note that
\[
\Dh = (\Delta_A \Phix(F) + \Delta_B \Phiu(F))(I + z^{-F}V(F))^{-1}.
\]
From the previous argument, we have that
\begin{align*}
\Delta_A \Phix(F) (I + z^{-F}V(F))^{-1} &\in \RHinf\left(\varepsilon\frac{\const C_x}{1-\rho^{1/2}}, \mathsf{Avg}(\mathsf{Avg}(\rho^{1/2},1),1)\right),\\
 \Delta_B \Phiu(F) (I + z^{-F}V(F))^{-1}& \in \RHinf\left(\varepsilon\frac{\const C_u}{1-\rho^{1/2}}, \mathsf{Avg}(\mathsf{Avg}(\rho^{1/2},1),1)\right),
\end{align*}
from which it follows that
\[
\Dh \in \RHinf\left( \varepsilon \frac{\const(C_x + C_u)}{1-\rho^{1/2}}, \mathsf{Avg}(\mathsf{Avg}(\rho^{1/2},1),1)\right).
\]
By the assumptions of the Theorem, we have that $\varepsilon \frac{\const(C_x + C_u)}{1-\rho^{1/2}} \leq 2$, allowing us to apply Lemma~\ref{lem:inverse_rhinf_direct} to conclude that
\[
(I+\Dh)^{-1} \in \RHinf\left(\const, \mathsf{Avg}(\mathsf{Avg}(\mathsf{Avg}(\rho^{1/2},1),1),1)\right).
\]
Applying Lemma~\ref{lem:G1G2}, we see that
\begin{align*}
\Phix(F) (I + z^{-F}V(F))^{-1}(I+\Dh)^{-1} &\in \RHinf\left(\frac{\const C_x}{1-\rho^{1/2}},\mathsf{Avg}(\mathsf{Avg}(\mathsf{Avg}(\mathsf{Avg}(\rho^{1/2},1),1),1)1)\right) \\
\Phiu(F) (I + z^{-F}V(F))^{-1}(I+\Dh)^{-1} &\in \RHinf\left(\frac{\const C_u}{1-\rho^{1/2}},\mathsf{Avg}(\mathsf{Avg}(\mathsf{Avg}(\mathsf{Avg}(\rho^{1/2},1),1),1)1)\right)
\end{align*}

Finally, to simplify these bounds to those in the Theorem statement, notice first that for $\rho\geq .4$, we have that $(1-\rho^{1/2})>(1-\rho)^2$.  Then, we also have that
\[
\mathsf{Avg}(\mathsf{Avg}(\mathsf{Avg}(\mathsf{Avg}(\rho^{1/2},1),1),1)1) = \frac{31}{32} + \frac{1}{32}\rho^{1/2} =  \frac{31}{32} + \frac{1}{32}(\frac{1}{4}\rho_\star + \frac{3}{4})^{1/2}.
\]
Finally, one can check that for $\rho_\star \geq .4$, we have that $(\frac{1}{4}\rho_\star + \frac{3}{4})^{1/2} \leq .95 + .05 \rho_\star$, leading to the bound
\[
\frac{31}{32} + \frac{1}{32}(\frac{1}{4}\rho_\star + \frac{3}{4})^{1/2} \leq \frac{31.95}{32} + \frac{.05}{32}\rho_\star \leq .999 + .001\rho_\star.
\]
We note that these constants are by no means optimized.
\end{proof}

\section{Estimation}

Recall that Algorithm~\ref{alg:adaptive} proceeds in epochs and that we denote by $x_t^{(i)}$ and $u_{t}^{(i)}$ the state
and input at time $t$ during epoch $i$, respectively. The $i$-th epoch has length $T_i$. Note that $x_{T_i}^{(i)}$, the last state of epoch $i$,
is equal to $x_{0}^{(i + 1)}$, the first state of epoch $i + 1$.

At the end of each epoch our method estimates the parameters $(\trueA, \trueB)$ from the trajectory observed during that epoch, i.e.
\begin{align}
\label{eq:app_ols}
(\Ah, \Bh) \in \arg\min_{A, B} \sum_{t=0}^{T_i-1} \frac{1}{2}\ltwonorm{x_{t+1}^{(i)} - A x_t^{(i)} - B u_t^{(i)}}^2.
\end{align}

The goal of this section is to offer high probability confidence bounds on the estimation error of \eqref{eq:app_ols}.
For the rest of the section we suppress the dependence on the epoch index $i$ because we prove a statistical rate
for a fixed epoch.

Algorithm~\ref{alg:adaptive} generates the inputs $u_t$ using a feedback
controller $\tf K$ which stabilizes the true system $(\trueA, \trueB)$. Let $\{
  \tf \Phi_x, \tf \Phi_u \}$ denote the response of $\tf K$ on the true system
$(\trueA, \trueB)$, and suppose that $\tf \Phi_x \in \frac{1}{z}\RHinf(C_x,
\rho)$ and $\tf \Phi_u \in \frac{1}{z}\RHinf(C_u, \rho)$.
More precisely, if $w_t \iid \calN(0, \sigma_w^2 I_\inputdim)$ is the process
noise at time $t$ and $\eta_t \iid \calN(0, \sigma_\eta^2 I_\inputdim)$ is the input noise added at time $t$,
then we can write
\begin{align}
\label{eq:states}
x_t &= \Phi_x (t + 1) x_0 + \sum_{k = 0}^{t - 1} \Phi_x(t - k) (\trueB \eta_{k} + w_{k})\\
\label{eq:inputs}
u_t &= \eta_t + \Phi_u (t + 1) x_0 + \sum_{k = 0}^{t - 1} \Phi_u(t - k) (\trueB \eta_{k} + w_{k}).
\end{align}

For the statistical analysis it is useful to consider the stochastic process  $z_t = [x_t^\T, u_t^\T]^\T$.
Also, we denote the filtration $\calF_t = \sigma(x_0, \eta_0, w_0 \ldots, \eta_{t - 1}, w_{t - 1}, \eta_t)$. It is clear that the process $\{z_t\}_{t \geq 0}$ is $\{\calF_t\}_{t \geq 0}$-adapted.
Throughout this section we assume that $C_u, C_x \geq 1$ and denote $C_K^2 := \statedim C_x^2 + \inputdim C_u^2$.

\subsection{Estimation after one epoch}

Throughout this section we assume that $\sigma_\eta \leq \sigma_w$. This condition is not needed for achieving the necessary statistical
rate of estimation of $(A, B)$, but it aids in simplifying several algebraic quantities.

\begin{proposition}
\label{prop:one_epoch_estimate}
Let $x_0 \in \R^{\statedim}$ be any initial state, let $\sigma_\eta \leq \sigma_w$, and assume that a trajectory $\{(x_t, u_t)\}_{t = 0}^{T - 1}$ is observed.
Furthermore, suppose the inputs $u_t \in \R^\inputdim$ are generated
by a feedback controller $\tf K$ which stabilizes and achieves a
response $\{ \tf \Phi_x, \tf \Phi_u \}$ on $(\trueA, \trueB)$
  with $\tf \Phi_x \in \frac{1}{z}\RHinf(C_x, \rho)$ and $\tf \Phi_u \in \frac{1}{z}\RHinf(C_u, \rho)$.  Then, the error of the OLS estimator $(\Ah, \Bh)$ from Eq.~\ref{eq:app_ols} satisfies with probability $1 - \delta$ the guarantee
\begin{align*}
\max\left\{ \substack{\norm{\Ah - \trueA},\\ \norm{\Bh - \trueB}}\right\} \lesssim \frac{\sigma_w C_u}{\sigma_\eta} \sqrt{\frac{(\statedim + \inputdim)}{T} \log\left(1 + \frac{\inputdim C_u}{\delta} + \frac{\sigma_w}{\sigma_\eta}\frac{\rho C_u C_K}{\delta(1 - \rho^2)} \left( 1 + \norm{\trueB}  + \frac{\ltwonorm{x_0} }{\sigma_w\sqrt{T}} \right) \right)},
\end{align*}
as long as
\begin{align}
    T \gtrsim (\statedim + \inputdim) \log\left(1 + \frac{\inputdim C_u^2}{\delta} + \frac{\sigma_w^2}{\sigma_\eta^2}\frac{\rho^2 C_u^2 C_K^2}{\delta(1 - \rho^2)}\left(1 + \norm{\trueB}^2 + \frac{\ltwonorm{x_0}^2}{\sigma_w^2 T}\right)\right). \label{eq:est_T_cond}
\end{align}
\end{proposition}

The proof of this result follows from a result by \citet{simchowitz18} on the estimation of linear response time-series. We present that result in the context of our problem. Let $M_\star = [\trueA, \trueB]$, and recall that $z_t = [x_t^\T, y_t^\T]^\T$. Then, the OLS estimator \eqref{eq:app_ols} can be written in the form
\begin{align}
\label{eq:ols_M}
\widehat{M} \in \arg \min_{M} \sum_{t = 0}^{T - 1} \frac{1}{2}\ltwonorm{x_{t + 1} - Mz_t}^2.
\end{align}

The process $\{z_t\}_{t \geq 0}$ is said to satisfy the $(k, \nu, \beta)$-\emph{block martingale small-ball} (BMSB) condition if for any $j \geq 0$ and $v \in \R^{\statedim + \inputdim}$, one has that
\begin{align*}
\frac{1}{k}\sum_{i = 1}^k \Pr\left( |\langle v, z_{j + i}\rangle| \geq \nu \right) \geq \beta \: \text{ almost surely.}
\end{align*}
This condition is used for characterizing the size of the minimum eigenvalue of the covariance matrix $\sum_{t = 0}^{T - 1} z_t z_t^\top$.
A larger $\nu$ guarantees a larger lower bound of the minimum eigenvalue. In the context of our problem the result by \citet{simchowitz18} translates as follows.

\begin{theorem}[\citet{simchowitz18}]
\label{thm:lwm}
Fix $\epsilon, \delta \in (0,1)$. For every $T$, $k$, $\nu$, and $\beta$ such that $\{z_t\}_{t \geq 0}$ satisfies the $(k, \nu, \beta)$-BMSB and
\begin{align*}
\left\lfloor \frac{T}{k} \right\rfloor \gtrsim \frac{\statedim + \inputdim}{\beta^2} \log \left(1 + \frac{\sum_{t = 0}^{T - 1} \Tr (\E z_t z_t^\T)}{k \lfloor T / k \rfloor \beta^2 \nu^2 \delta} \right),
\end{align*}
the estimate $\widehat{M}$ defined in Eq.~\ref{eq:ols_M} satisfies the following statistical rate
\begin{align*}
\Pr \left(\ltwonorm{\widehat{M} - M} > \frac{\const \sigma_w}{\beta \nu} \sqrt{\frac{\statedim + \inputdim}{k \lfloor T / k \rfloor} \log \left(1 + \frac{\sum_{t = 0}^{T - 1} \Tr (\E z_t z_t^\T)}{k \lfloor T / k \rfloor \beta^2 \nu^2 \delta} \right) }\right) \leq \delta.
\end{align*}
\end{theorem}

Therefore, in order to apply this result we need to find $k$, $\nu$, and $\beta$ such that $\{z_t\}_{t \geq 0}$ satisfies the
$(k, \nu, \beta)$-BMSB condition, and we also have to upper bound the trace of the covariance of $z_t$. The next two lemmas address these two issues.

\begin{lemma}
\label{lem:bmsb}
Let $x_0$ be any initial state in $\R^\statedim$ and let $\{u_t\}_{t \geq 0}$ be the sequence of inputs
generated according to \eqref{eq:inputs}, and assume $\sigma_\eta \leq \sigma_w$. Then, the process $z_t = [x_t^\T, u_t^\T]^\T$
satisfies the

\begin{align*}
\left(1, \frac{\sigma_\eta}{ 2C_u} , \frac{3}{20}\right)\text{ BMSB condition.}
\end{align*}
\end{lemma}
\begin{proof}
For all $t \geq 1$, denote
\begin{align*}
\xi_t &= u_t - \eta_t - \Phi_u(1) w_{t - 1} \\
&=  \Phi_u(t + 1) x_0 + \sum_{k = 0}^{t - 2} \Phi_u(t - k) (\trueB \eta_{k}+ w_{k}) + \Phi_u(1) \trueB \eta_{t - 1}.
\end{align*}

Therefore, we have
\begin{align*}
\begin{bmatrix}
x_{t + 1}\\
u_{t + 1}
\end{bmatrix} = \begin{bmatrix}
\trueA x_{t} + \trueB u_{t} \\
\xi_{t + 1}
\end{bmatrix} +
\begin{bmatrix}
I_\statedim & 0 \\
\Phi_u(1) & I_\inputdim
\end{bmatrix}\begin{bmatrix}
w_t \\
\eta_{t + 1}
\end{bmatrix},
\end{align*}
and hence
\begin{align*}
\begin{bmatrix}
x_{t + 1}\\
u_{t + 1}
\end{bmatrix} \left| \calF_t \right. \sim
\calN\left(
\begin{bmatrix}
\trueA x_{t} + \trueB u_{t} \\
\xi_{t + 1}
\end{bmatrix},
\begin{bmatrix}
\sigma_w^2 I_\statedim & \sigma_w^2 \Phi_u(1)^\top \\
\sigma_w^2 \Phi_u(1) & \sigma_w^2 \Phi_u(1)\Phi_u(1)^\top  + \sigma_\eta^2 I_\inputdim
\end{bmatrix}
\right).
\end{align*}
Denote by $\mu_{z, t}$ and $\Sigma_z$ the mean and covariance of this multivariate normal distribution. Recall that we denote $z_t = [x_t^\T, u_t^\T]^\T$.
Let $v \in \R^{\statedim + \inputdim}$ and then $\langle v, z_t \rangle \sim \calN(\langle v, \mu_{z, t} \rangle, v^\top \Sigma_z v)$. Therefore,
\begin{align*}
\Pr \left(|\langle v, z_t \rangle | \geq \sqrt{\lambda_{\min}(\Sigma_z)}\right) &\geq \Pr\left(|\langle v, z_t \rangle | \geq \sqrt{ v^\top \Sigma_z v}\right)\\
&\geq \Pr\left(|\langle v, z_t - \mu_{z, t} \rangle | \geq \sqrt{ v^\top \Sigma_z v}\right) \geq 3/10,
\end{align*}
where the last two inequalities follow because for any $\mu, \sigma^2 \in \R$ and $\omega \sim \calN(0, \sigma^2)$ we have
\begin{align*}
\Pr(|\mu + \omega| \geq \sigma) \geq \Pr(|\omega| \geq \sigma) \geq 3/10.
\end{align*}

Since $\Phi_u \in \frac{1}{z}\RHinf(C_u, \rho)$ we have $\norm{\Phi_u(1)} \leq C_u$.
Then, by a simple argument based on a Schur complement (detailed in Lemma~\ref{lem:lambda_min_block_matrix}) it follows that
\begin{align*}
\lambda_{\min}(\Sigma_z) \geq \sigma_\eta^2 \min\left(\frac{1}{2}, \frac{\sigma_w^2}{2 \sigma_w^2 C_u^2  + \sigma_\eta^2} \right) \:.
\end{align*}
The conclusion follows since $C_u \geq 1$.
\end{proof}

\begin{lemma}
\label{lem:bound_covariance}
Let $\sigma_\eta \leq \sigma_w$. Then, the process $z_t = [x_t^\T, u_t^\T]^\T$ satisfies
\begin{align*}
\sum_{t = 0}^{T - 1} \Tr \left( \E z_t z_t^\T\right) \leq \sigma_\eta^2 \inputdim  T + \sigma_w^2 \frac{\rho^2 C_K^2 T}{(1 - \rho^2)}\left(1 + \norm{\trueB}^2 + \frac{\ltwonorm{x_0}^2}{\sigma_w^2 T}\right).
\end{align*}
\end{lemma}
\begin{proof}

Now, note that
\begin{align*}
\E z_t z_t^\T &= \begin{bmatrix}
\Phi_x(t + 1)\\
\Phi_u(t + 1)
\end{bmatrix}
x_0 x_0^\T
\begin{bmatrix}
\Phi_x(t + 1)\\
\Phi_u(t + 1)
\end{bmatrix}^\T
+ \begin{bmatrix}
0 & 0\\
0 & \sigma_\eta^2 I_\inputdim
\end{bmatrix}
\\
&\quad  + \sum_{k = 0}^{t - 1} \begin{bmatrix}
\Phi_x(t - k) \\
\Phi_u(t - k)
\end{bmatrix}
(\sigma_\eta^2 \trueB \trueB^\T + \sigma_w^2 I_\statedim)
\begin{bmatrix}
\Phi_x(t - k) \\
\Phi_u(t - k)
\end{bmatrix}.
\end{align*}

Since for all $j \geq 1$ we have $\norm{\Phi_x(j)} \leq C_x \rho^j$ and $\norm{\Phi_u(j)} \leq C_u \rho^j$, we obtain
\begin{align*}
\Tr \E z_t z_t^\T \leq \inputdim \sigma_\eta^2 + (\statedim C_x^2 + \inputdim C_u^2)\left(\rho^{2t + 2}\ltwonorm{x_0}^2 + (\sigma_w^2 + \sigma_\eta^2 \norm{\trueB}^2) \sum_{k = 1}^{t}  \rho^{2k}\right)
\end{align*}

Therefore, we get that
\begin{align*}
\sum_{t = 0}^{T - 1} \Tr\E z_t z_t^\top \leq \inputdim \sigma_\eta^2 T + \frac{\rho^2 T}{1 - \rho^2}(\statedim C_x^2 + \inputdim C_u^2)(\sigma_w^2 + \sigma_\eta^2 \norm{\trueB}^2) + \frac{\rho^2}{1 - \rho^2}(\statedim C_x^2 + \inputdim C_u^2)\ltwonorm{x_0}^2,
\end{align*}
and the conclusion follows by simple algebra.
\end{proof}

Proposition~\ref{prop:one_epoch_estimate} follows from Theorem~\ref{thm:lwm}, Lemma~\ref{lem:bmsb}, Lemma~\ref{lem:bound_covariance}, and simple algebra.

\subsection{Stitching the epochs together}

We start by bounding with high probability the size of the initial states of the epochs.
Recall that epoch $i$ has length $T_i$ and that we denote by $x_{T_i}^{(i)}$ the last
state of the epoch $i$, which is equal to the first state $x_{0}^{(i + 1)}$ of the epoch $i + 1$.
For simplicity we assume that $x_0^{(0)} = 0$, an assumption that is not restrictive in any way.

 \begin{lemma}
\label{lem:initial_state}
Fix $\delta \in (0,1)$, $r > 0$, and an epoch $i$. Assume that for all $k \leq i$ the epoch length $T_k$ is large enough so that $C_x \rho^{T_k} \leq \rho^r$. Then, for any $t \geq 0$ we have
\begin{align*}
\Pr\left(\ltwonorm{x_0^{(i + 1)}} \geq \sigma_w (\sqrt{\statedim} + t) \frac{ C_x \rho (1 + \norm{\trueB})}{(1 - \rho^r)(1 - \rho^2)}  \right) \leq \exp\left(- \frac{t^2}{2}\right).
\end{align*}
\end{lemma}
\begin{proof}
From Eq. \eqref{eq:states} we have that
\begin{align*}
x_{0}^{(i + 1)} = \Phi_x^{(i)}(T_i + 1) x_{0}^{(i)} + \underbrace{\sum_{j = 0}^{T_i - 1} \Phi_x^{(i)}(T_i - 1 - j) (\trueB \eta_j^{(i)} + w_j^{(i)})}_{\xi_{i}},
\end{align*}
where we denoted the sum over disturbances during the epoch $i$ by $\xi_{i}$. Therefore,
\begin{align*}
\ltwonorm{x_0^{(i + 1)}} &\leq C_x \rho^{T_i}\ltwonorm{x_0^{(i)}} + \ltwonorm{\xi_{i}}\\
&\leq \rho^r \ltwonorm{x_0^{(i)}} + \ltwonorm{\xi_{i}}\\
&\leq \sum_{k = 0}^{i}\rho^{r(i - k)}\ltwonorm{\xi_{k}}.
\end{align*}

By definition $\xi_{k}$ is a zero-mean multivariate Gaussian random vector with covariance
\begin{align*}
\Sigma_{x,k} := \sum_{j = 0}^{T_k - 1} \Phi^{(k)}_x(T_k - 1 - j)(\sigma_w^2 + \sigma_{\eta, k}^2 \trueB \trueB^\T) \Phi^{(k)}_x(T_k - 1 - j)^\T,
\end{align*}
whose top eigenvalue is upper bounded by
\begin{align}
\nonumber
\sum_{j = 0}^{T_k - 2} C_x^2 (\sigma_w^2 + \sigma_{\eta, k}^2 \norm{\trueB}^2) \rho^{2(T_k - 1 - j)} &\leq (\sigma_w^2 + \sigma_{\eta, k}^2 \norm{\trueB}^2) \frac{C_x^2 \rho^2}{1 - \rho^2}\\
\label{eq:Lipschitz_constant}
&\leq \sigma_w^2 (1 + \norm{\trueB}^2) \frac{C_x^2 \rho^2}{1 - \rho^2},
\end{align}
where the last inequality follows because $\sigma_{\eta, k} \leq \sigma_w$.

Then, we can write $\ltwonorm{\xi_k}$ as $\ltwonorm{\Sigma_{x,k}^{1/2}\omega_k}$, where $\omega_k$ is a standard Gaussian random vector distributed according to $\calN(0, I_\statedim)$, and hence $\ltwonorm{\xi_i}$ is a Lipschitz function of $\omega_i$ with Lipschitz constant equal to squared root of \eqref{eq:Lipschitz_constant}. Hence, $\ltwonorm{x_0^{(i)}}$ is a Lipschitz function of standard normal random variables with the Lipschitz constant
\begin{align*}
\sqrt{\sigma_w^2 \frac{(1 + \norm{\trueB}^2)}{1 - \rho^r} \frac{C_x^2 \rho^2}{1 - \rho^2}}.
\end{align*}

By the concentration of Lipschitz functions of isotropic Gaussians, for $\nu \geq 0$, we have that
\begin{align*}
\Pr\left(\ltwonorm{x_0^{(i + 1)}} \geq \E \ltwonorm{x_0^{(i + 1)}} + \nu \right) \leq \exp\left(- \frac{\nu^2(1 - \rho^2)(1 - \rho^r)}{2 \sigma_w^2 \rho^2 (1 + \norm{\trueB}^2) C_x}\right).
\end{align*}
By Jensen's inequality we have that
\begin{align*}
\E \ltwonorm{x_0^{(i + 1)}} &\leq \sqrt{\E \ltwonorm{x_0^{(i + 1)}}^2} \leq \sqrt{\sum_{k = 0}^{i} \rho^{r(i - k)}\Tr \left( \E \xi_k \xi_k^\T \right)}\\
&\leq \sqrt{ \statedim \sigma_w^2 \frac{(1 + \norm{\trueB}^2)}{1 - \rho^r}\frac{C_x^2 \rho^2}{1 - \rho^2}}.
\end{align*}
The conclusion follows.
\end{proof}

We are now ready to prove that the statistical rate holds across epochs.
In order to achieve this, we need the statistical rate after the first epoch to
be small enough to satisfy the feasibility constraints on $\varepsilon$
given in Theorem~\ref{thm:lqr_feasibility_and_subopt} for the IIR case
and given in Theorem~\ref{thm:lqr_feasibility_and_subopt_fir}
for the FIR truncated case. Once this occurs, we immediately have
feasibility at the next epoch (w.h.p.), and iterating the argument gives us
recursive feasibility (w.h.p.).

\begin{theorem}
\label{thm:estimation}
Fix a $\delta \in (0, 1)$.
For the IIR case, let $C_x, C_u, \rho$ be defined as
\begin{align*}
  C_x &= \frac{\const C_\star}{(1-\rho_\star)^2} \:, \\
  C_u &= \frac{\const \norm{\trueK} C_\star}{(1-\rho_\star)^2} \:, \\
  \rho &= (1/8) \rho_\star + (7/8) \:,
\end{align*}
and for the FIR case, let $C_x, C_u, \rho$ be defined as
\begin{align*}
C_x &= \frac{\const C_\star}{(1-\rho_\star)^3} \:, \\
C_u &=  \frac{\const \norm{\trueK}C_\star}{(1-\rho_\star)^3} \:, \\
\rho &= 0.001 \rho_\star + .999 \:,
\end{align*}
where $(C_\star, \rho_\star)$ are as defined in Theorem~\ref{thm:lqr_feasibility_and_subopt}
(resp. Theorem~\ref{thm:lqr_feasibility_and_subopt_fir}),
and suppose (wlog) that $C_x \geq 1$ and $C_u \geq 1$.
Let the length of epoch $i \in \{ 0, 1, 2, ... \}$ be $T_i = C_T 2^i$ time steps and let
the injected noise variance at epoch $i$ be $\sigma_{\eta, i}^2 = \sigma_w^2 2^{-i/3}$.
Suppose the constant $C_T$ is large enough to satisfy the following inequalities,
\begin{align}
    C_T &\geq \frac{\log(2C_x)}{\log(1/\rho)}  \:, \label{eq:first_ineq} \\
    C_T &\gtrsim \frac{1}{2^i}\left(n + \log\left( \frac{i+1}{\delta} \right)\right)\text{ for all } i=0, 1, 2, ... \:, \label{eq:second_ineq} \\
    C_T &\gtrsim \frac{(\statedim + \inputdim)}{2^i} \log\left(1 + (i+1)^2\frac{\inputdim C_u^2}{\delta} + (i+1)^2 2^{i/3}\frac{\rho^2 C_u^2 C_K^2}{\delta(1 - \rho^2)}\left( \frac{C_x^2(1 + \norm{\trueB})^2}{(1-\rho)^2}\right)\right) \label{eq:third_ineq} \\
        &\qquad\text{ for all } i =0, 1, 2, ... \:, \nonumber \\
    C_T &\gtrsim \frac{(n+p)}{2^{2i/3}} \frac{C_u^2(C_x + C_u)^2}{(1-\rho_\star)^\alpha}\label{eq:fourth_ineq} \\
    &\qquad\times \log\left( 1 + (i+1)\frac{p C_u}{\delta} + (i+1) 2^{i/6} \frac{\rho C_u C_K}{\delta(1-\rho^2)}\left( \frac{C_x(1 + \norm{\trueB})}{1-\rho} \right)   \right) \nonumber \\
    &\qquad\text{ for all } i=0, 1, 2, ... \:, \nonumber
\end{align}
where above $\alpha = 2$ for the IIR case and $\alpha = 4$ for the FIR case.
Then, with probability $1 - \delta$, the following two statements hold.
First, for all epochs $i$, the norm of the first state at the beginning
of each epoch satisfies
\begin{align}
  \ltwonorm{x_0^{(i)}} \lesssim \sigma_w \left(\sqrt{\statedim} + \sqrt{\log\left(\frac{i + 1}{\delta}\right)} \right) \frac{ C_x \rho (1 + \norm{\trueB})}{1 - \rho^2} \:. \label{eq:est_x_begin_bound}
\end{align}
Second, for all epochs $i$, the OLS estimate $(\Ah^{(i)}, \Bh^{(i)})$ satisfies the statistical rate
\begin{align}
  \max\left\{ \substack{\norm{\Ahat^{(i)} - A},\\ \norm{\Bhat^{(i)} - B}}\right\} \lesssim \frac{\sigma_w C_u}{\sigma_{\eta, i}} \sqrt{\frac{(\statedim + \inputdim)}{T_i} \log\left(1 + (i+1)\frac{\inputdim C_u}{\delta} + (i+1)\frac{\sigma_w}{\sigma_{\eta, i}}\frac{\rho C_u C_K}{\delta(1 - \rho^2)} \left( \frac{C_x(1 + \norm{\trueB})}{1-\rho} \right) \right)} \:. \label{eq:est_bound}
\end{align}
\end{theorem}
\begin{proof}
For this proof, we set $r = \log(2)/\log(1/\rho)$.

By Theorem~\ref{thm:lqr_feasibility_and_subopt} for the IIR case
and Theorem~\ref{thm:lqr_feasibility_and_subopt_fir} for the FIR case, we know that
the true responses $\{ \tf \Phi_x, \tf \Phi_u \}$ of the synthesized controllers
$\tf K_i$ on $(\trueA, \trueB)$ at every epoch satisfy $\tf \Phi_x \in \RHinf(C_x, \rho)$
and $\tf \Phi_u \in \RHinf(C_u, \rho)$.

Because of the assumption \eqref{eq:first_ineq} on $C_T$  we have $C_x \rho^{T_i} \leq \rho^r$. Therefore, we can apply Lemma~\ref{lem:initial_state} with $t^2 = \log(\const (i + 1)^2/\delta)$ to obtain that with probability at least $1 - \delta / 2$ the norm of $x_0^{(i)}$ for all epochs $i$ satisfies
\begin{align*}
\ltwonorm{x_0^{(i)}} \lesssim \sigma_w \left(\sqrt{\statedim} + \sqrt{\log\left(\frac{i + 1}{\delta}\right)} \right) \frac{ C_x \rho (1 + \norm{\trueB})}{(1 - \rho^r)(1 - \rho^2)}.
\end{align*}

Furthermore, by the assumption \eqref{eq:second_ineq} on $C_T$ we have that with probability at least $1 - \delta/2$,
\begin{align*}
  \frac{\ltwonorm{x_0^{(i)}}^2}{\sigma_w^2 T_i} \leq \frac{C_x^2 (1 + \norm{\trueB})^2}{(1-\rho)^2} \:.
\end{align*}

Our assumption \eqref{eq:third_ineq} means that
condition \eqref{eq:est_T_cond} is satisfied for each epoch $i$ and therefore
under the assumption the SLS program is feasible at every iteration,
we can invoke Proposition~\ref{prop:one_epoch_estimate} with $\delta = \const \delta/(i+1)^2$
and reach the desired conclusions.

To show feasibility of the SLS at every epoch,
Theorem~\ref{thm:lqr_feasibility_and_subopt} for the IIR case
requires that
\begin{align*}
  \varepsilon(i) \leq \const \frac{1 - \rho_\star}{C_x + C_u} \:,
\end{align*}
and Theorem~\ref{thm:lqr_feasibility_and_subopt_fir} for the FIR case requires that
\begin{align*}
  \varepsilon(i) \leq \const \frac{(1 - \rho_\star)^2}{C_x + C_u} \:,
\end{align*}
where $\varepsilon(i)$ is our statistical upper bound on the errors
$\max\left\{ \substack{\norm{\Ahat^{(i)} - A},\\ \norm{\Bhat^{(i)} - B}}\right\}$.
This condition is ensured by our assumption \eqref{eq:fourth_ineq}
on $C_T$.
\end{proof}

We now remark on the satisfiability of the constraints on $C_T$
given by \eqref{eq:second_ineq}, \eqref{eq:third_ineq}, and \eqref{eq:fourth_ineq}.
For \eqref{eq:second_ineq} and \eqref{eq:third_ineq} (resp. \eqref{eq:fourth_ineq}), the RHS grows like
$\mathrm{poly}(i)/2^i$ (resp. $\mathrm{poly}(i)/2^{2i/3}$) and hence the supremum of the RHS
(as a function of $i$) is achieved for some finite $i$.
Therefore, we have that $C_T$ satisfies in the IIR case
\begin{align}
  C_T &= \Otilde\left( \max\left\{ \frac{1}{1-\rho_\star}, n, (n+p)\frac{C_\star^4 (1 + \norm{\trueK})^4}{(1-\rho_\star)^8} \right\} \right) \nonumber \\
  &= \Otilde\left( (n+p)\frac{C_\star^4 (1 + \norm{\trueK})^4}{(1-\rho_\star)^8}  \right) \label{eq:C_T_bound} \:,
\end{align}
and that $C_T$ satisfies in the FIR case
\begin{align}
    C_T &= \Otilde\left( \max\left\{ \frac{1}{1-\rho_\star}, n, (n+p)\frac{C_\star^4 (1 + \norm{\trueK})^4}{(1-\rho_\star)^{10}} \right\} \right) \nonumber \\
    &= \Otilde\left( (n+p)\frac{C_\star^4 (1 + \norm{\trueK})^4}{(1-\rho_\star)^{16}}  \right) \label{eq:C_T_bound_FIR} \:.
\end{align}

\section{Regret Decomposition and Analysis}

We use the following regret decomposition, and for simplicity we assume that $T$ is such that $T_0 + T_1 + ... + T_{E-1} = T$ for some $E$.
Note that $E = O(\log_2{T})$.
\begin{align}
    \mathsf{Regret}(T) = \sum_{k=1}^{T} (x_k^\T Q x_k + u_k^\T R u_k - J_\star) = \sum_{i=0}^{E-1} \sum_{j=1}^{T_i} (x_{i, j}^\T Q x_{i, j} + u_{i, j}^\T R u_{i, j} - J_\star) \:.
  \label{eq:regret}
\end{align}
Here, we let $x_{i,j}$ denote the $j$-th state at the $i$-th epoch (and similarly for $u_{i, j}$).
Our definition of regret is defined for a given realization, as opposed to
in expectation.
However, in our analysis so far we have considered sub-optimality guarantees in expectation.
Hence, our first concern is going from a realization to expectation.

Denote by $J_{T}(A,B,\tf K; \Sigma)$ the expected cost incurred by a (stabilizing) feedback policy $\tf K$ over a finite horizon $T$ on system $(A,B)$ being driven by process noise $w \iid \calN(0, \Sigma)$ and starting from an initial condition of $x_0 = 0$, i.e.,
\begin{equation}
J_{T}(A,B,\tf K; \Sigma) := \sum_{k=1}^T \E\left[ x_k^\T Q x_k + u_k^\T R u_k\right].
\label{eq:epochexpcost}
\end{equation}
Recall also that $J(A, B, \tf K; \Sigma)$ is the infinite-horizon LQR cost of $\tf K$ in feedback with $(A, B)$.
We now state some basic properties of $J_T$ and $J$. We omit the proofs of these properties as they
are standard.
\begin{lemma}
The following are true
\begin{enumerate}[(i)]
    \item $J_{T}(A, B, \tf K; \Sigma) \leq T J(A, B, \tf K; \Sigma)$,
    \item $J(A, B, \tf K; \Sigma_1 + \Sigma_2) = J(A, B, \tf K; \Sigma_1) + J(A, B, \tf K; \Sigma_2)$,
    \item $J(A, B, \tf K; \alpha \Sigma) = \alpha J(A, B, \tf K; \Sigma)$ for $\alpha > 0$,
    \item $J(A, B, \tf K; \Sigma_1) \leq J(A, B, \tf K; \Sigma_2)$ if $\Sigma_1 \preceq \Sigma_2$.
\end{enumerate}
\end{lemma}
From these properties, we immediately conclude that
\begin{align}
    J_{T}(A, B, \tf K; \sigma_w^2 I + \sigma_\eta^2 BB^\T) \leq T\left( 1 + \frac{\sigma_\eta^2 \norm{B}^2}{\sigma_w^2}\right) J(A, B, \tf K; \sigma_w^2 I) \:, \label{eq:finite_upper_bound_ineq}
\end{align}
a fact we will make use of later on.

The following lemma relates the finite horizon cost to its expectation.
\begin{lemma}
Let $\tf K$ be a feedback policy that stabilizes $(A,B)$ and that induces system responses $\tf \Phi_x \in \RHinf(C_x,\rho)$ and $\tf \Phi_u \in \RHinf(C_u,\rho)$.  Suppose that the system $(A,B)$ is started at $x_0 = x$ and is  driven by process noise $w \iid \calN(0,\Sigma)$ with $\Sigma \succ 0$ and $\norm{\Sigma} \leq \sigma^2$.  Then with probability at least $1 - \tfrac{1}{\delta}$ over the randomness of the process noise,
\begin{align}
  \sum_{k=1}^T x_k^\T Q x_k + u_k^\T R u_k & \leq J_{T}(A,B,\tf K; \Sigma) + C_c \cdot \calO\left(\ltwonorm{x}^2 + \sigma^2(\sqrt{nT\logg{\tfrac{2}{\delta}}}+\logg{\tfrac{2}{\delta}})\right),
\label{eq:epochbnd}
\end{align}
for $C_c := (1-\rho)^{-2}(\|Q\|C_x^2 + \|R\|C_u^2).$
\label{prop:epochbnd}
\end{lemma}
\begin{proof}
Writing $\tf \Phi_x$ as $\tf\Phi_x = \sum_{k=1}^{\infty} \Phi_x(k) z^{-k}$, we define the following finite-horizon truncations of its block-Toeplitz representation:
\[
\tf \Phi_{x,T} := \begin{bmatrix} \Phi_x(1) & & \\ \vdots &\ddots &  \\ \Phi_x(T) & \dots & \Phi_x(1) \end{bmatrix} \quad \tf \Phi_{x,+} := \begin{bmatrix} \Phi_x(2) \\ \Phi_x(3) \\ \vdots \\ \Phi_x(T+1) \end{bmatrix}.
\]
We let $\tf \Phi_{u, T}$ and $\tf \Phi_{u,T,+}$ define similar matrices for $\tf \Phi_u$.
Using these definitions, we can write
\[
\sum_{k=1}^T x_k^\T Q x_k + u_k^\T R u_k  = \begin{bmatrix} x \\ \omega \end{bmatrix}^\T \begin{bmatrix} M_{11} & M_{12} \\ M_{12}^\T & M_{22} \end{bmatrix} \begin{bmatrix} x \\ \omega \end{bmatrix},
\]
for
\begin{align*}
\omega^\T & = \begin{bmatrix} w_0^\T & w_1^\T & \dots & w_{T-1}^\T\end{bmatrix} \\
M_{11} &= \begin{bmatrix} \tf \Phi_{x,+} \\ \tf \Phi_{u,+}\end{bmatrix}^\T\begin{bmatrix} \calQ & \\ & \calR \end{bmatrix}\begin{bmatrix} \tf \Phi_{x,+} \\ \tf \Phi_{u,+}\end{bmatrix} \\
M_{12} &= \begin{bmatrix} \tf\Phi_{x,+} \\ \tf\Phi_{u,+}\end{bmatrix}^\T\begin{bmatrix} \calQ & \\ & \calR \end{bmatrix}\begin{bmatrix} \tf\Phi_{x,T} \\ \tf\Phi_{u,T}\end{bmatrix} \\
M_{22} &= \begin{bmatrix} \tf\Phi_{x,T} \\ \tf\Phi_{u,T}\end{bmatrix}^\T\begin{bmatrix} \calQ & \\ & \calR \end{bmatrix}\begin{bmatrix} \tf\Phi_{x,T} \\ \tf\Phi_{u,T}\end{bmatrix},
\end{align*}
where $\calQ := \mathrm{blkdiag}(Q)$ and $\calR := \mathrm{blkdiag}(R)$ are block-diagonal matrices of compatible dimension.
With these definitions, one can then check that $\Tr M_{22}\mathrm{blkdiag}(\Sigma) = J_{T}(A,B,\tf K; \Sigma)$.

Finally, given that $\tf \Phi_{x,+}, \tf \Phi_{x,T}$ are sub-matrices of the block-Toeplitz representation of $\tf \Phi_x$, it follows that $\max\{\norm{\tf \Phi_{x,+}}, \norm{\tf \Phi_{x,T}}\} \leq \hinfnorm{\tf \Phi_x} \leq \frac{C_x}{1-\rho},$ where the last inequality follows from Lemma \ref{lem:hinfbnd}.
Similarly, we have that $\max\{\norm{\tf \Phi_{u,+}},\norm{\tf \Phi_{u,T}}\} \leq \hinfnorm{\tf \Phi_u} \leq \frac{C_u}{1-\rho}$ .
The result then follows by using these bounds, noting that $\omega \sim \calN(0,\mathrm{blkdiag}(\Sigma))$, and applying Lemma \ref{lem:quadconc}
with the inequality $\norm{M}_F \leq \sqrt{\mathrm{rank}(M)} \norm{M} \leq \sqrt{\max(n_1, n_2)} \norm{M}$ for an $n_1 \times n_2$ matrix $M$.
\end{proof}

We now proceed to prove our main regret upper bounds, for both
the IIR and FIR case.

Let $\calE_{\mathrm{est}, i}$ denote the event that the conclusions
of Theorem~\ref{thm:estimation} hold up to and including epoch $i$.
Let $\{ \tf{\widehat{\Phi}}_{i,x} \}_{i \geq 0}$ and $\{ \tf{\widehat{\Phi}}_{i, u} \}_{i \geq 0}$ denote the
closed loop SLS responses on the true system $(\trueA, \trueB)$.
When $\calE_{\mathrm{est}, i}$ holds, Theorem~\ref{thm:lqr_feasibility_and_subopt}
in the IIR case and Theorem~\ref{thm:lqr_feasibility_and_subopt_fir} in the FIR case
state that uniformly for all epochs $i$ we have
\begin{align*}
    \tf{\widehat{\Phi}}_{i, x} \in \RHinf(\widehat{C}, \widehat{\rho}) \:, \:\:
    \tf{\widehat{\Phi}}_{i, u} \in \RHinf(\norm{\trueK}\widehat{C}, \widehat{\rho}) \:,
\end{align*}
for
\begin{align*}
    \widehat{C} &= \frac{\const C_\star}{(1-\rho_\star)^2} \:, \\
    \widehat{\rho} &= 7/8 + (1/8) \rho_\star \:,
\end{align*}
in the IIR case and
\begin{align*}
    \widehat{C} &= \frac{\const C_\star}{(1-\rho_\star)^3} \:, \\
    \widehat{\rho} &= 0.999 + 0.001 \rho_\star \:,
\end{align*}
in the FIR case.
For ease of notation, define $\widehat{C}_c^2 := \frac{(\norm{Q} + \norm{R}\norm{\trueK})\widehat{C}^2}{(1-\widehat{\rho})^2}$.

Now fix an epoch $i \geq 1$ (the epoch $i = 0$ will be dealt with separately)
and let $\tf K_i$ denote the controller that is active during epoch $i$.
We invoke Lemma~\ref{prop:epochbnd} conditioned on $\calE_{\mathrm{est},i}$ and $x_{i, 0}$ with $\delta \gets \const \delta / (i+1)^2$,
$\Sigma \gets \sigma_w^2 I + \sigma_{\eta,i}^2 \trueB \trueB^\T$, $C_x \gets \widehat{C}$, $C_u \gets \norm{\trueK} \widehat{C}$, and
$\rho \gets \widehat{\rho}$.
The conclusion is that with (conditional) probability at least $1 - \const \delta / (i+1)^2$,
\begin{align*}
    &\sum_{k=1}^{T_i} x_{i,k}^\T Q x_{i,k} + u_{i,k}^\T R u_{i,k} \\
    &\qquad\leq J_T(\trueA, \trueB, \tf K_i ;\sigma_w^2 I + \sigma_{\eta, i}^2 \trueB \trueB^\T) \\
    &\qquad\qquad+ \widehat{C}^2_c\calO\left( \ltwonorm{x_{i,0}}^2 + (\sigma_w^2 + \sigma_{\eta,i}^2 \norm{\trueB}^2)( \sqrt{n T_i \log((i+1)/\delta)} + \log((i+1)/\delta)  ) \right) \\
    &\qquad\leq T_i \left( 1 + \frac{\sigma_{\eta, i}^2 \norm{\trueB}^2}{\sigma_w^2}\right)J(\trueA, \trueB, \tf K_i ;\sigma_w^2 I) \\
    &\qquad\qquad+ \widehat{C}^2_c \calO\bigg( \sigma_w^2 (n + \log((i+1)/\delta)) \frac{\widehat{C}^2 \widehat{\rho}^2 (1 + \norm{\trueB})^2}{(1-\widehat{\rho})^2} \\
    &\qquad\qquad\qquad+ (\sigma_w^2 + \sigma_{\eta,i}^2 \norm{\trueB}^2)( \sqrt{n T_i \log((i+1)/\delta)} + \log((i+1)/\delta)  ) \bigg) \:.
\end{align*}
For the second inequality, we used the
bound \eqref{eq:finite_upper_bound_ineq} and the bound on
$\ltwonorm{x_{i,0}}$ from \eqref{eq:est_x_begin_bound}.

Furthermore, \eqref{eq:est_bound}
and
Theorem~\ref{thm:lqr_feasibility_and_subopt} in the IIR case
(Theorem~\ref{thm:lqr_feasibility_and_subopt_fir} in the FIR case)
tell us that
on $\calE_{\mathrm{est},i}$, we have the sub-optimality bound
\begin{align*}
  J(\trueA, \trueB, \tf K_i ;\sigma_w^2 I) &\leq (1 + C_{J_{i-1}})^2 (1 + \const \varepsilon_{i-1} (1 + \norm{\trueK}) \hinfnorm{\res{\trueA + \trueB \trueK}})^2 J_\star \;, \\
    \varepsilon_i &= \Otilde\left( \frac{\sigma_w \norm{\trueK} \widehat{C}}{\sigma_{\eta, i}} \sqrt{ \frac{n+p}{T_i} } \right) \:.
\end{align*}
Above, in the IIR case, we set $C_{J_i} = 0$ for all $i$, and in the FIR
case we choose $C_{J_i} = 1/2^{i+1}$. Since $C_{J_i} \leq 1$, we have that $(1 + C_{J_i})^2 \leq 1 + 3C_{j_i}$.
Recalling that $\sigma_{\eta, i} / \sigma_w = 2^{-i/6}$ and that $T_i = C_T 2^i$,
we simplify $\varepsilon_i = \Otilde\left( \norm{\trueK} \widehat{C} \sqrt{\frac{n+p}{C_T}} 2^{-i/3}\right) := \Otilde( \frac{D_1}{\sqrt{C_T}} 2^{-i/3} )$
which gives us
\begin{align*}
  &(1 + \const \varepsilon_{i-1} (1 + \norm{\trueK}) \hinfnorm{\res{\trueA + \trueB \trueK}})^2 \\
  &\qquad= 1 + \Otilde\left( \frac{D_1}{\sqrt{C_T}} (1 + \norm{\trueK})\hinfnorm{\res{\trueA + \trueB \trueK}} 2^{-i/3} + \frac{D_1^2}{C_T} (1 + \norm{\trueK})^2\hinfnorm{\res{\trueA + \trueB \trueK}}^2 2^{-2i/3} \right) \\
  &\qquad:= 1 + \Otilde\left( \frac{D_2}{\sqrt{C_T}} 2^{-i/3} + \frac{D_2^2}{C_T} 2^{-2i/3} \right) \:.
\end{align*}
This means that
\begin{align*}
  &T_i \left( 1 + \frac{\sigma_{\eta, i}^2 \norm{\trueB}^2}{\sigma_w^2}\right)J(\trueA, \trueB, \tf K_i ;\sigma_w^2 I) \\
  &\qquad\leq T_i \left( 1 + 2^{-i/3}\norm{\trueB}^2 \right) (1 + 3C_{J_{i-1}}) \left(1 + \Otilde\left( \frac{D_2}{\sqrt{C_T}} 2^{-i/3} + \frac{D_2^2}{C_T} 2^{-2i/3} \right)\right) J_\star \\
  &\qquad\leq T_i \left( 1 + \Otilde\left( \left(\frac{D_2}{\sqrt{C_T}} + \norm{\trueB}^2 \right) 2^{-i/3} + \left(\frac{D_2^2}{C_T} + \frac{D_2 \norm{\trueB}^2}{\sqrt{C_T}} \right) 2^{-2i/3} + \frac{D_2^2 \norm{\trueB}^2}{C_T} 2^{-i} \right) \right) J_\star \\
  &\qquad\qquad + \Otilde((1+\norm{\trueB}^2)( C_T + D_2\sqrt{C_T} + D_2^2) J_\star ) \\
  &\qquad= T_i J_\star + \Otilde( \sqrt{C_T} D_2 + C_T \norm{\trueB}^2 ) J_\star 2^{2i/3} + \Otilde( D_2^2 + \sqrt{C_T} D_2 \norm{\trueB}^2 ) J_\star 2^{i/3} \\
  &\qquad\qquad + \Otilde((1+\norm{\trueB}^2)( C_T + D_2\sqrt{C_T} + D_2^2) J_\star ) \:.
\end{align*}
Hence,
\begin{align*}
    &\sum_{k=1}^{T_i} (x_{i,k}^\T Q x_{i,k} + u_{i,k}^\T R u_{i,k} - J_\star) \\
    &\qquad\leq \Otilde( \sqrt{C_T} D_2 + C_T \norm{\trueB}^2 ) J_\star 2^{2i/3} + \Otilde( D_2^2 + \sqrt{C_T} D_2 \norm{\trueB}^2 ) J_\star 2^{i/3} \\
    &\qquad\qquad + \Otilde\left( \frac{\widehat{C}_c^2 \sigma_w^2 n \widehat{C}^2 (1 + \norm{\trueB})^2}{(1 - \widehat{\rho})^2} \right) + \Otilde( \widehat{C}_c^2 \sigma_w^2 \sqrt{n C_T} 2^{i/2} ) + \Otilde( \widehat{C}_c^2 \sigma_w^2 \norm{\trueB}^2 \sqrt{n C_T} 2^{i/6} ) \\
    &\qquad\qquad + \calO(C_T 2^{i/2} (1 + \norm{\trueB}^2)) + \Otilde((1+\norm{\trueB}^2)( C_T + D_2\sqrt{C_T} + D_2^2) J_\star ) \:.
\end{align*}

On the other hand, when epoch $i=0$, we have that
\begin{align*}
  \sum_{k=1}^{T} x_{0,k}^\T Q x_{0,k} + u_{0,k}^\T R u_{0,k} &\leq
    J_T(\trueA, \trueB, \tf K_0, \sigma_w^2 I + \sigma_{\eta, 0}^2 \trueB \trueB^\T ) + \Otilde( \widehat{C}_c^2 \sigma_w^2 (1 + \norm{\trueB}^2) \sqrt{n C_T} ) \\
    &\leq C_T (1 + \norm{\trueB}^2) J(\trueA, \trueB, \tf K_0, \sigma_w^2 I) + \Otilde( \widehat{C}_c^2 \sigma_w^2 (1 + \norm{\trueB}^2) \sqrt{n C_T} ) \:.
\end{align*}

Summing over all the epochs,
\begin{align*}
    \mathsf{Regret}(T) &= \sum_{i=0}^{O(\log_2{T})} \sum_{k=1}^{T_i} (x_{i,k}^\T Q x_{i,k} + u_{i,k}^\T R u_{i,k} - J_\star) \\
    &\leq \Otilde( (\sqrt{C_T} D_2 + C_T \norm{\trueB}^2) J_\star T^{2/3} ) + \Otilde( \widehat{C}_c^2 \sigma_w^2 \sqrt{n C_T} T^{1/2} ) \\
    &\qquad + \Otilde( D_2^2 + \sqrt{C_T} D_2 \norm{\trueB}^2 J_\star T^{1/3} ) + \Otilde( \widehat{C}_c^2 \sigma_w^2 \norm{\trueB}^2 \sqrt{n C_T} T^{1/6} ) \\
    &\qquad + \Otilde\left( \frac{\widehat{C}_c^2 \sigma_w^2 n \widehat{C}^2 (1 + \norm{\trueB})^2}{(1 - \widehat{\rho})^2}
     + C_T (1 + \norm{\trueB}^2) J(\trueA, \trueB, \tf K_0, \sigma_w^2 I)
    \right) \\
    &\qquad + \Otilde((1+\norm{\trueB}^2)( C_T + D_2\sqrt{C_T} + D_2^2) J_\star ) \\
    &\qquad + \Otilde( \widehat{C}_c^2 \sigma_w^2 (1 + \norm{\trueB}^2) \sqrt{n C_T} )
    + \calO(C_T  (1 + \norm{\trueB}^2) \sqrt{T} )
    \:.
\end{align*}
Using the bound on $C_T$ from \eqref{eq:C_T_bound},
recalling that
\begin{align*}
  D_2 = \sqrt{n+p} \norm{\trueK} \widehat{C} (1 + \norm{\trueK}) \hinfnorm{\res{\trueA + \trueB\trueK}} \:,
\end{align*}
and ignoring the $o(T^{2/3})$ terms in the regret bound,
we have that the regret is bounded by in the IIR case
\begin{align*}
  \Otilde\left( (n+p) \hinfnorm{\res{\trueA + \trueB \trueK}} \frac{C_\star^3 (1 + \norm{\trueK})^4}{(1-\rho_\star)^6}  J_\star T^{2/3} + (n+p) \frac{C_\star^4 (1 + \norm{\trueK})^4 \norm{\trueB}^2}{(1-\rho_\star)^8} J_\star T^{2/3} \right) \:.
\end{align*}
By using Lemma~\ref{lem:hinfbnd}, we have that $\hinfnorm{\res{\trueA + \trueB \trueK}} \leq \frac{C_\star}{1-\rho_\star}$, and hence the bound in the IIR case
simplifies to
\begin{align*}
  \Otilde\left( (n+p) \frac{C_\star^4 (1 + \norm{\trueK})^4 (1 + \norm{\trueB})^2 J_\star}{(1-\rho_\star)^8} T^{2/3} \right) \:.
\end{align*}

Now for the FIR case, we use the bound \eqref{eq:C_T_bound_FIR} and ignoring the $o(T^{2/3})$ terms,
the regret is bounded by
\begin{align*}
  \Otilde\left((n+p) \hinfnorm{\res{\trueA + \trueB \trueK}} \frac{C_\star^3 (1 + \norm{\trueK})^4}{(1-\rho_\star)^{11}}  J_\star T^{2/3}  +  (n+p) \frac{C_\star^4 (1 + \norm{\trueK})^4 \norm{\trueB}^2}{(1-\rho_\star)^{16}} J_\star T^{2/3}  \right) \:.
\end{align*}
Using the same bound on $\hinfnorm{\res{\trueA + \trueB \trueK}}$ as before, we obtain the FIR regret bound
\begin{align*}
  \Otilde\left( (n+p) \frac{C_\star^4 (1 + \norm{\trueK})^4 (1 + \norm{\trueB})^2}{(1-\rho_\star)^{16}} J_\star T^{2/3} \right) \:.
\end{align*}

\section{Lower bound}
\label{app:lower_bound}

This section is dedicated to proving Theorem~\ref{thm:lower_bound}. Throughout this section we assume the following setup and notation.
We consider the LQR problem defined by
\begin{align*}
\min_{u_0, u_1, \ldots, u_{T - 1}} &\E \left[ x_{T}^\top P x_T + \sum_{t = 0}^{T - 1} u_t^\top R u_t + x_{t}^\top Q x_{t} \right]\:,\\
& \text{s.t. } x_{t + 1} = \trueA x_t + \trueB u_t + w_t.
\end{align*}
where $u_t$ is allowed to be any random variable taking values in $\R^{\inputdim}$ that is independent of the sigma algebra $\sigma(w_t, w_{t + 1}, \ldots)$. In particular, $u_t$ can be a measurable function of $x_0$, $w_0$, $w_1$, \ldots, $w_{t - 1}$, and possibly other exogenous randomness.

We assume that $Q$ and $R$ are both positive definite matrices. Throughout this section we denote by $P$ the solution to the discrete algebraic Riccati equation:
\begin{align*}
P_\star = A^\T P_\star A - A^\T P_\star B(R + B^\T P_\star B)^{-1}B^\T P_\star A + Q.
\end{align*}
Moreover, we denote by $K_\star$ the optimal controller for the infinite horizon LQR problem, namely  $\trueK = - (R + B^\T P_\star B)^{-1} B^\T P_\star A$. Hence, the optimal closed loop matrix is given by $M = \trueA + \trueB \trueK$. Throughout this section we assume that the system $(A,B)$ is controllable and hence $\rho(M) < 1$. Therefore, there exist $C > 0$ and $\rho \in (0,1)$ such that $\norm{M^k} \leq C \rho^k$ for all $k \geq 1$.

The initial state $x_0$ for the LQR problem defined above is assumed to have  distribution  $\calN(0, P_\infty)$, where $P_\infty$ is the unique solution to the Lyapunov equation
\begin{align*}
P_\infty = (\trueA + \trueB \trueK) P_\infty (\trueA + \trueB \trueK)^\T + \sigma_w^2 I_\statedim.
\end{align*}
The distribution $\calN(0, P_\infty)$ corresponds to the stationary distribution of the optimal closed loop system $x_{t + 1} = (\trueA + \trueB \trueK) x_t + w_t$. In particular, if $x_t \sim \calN(0, P_\infty)$, then $x_{t + 1} \sim \calN(0, P_\infty)$.

We consider the objective
\begin{align}
\label{eq:finite_horizon_lqr}
J_T(\nu_0, \nu_1, \ldots, \nu_{T - 1}) = \E \left[x_{T}^\top P_\star x_T + \sum_{t = 0}^{T - 1} u_t^\top R u_t + x_{t}^\top Q x_{t}\right],
\end{align}
where $u_t = \trueK x_t + \nu_t$ for the optimal controller $\trueK$. Then, since the terminal cost is given by $P_\star$,  we know that the minimum of objective \eqref{eq:finite_horizon_lqr} over $\nu_0$, $\nu_1$, \ldots, $\nu_{T - 1}$ such that $\nu_t$ is independent of  $\sigma(w_t, w_{t + 1}, \ldots)$ is achieved when all $\nu_t$ are identically zero. The random variables $\nu_t$ should be thought of as deviations from the optimal inputs $\trueK x_t$ for the infinite horizon LQR.
Finally, since $x_0 \sim \calN(0, P_\infty)$ we have that the optimal objective value is  $J_T^\star = J_T(0) = T J_\star + \Tr (P_\star P_\infty)$, where $J_\star = \sigma_w^2 \Tr(P_\star)$ is the optimal objective value of the infinite horizon LQR.

The proof of Theorem~\ref{thm:lower_bound} follows an argument inspired from the field of strongly convex optimization. We show that under the minimum eigenvalue condition of the process $z_t = [x_t^\T, u_t^\T]^\T$, the process $\{\nu_t\}_{t \geq 0}$ is bounded away from zero. Moreover, we show that the expected regret at time $T$ is a strongly convex function of $\nu_0$, $\nu_1$, \ldots, $\nu_{T - 1}$, leading us to the desired conclusion. We proceed by proving a sequence of technical result, followed by the proof of Theorem~\ref{thm:lower_bound}.

\begin{lemma}
\label{lem:perturbation}.
Suppose that
\begin{align}
\lambda_{\min}\left(\sum_{t = 0}^{T - 1} \begin{bmatrix}
x_t \\
u_t
\end{bmatrix} \begin{bmatrix}
x_t^\top &
u_t^\top
\end{bmatrix}
\right) \geq \tau,
\end{align}
with $u_t = \trueK x_t + \nu_t$. Then
\begin{align}
\sum_{t = 0}^{T - 1} \|\nu_t\|_{2}^2 \geq \left(1 + \sigma_{\min}(\trueK)^2\right)\tau
\end{align}
\end{lemma}
\begin{proof}
Consider $v = [v_1^\top, v_2^\top]^\top \in \R^{n + p}$ such that $\|v\|_2 = 1$ and $v_1 + K^\top v_2 = 0$
(such $v$ exists because $[I, K^\top]$ is an $n \times (n + p)$ matrix and hence has a non-trivial null space).
Moreover, $\|v_2\|_2^2 \leq (1 + \sigma_{\min}(K_\star)^2)^{-1}$. Then, by assumption we have
\begin{align*}
\tau &\leq \sum_{t = 0}^{T - 1} (\langle x_t, v_1\rangle + \langle u_t, v_2\rangle )^2 = \sum_{t = 0}^{T - 1} (\langle x_t, v_1\rangle + \langle K x_t + \nu_t, v_2\rangle)^2 = \sum_{t = 0}^{T - 1} \langle \nu_t, v_2\rangle^2\\
&\leq  \|v_2\|^2_2 \sum_{t = 0}^{T - 1} \| \nu_t \|_2^2 \leq  \frac{1}{1 + \sigma_{\min}(K)^2} \sum_{t = 0}^{T - 1} \| \nu_t \|_2^2.
\end{align*}
\end{proof}

\begin{lemma}
\label{lem:cancel}
Denote by $M$ the optimal closed loop matrix $\trueA + \trueB \trueK$. Then
\begin{align*}
J_T(\nu_0, \nu_1, \ldots, \nu_{T - 1}) - J_T^\star &= \E \left[\sum_{j = 0}^{T - 1} \nu_j^\top (\trueB^\top P_\star \trueB + R)\nu_j \right] \\
&\quad + 2\E \left[\sum_{0 \leq i < j \leq T - 1} \nu_i^\top \trueB^\top (M^\top)^{j - i} P_\star \trueB \nu_j \right].
\end{align*}
\end{lemma}
\begin{proof}
We know that
\begin{align*}
J_T^* = \E \left[\sum_{t = 0}^{T - 1} x_{\star,t}^\top (Q + \trueK ^\top R \trueK)x_{\star, t} \right] + \E \left[x_{\star, T}^\top P_\star x_{\star, T}\right],
\end{align*}
where $x_{\star, t} = \sum_{j = -1}^{t - 1} M^{t - 1 - j} w_j$. Here, $w_{-1} = x_0$ for convenience, and $w_t \iid \calN(0, \sigma_w^2 I_\statedim)$ for convenience. Also,
\begin{align*}
J_T(\nu_0, \nu_1, \ldots, \nu_{T - 1}) = \E \left[\sum_{t = 0}^{T - 1} x_{t}^\top Q x_{t} + (\trueK x_t + \nu_t)^\top R (\trueK x_t + \nu_t) \right] + \E \left[x_{T}^\top P_\star x_{T}\right],
\end{align*}
where $x_{t} = \sum_{j = -1}^{t - 1} M^{t - 1 - j} w_j + M^{t - 1 - j} B \nu_j$ and $\nu_{-1} = 0$. Recall that $\nu_t$ is independent of any $w_i$ with $i \geq t$. Hence, for any matrix $N$ we have that $\E \left[ w_i^\top N \nu_t\right] = 0$ if $i \geq t$. Therefore

\begin{align*}
J_T - J_T^\star &= \E \left[ \sum_{t = 0}^{T - 1} \sum_{0 \leq i < j \leq t - 1} 2 w_i^\top (M^\top)^{t - 1 - i}(Q + \trueK^\top R \trueK) M^{t - 1 - j} \trueB \nu_j \right]\\
&\qquad+ \E \left[ \sum_{t = 0}^{T - 1} \sum_{i, j = 0}^{t - 1} \nu_i^\top \trueB^\top (M^\top)^{t - 1 - i} (Q + \trueK^\top R \trueK) M^{t - 1 - j} \trueB \nu_j \right] + \E \left[ \sum_{t = 0}^{T - 1} \nu_t^\top R \nu_t\right]\\
&\qquad+   \E \left[ \sum_{t = 0}^{T - 1} \sum_{i = 0}^{t - 1} 2 w_i^\top (M^\top)^{t - 1 - i} \trueK^\top R \nu_t \right]\\
&\qquad+ \E \left[\sum_{0 \leq i < j \leq T - 1} 2 w_i^\top (M^\top)^{T - 1 - i} P_\star M^{T - 1 - j} \trueB \nu_j \right]\\
&\qquad+ \E \left[\sum_{i,j = 0}^{T - 1} \nu_i^\top \trueB^\top (M^\top)^{T - 1 - i} P_\star M^{T - 1 - j} \trueB \nu_j \right].
\end{align*}

Now, we note that the sum of the terms that depend linearly on $\nu_t$ is equal to zero, otherwise
the optimum of $J_T$ would not be achieved at $\nu_t = 0$ for all $t$. Indeed, this can be checked through direct computation by remarking that
the optimal controller $\trueK$ satisfies $\trueK^\top R = - M^\top P_\star \trueB$, and recalling that $P_\star$ satisfies the Lyapunov equation
\begin{align}
\label{eq:lyp}
P_\star = M^\T P_\star M + Q + \trueK^\T R \trueK.
\end{align}
Hence, we have
\begin{align*}
J_T - J_T^\star &=  \E \left[\sum_{j = 0}^{T - 2} \nu_j^\top \trueB^\top \left(\sum_{t = j + 1}^{T - 1} (M^\top)^{t - 1 - j} (Q + K^\top R K) M^{t - 1 - j}\right) \trueB \nu_j \right] \\
&\qquad+ 2\E \left[\sum_{0 \leq i < j \leq T - 2} \nu_i^\top \trueB^\top (M^\top)^{j - i} \left(\sum_{t = j + 1}^{T - 1} (M^\top)^{t - 1 - j} (Q + \trueK^\top R \trueK) M^{t - 1 - j}\right) \trueB \nu_j \right]\\
&\qquad+ \E \left[ \sum_{t = 0}^{T - 1} \nu_t^\top R \nu_t\right] + \E \left[\sum_{j = 0}^{T - 1} \nu_j^\top \trueB^\top (M^\top)^{T - 1 - j} P_\star M^{T - 1 - j} \trueB \nu_j \right]\\
&\qquad+ 2\E \left[\sum_{0 \leq i < j \leq T - 1} \nu_i^\top \trueB^\top (M^\top)^{T - 1 - i} P_\star M^{T - 1 - j} \trueB \nu_j \right] \:.
\end{align*}
The conclusion follows  by using the Lyapunov equation~\eqref{eq:lyp} and simple algebra.
\end{proof}

\begin{lemma}
\label{lem:psd_helper}
Let $M$ and $N$ be any matrices in $\R^{n \times n}$, with $N$ positive definite, and let $T$ be any positive integer. Also, consider the $(nT) \times (nT)$ block matrix $D(T)$
 with blocks $D(T)_{i,j}$ equal to
\begin{align*}
D_{i,j} = \left\{
\begin{array}{cc}
(M^\top)^{j - i} \left(\sum_{k = 0}^{T - j} (M^\top)^k N M^k \right) & \text{if } i < j,\\
\sum_{k = 0}^{T - j} (M^\top)^k N M^k & \text{if } i = j,\\
\left(\sum_{k = 0}^{T - i} (M^\top)^k N M^k\right) M^{i - j} & \text{if } i < j,
\end{array}
\right.
\end{align*}
where $1 \leq i, j\leq T$. The matrix $D$ is positive definite.
\end{lemma}
\begin{proof}
We proceed by induction. Let $T = 2$. Then the matrix of interest is
\begin{align*}
D(2) = \begin{bmatrix}
N + M^\top N M & M^\top N\\
N M & N
\end{bmatrix}.
\end{align*}
Since $N \succ 0$, we see that $D(T)$ is positive definite because its Schur complement is
\begin{align*}
N + M^\top N M - M^\top N N^{-1} N M = N \succ 0.
\end{align*}
For $T > 2$ we proceed similarly. We consider the matrix $D(T)$ and take its Schur complement
with respect to bottom right corner, i.e.
\begin{align*}
\begin{bmatrix}
D(T)_{1,1} & \ldots & D(T)_{1, T - 1}\\
\vdots & \ddots & \vdots \\
D(T)_{T - 1, 1} & \ldots & D(T)_{T - 1, T - 1}\\
\end{bmatrix}
 - \begin{bmatrix}
D(T)_{1, T}\\
\vdots \\
D(T)_{T - 1, T}
\end{bmatrix}
D(T)_{T, T}^{-1}
\begin{bmatrix}
D_{T, 1} & \ldots & D_{T, T - 1}
\end{bmatrix}
\end{align*}
Let $i \leq j < T$. Then, the $(i,j)$ block of the Schur complement of $D(T)$ is
\begin{align*}
D(T)_{i,j} - D(T)_{i, T} D(T)^{-1} D(T)_{T, j} &= (M^\top)^{j - i} \left(\sum_{k = 0}^{T - j} (M^\top)^k N M^k \right) - (M^\top)^{T - i} N N^{-1} N M^{T - j}\\
&= (M^\top)^{j - i} \left(\sum_{k = 0}^{T - 1 - j} (M^\top)^k N M^k \right) = D(T - 1)_{i, j}.
\end{align*}
Similarly, if $j \leq i < T$ we have that $D(T)_{i,j} - D(T)_{i, T} D(T)^{-1} D(T)_{T, j} = D(T - 1)_{i,j}$. Hence, we have shown that the Schur complement
of $D(T)$ with respect to the entry $D(T)_{T, T}$ is $D(T - 1)$. By induction this matrix is positive definite and the conclusion follows.
\end{proof}

\begin{lemma}
\label{lem:psd}
As before, $P_\star$ is the solution to the algebraic Riccati equation and $M = \trueA + \trueB \trueK$ is the optimal closed loop matrix. For any vectors $v_0$, $v_1$, \ldots, $v_{T - 1}$ in $\R^p$ we have
\begin{align*}
\sum_{j = 0}^{T - 1} v_j^\top (\trueB^\top P_\star \trueB + R) v_j + 2\sum_{0 \leq i < j \leq T - 1} v_i^\top B^\top (M^\top)^{j - i} P_\star \trueB v_j \geq \lambda_{\min}(R)\sum_{j= 0}^{T - 1}\|v_j\|_2^2.
\end{align*}
\end{lemma}
\begin{proof}
It suffices to prove that the following matrix is positive semi-definite:
\begin{align*}
\begin{bmatrix}
P_\star & M^\top P_\star & (M^\top)^2 P_\star & \ldots & (M^\top)^{T - 1} P_\star\\
P_\star M & P_\star & M^\top P_\star &\ldots & (M^\top)^{T - 2} P_\star\\
P_\star M^2 & P_\star M &  P_\star &\ldots & (M^\top)^{T - 2} P_\star \\
\vdots & \vdots & \vdots & \ddots & \vdots \\
P_\star M^{T - 1} & P_\star M^{T - 2} & P_\star M^{T - 3} & \ldots & P_\star
\end{bmatrix}.
\end{align*}
The Schur complement of this matrix around the bottom right corner $P_\star$ has the form $D(T - 1)$ with $N = Q + \trueK^\T R \trueK$, where $D(T - 1)$ is defined as in Lemma~\ref{lem:psd_helper}. To see this recall that $P_\star$ satisfies the Lyapunov equation~\eqref{eq:lyp}. The conclusion follows.
\end{proof}

\begin{lemma}
\label{lem:state_bound_2}
Fix a horizon $T_0 > 0$, and suppose the inputs are of the form $u_t = \trueK x_t + \nu_t$. Recall that there exists constants $C > 0$ and $\rho \in (0,1)$ such that $\norm{M^k} \leq C \rho^k$ for all $k \geq 1$. Then
\begin{align*}
\E \ltwonorm{x_{T_0}}^2 \leq 3 C^2 \rho^{2T_0} \E \ltwonorm{x_0}^2 + 3\frac{ \statedim \sigma_w^2 C^2}{1 - \rho^2} + 3 \frac{C^2}{1 - \rho^2} \E \left[\sum_{t = 0}^{T_0} \ltwonorm{\nu_t}^2\right].
\end{align*}
\end{lemma}
\begin{proof}
Recall that we denote by $M$ the closed loop matrix $\trueA + \trueB \trueK$. We have that
\begin{align*}
x_{T_0} = M^{T_0} x_0 + \sum_{t = 0}^{T_0 - 1} M^{T_0 - 1 - t} (\trueB \nu_t + w_t).
\end{align*}
Then
\begin{align*}
\ltwonorm{x_{T_0}}^2 \leq 3 \ltwonorm{M^{T_0} x_0}^2 + 3 \ltwonorm{\sum_{t = 0}^{T_0 - 1} M^{T_0 - 1 - t} w_t}^2 + 3 \ltwonorm{\sum_{t = 0}^{T_0 - 1} M^{T_0 - 1 - t} B\nu_t}^2.
\end{align*}

Recall that $\norm{M^{t}} \leq C \rho^t$. Then
\begin{align*}
\E \ltwonorm{x_{T_0}}^2 \leq 3 C^2 \rho^{2T_0} \E \ltwonorm{x_0}^2 + 3\frac{ \statedim \sigma_w^2 C^2}{1 - \rho^2} + 3 \frac{C^2}{1 - \rho^2} \E \left[\sum_{t = 0}^{T_0} \ltwonorm{\nu_t}^2\right].
\end{align*}

\end{proof}

\begin{lemma}
\label{lem:conv_riccati}
Let $Q$ and $R$ be positive definite matrices, and $P_0 = 0$. Consider the Riccati recursion
\begin{align*}
P_{t + 1} = A^\T P_t A - A^\T P_t B (R + B^\T P_t B)^{-1} B^\T P_t A + Q.
\end{align*}
Then, if $P_\star$ is the unique solution of the Riccati equation, we have
\begin{align*}
\norm{P_t - P_\star} \leq \left(1 + \frac{1}{\nu} \right)^{- t}, \text{ where }\: \nu = 2\norm{P_\star}\max\left\{\frac{\norm{\trueA}^2}{\lambda_{\min}(Q)}, \frac{\norm{\trueB}^2}{\lambda_{\min}(R)} \right\}.
\end{align*}
Moreover, we have that
\begin{align*}
\sum_{t = 0}^\infty \Tr(P_t) - \Tr(P_\star) \geq  - \statedim \left(1 + \nu \right).
\end{align*}
\end{lemma}
\begin{proof}
The first part follows from Proposition 1 of \citet{lincoln2006relaxing} on value iteration. The second part follows
by bounding
\begin{align*}
\Tr(P_t) - \Tr(P_\star) \geq - \statedim \norm{P_t - P_\star} \geq - \statedim \left(1 + \frac{1}{\nu} \right)^{- t},
\end{align*}
and summing up these inequalities.
\end{proof}

\begin{lemma}
\label{lem:bound_by_future}
Fix a horizon $T_0 > 0$ and denote $\hat{x}_t = x_{t + T_0}$ and $\hat{u}_t = \hat{u}_{t + T_0}$. Then
\begin{align*}
\E \left[\hat{x}_0^\T P \hat{x}_0 \right] \leq \min_{\hat{u}_0, \hat{u}_1, \ldots} \E \left[\sum_{t = 0}^{T - 1} \hat{x}_t^\T Q \hat{x}_t + \hat{u}_t^\T R \hat{u}_t\right] - T J_\star + n \sigma_w^2 (1 + \nu) + \left(1 + \frac{1}{\nu} \right)^{- T}  \E \ltwonorm{\hat{x}_0}^2\:,
\end{align*}
where
\begin{align*}
\nu = 2\norm{P_\star}\max\left\{\frac{\norm{\trueA}^2}{\lambda_{\min}(Q)}, \frac{\norm{\trueB}^2}{\lambda_{\min}(R)} \right\}.
\end{align*}
\end{lemma}
\begin{proof}
Let us consider the Ricatti recursion
\begin{align*}
P_{t + 1} = A^\T P_{t} A - A^\T P_t B (R + B^\T P_t B)^{-1} B^\T P_t A + Q,
\end{align*}
where $P_{0} = 0$. Then
\begin{align*}
\min_{\hat{u}_0, \hat{u}_1, \ldots} \E \left[\sum_{t = 0}^T \hat{x}_t^\T Q \hat{x}_t + \hat{u}_t^\T R \hat{u}_t\right] = \E \hat{x}_0^\T P_{T} \hat{x}_0 + \sigma_w^2 \sum_{t = 0}^{T - 1} \Tr\left( P_t \right).
\end{align*}

From the first part of Lemma~\ref{lem:conv_riccati} we know that
\begin{align*}
\norm{P_T - P_\star} \leq \left(1 + \frac{1}{\nu} \right)^{- T},
\end{align*}
while from the second part of that Lemma we know that
\begin{align*}
\sum_{t = 0}^{T - 1}\left[\Tr\left( P_t \right) - \Tr\left( P_\star \right) \right] \geq - n (1 + \nu).
\end{align*}
The conclusion follows once we recall that $J_\star = \sigma_w^2 \Tr(P_\star)$.
\end{proof}

\begin{proof}[Proof of Theorem~\ref{thm:lower_bound}]
Let $T_0 > 0$ to be chosen later and let $T \geq T_0$. We decompose the regret as the sum of the regret from $0$ to $T - T_0 - 1$ and the regret from $T - T_0$ to $T - 1$, and  we write the first component in terms terms of the expected cost $J_{T - T_0}$ defined in Eq.~\eqref{eq:finite_horizon_lqr}. We have
\begin{align*}
\sum_{t = 0}^{T - 1} \E \left[x_t^\T Q x_t + u_t^\T R u_t - J_\star\right] &= \E \left[x_{T - T_0}^\T P_\star x_{T - T_0} + \sum_{t = 0}^{T - T_0 - 1} x_t^\T Q x_t + u_t^\T R u_t \right] - J_{T - T_0}^\star \\
&\quad +\sum_{t = T - T_0}^{T - 1} \E \left[x_t^\T Q x_t + u_t^\T R u_t - J_\star\right] - T_0 J_\star  \\
&\quad + \Tr( P_\star P_\infty) - \E x_{T - T_0}^\T P_\star x_{T - T_0},
\end{align*}
where we used $J_{T - T_0}^\star = (T - T_0)J_\star + \Tr(P_\star P_\infty)$. The term $\Tr(P_\star P_\infty )$ we an simply lower bound by zero since $P_\infty$ and $P_\star$ are positive semi-definite matrices. From Lemmas~\ref{lem:cancel} and \ref{lem:psd} we have
\begin{align*}
\E \left[x_{T - T_0}^\T P_\star x_{T - T_0} + \sum_{t = 0}^{T - T_0 - 1} x_t^\T Q x_t + u_t^\T R u_t \right] - J_{T - T_0}^\star \geq \lambda_{\min}(R) \sum_{t = 0}^{T - T_0 - 1} \ltwonorm{\nu_t}^2.
\end{align*}
By Lemma~\ref{lem:bound_by_future} we have that
\begin{align*}
    &\sum_{t = T - T_0}^{T - 1} \E \left[x_t^\T Q x_t + u_t^\T R u_t - J_\star\right] - T_0 J_\star - \E x_{T - T_0}^\T P_\star x_{T - T_0} \\
    &\qquad\geq -\statedim \sigma_w^2 (1 + \nu) - \left(1 + \frac{1}{\nu}\right)^{- T_0} \E \ltwonorm{x_{T - T_0}}^2 \:.
\end{align*}

Then, from Lemma~\ref{lem:state_bound_2} we get
\begin{align*}
\sum_{t = 0}^{T - 1} \E \left[x_t^\T Q x_t + u_t^\T R u_t - J_\star\right] \geq& \frac{1}{2}\lambda_{\min}(R) \sum_{t = 0}^{T - T_0} \ltwonorm{\nu_t}^2 \\
&- \underbrace{\left(3C \rho^{2T_0}\Tr(P_\infty) + \statedim \sigma_w^2 \frac{\lambda_{\min}(R)}{2}\right)}_{C_0},
\end{align*}
 by choosing
\begin{align*}
T_0 \geq \frac{\log \left( \frac{2 C^2}{(1 - \rho^2)\lambda_{\min}(R)} \right)}{\log(1 + \nu^{-1})}.
\end{align*}
The conclusion follows by Lemma~\ref{lem:perturbation}.
\end{proof}


\section{Miscellaneous Results}

First we state some results for the function class $\mathcal{RH}_\infty(C, \rho)$.
\begin{lemma}
Let $\tf G_i \in \RHinf(C_i,\rho_i)$ for $i=1,2$ and  Then $\tf H = \tf G_1
\tf G_2 \in \RHinf(C,\rho)$ for any $\rho \in (\max(\rho_1, \rho_2),1)$
and $C=\max\left\{1,\frac{1}{e\logg{\tfrac{\rho}{\max(\rho_1, \rho_2)}}}
\frac{\rho}{\max(\rho_1, \rho_2)}\right\}C_1C_2$.
Note for simplicity if we assume $\rho \geq 1/4$
we can take $C = \frac{6C_1C_2}{1-\rho}$ and $\rho = \Avg(\max(\rho_1, \rho_2) , 1)$.
\label{lem:G1G2}
\end{lemma}
\begin{proof}
Assume wlog that $\rho_1 \geq \rho_2$.
Note that $H(k) = \sum_{t=0}^k G_1(t)G_2(k-t)$, and therefore for all $k\geq 0$ we have that
\begin{align*}
\norm{H(k)} &= \bignorm{\sum_{t=0}^k G_1(t)G_2(k-t)} \leq C_1C_2 \sum_{t=0}^k \rho_1^t \rho_2^{k-t}\\
& \leq C_1C_2 \sum_{t=0}^k \rho_1^k = C_1C_2(k+1)\rho_1^k,
\end{align*}

Fix a $\rho \in (\rho_1,1)$.
Define $g(k) = (k+1) (\rho_1/\rho)^k$ and $h(k) = \log{g(k)}$.
We see that $h'(k) = 0$ only for $k = k_* = \frac{1}{\log(\rho/\rho_1)} - 1$.
Furthermore, $h(k_*) = \log(1/\log(\rho/\rho_1)) - 1 + \log(\rho/\rho_1)$.
Hence, $g(k_*) = \frac{1}{e\log(\rho/\rho_1)} (\rho/\rho_1)$.

The claim now follows since for any $k \geq 0$,
\begin{align*}
  (k+1) \rho_1^k = (k+1) (\rho_1/\rho)^k \rho^k \leq \left[ \sup_{k=0, 1, ...} (k+1) (\rho_1/\rho)^k \right] \rho^k \leq \max\{ 1, g(k_*) \} \rho^k \:.
\end{align*}
We also use the inequality $\log(1 + x) \geq x/2$ for $x \in [0, 2.5]$.
\end{proof}

\begin{lemma}
  Let $\tf G_i \in \RHinf(C_i,\rho_i)$ for $i=1,2$. Then $\tf G_1 + \tf G_2 \in \RHinf(C_1 + C_2, \max\{ \rho_1, \rho_2 \})$.
\label{lem:G1+G2}
\end{lemma}
\begin{proof}
Straightforward from triangle inequality and the definitions.
\end{proof}

\begin{lemma}
\label{lem:inverse_rhinf_direct}
Suppose that $\tf \Delta \in \RHinf(C, \rho)$ with $C \leq 2$ and $\rho \geq 1/e$,
and furthermore $\hinfnorm{\tf \Delta} < 1$.
Then we have
\begin{align*}
    (I \pm \tf \Delta)^{-1} \in \RHinf\left( 1 + \frac{\const C}{1-\rho}, \Avg(\rho, 1) \right) \:.
\end{align*}
\end{lemma}
\begin{proof}
The function $f(x) = \frac{x}{e\log(x)}$ is monotonically decreasing on the
interval $(1, 1/\rho)$. Hence for any $x \in (1, 1/\rho)$, we have
$f(x) \geq f(1/\rho) \geq f(e) = 1$.
Applying the composition lemma (Lemma~\ref{lem:G1G2}) to the system $\tf \Delta \circ \tf \Delta$, we have that for $c_1 \in (1, 1/\rho)$,
\begin{align*}
  \tf \Delta^2 \in \RHinf\left( \frac{c_1}{e \log(c_1)} C^2 , c_1 \rho \right) \:.
\end{align*}
Now if we recursively set $c_k \in (c_{k-1}, 1/\rho)$ for $k= 2, 3, ...$,
repeated applications of the composition lemma yield that
\begin{align*}
  \tf \Delta^n \in \RHinf\left( C^n \prod_{i=1}^{n-1} \frac{c_i}{e \log(c_i)} , c_{n-1} \rho \right) \:.
\end{align*}
Let $c_\infty = \lim_{k \to \infty} c_k$, which exists and is finite because the sequence
$c_k$ is monotonically increasing and bounded above.
Furthermore, we have that for any $n \geq 2$,
\begin{align*}
    C^n \prod_{i=1}^{n-1} \frac{c_i}{e \log(c_i)} \leq \left(\frac{C
    c_\infty}{e} \right)^{n-1} \frac{C}{\log(\sum_{i=1}^{n-1} c_i)} \leq  \left(\frac{C
    c_\infty}{e} \right)^{n-1} \frac{C}{\log(c_1)} \:.
\end{align*}
Now choose any strictly increasing sequence such that $c_\infty = \Avg(1,
1/\rho) = (1/2)(1/\rho + 1)$ and $c_1 = \Avg(1, c_\infty) = (1/4)(3 +
1/\rho)$.
By the addition lemma (Lemma~\ref{lem:G1+G2}), the assumption on $C$, and a simple limiting argument,
\begin{align*}
  \sum_{n=0}^{\infty} \tf \Delta^n \in \RHinf\left( C' , c_\infty \rho \right) \:,
\end{align*}
where $C'$ is given as
\begin{align*}
    C' \leq 1 + C + \frac{C}{\log(c_1)} \frac{1}{1 - C c_\infty / e} \leq 1 +
    C + \frac{2C}{\log(c_1)} \:.
\end{align*}
The claim now follows by using the inequality $\log(1+x) \geq x/2$ for $x \in [0, 2.5]$
and the assumed bound $C \leq 2$.
\end{proof}

\begin{lemma}
Suppose that $\tf G \in \RHinf(C,\rho)$.  Then $\hinfnorm{\tf G}\leq\frac{C}{1-\rho}$.
\label{lem:hinfbnd}
\end{lemma}
\begin{proof}
We have that
\begin{align*}
    \hinfnorm{\tf G} &= \sup_{z \in \mathbb{T}} \norm{\tf G(z)} = \sup_{z \in \mathbb{T}} \bignorm{\sum_{k=0}^{\infty} G(k) z^{-k}} \leq C \sum_{k=0}^\infty \rho^k = \frac{C}{1-\rho}.
\end{align*}
\end{proof}

Next, a probabilistic lemma which we use to control the LQR cost on a finite horizon.
\begin{lemma}
Let $x$ and $M$ be fixed, and $w \sim \calN(0,\Sigma)$, with $\Sigma \succ 0$ and $\|\Sigma\| = \sigma^2$.  Then there exists a universal constant $c>0$ such that with probability at least $1-\delta$
\begin{align}
\begin{bmatrix} x \\ w \end{bmatrix}^\T M \begin{bmatrix} x \\ w \end{bmatrix} & \leq x^\T M_{11} x + 2\sqrt{2}\sigma\|x\|\|M_{12}\|\sqrt{\logg{\tfrac{2}{\delta}}} \nonumber \\ &\quad + \Tr M_{22}\Sigma + c\sigma^2\|M_{22}\|_F\sqrt{\logg{\tfrac{2}{\delta}}} + c\sigma^2\|M_{22}\|\logg{\tfrac{2}{\delta}}.
\label{eq:quadconc}
\end{align}
\label{lem:quadconc}
\end{lemma}
\begin{proof}
Expanding the quadratic we have
\[
\begin{bmatrix} x \\ w \end{bmatrix}^\T M \begin{bmatrix} x \\ w \end{bmatrix} = x^\T M_{11}x + 2x^\T M_{12}w + w^\T M_{22} w.
\]

Noting that $x^\T M_{12} w \sim \calN(0,x^\T M_{12} \Sigma M_{12}^\T x)$, by standard Gaussian concentration we have with probability at least $1-\frac{\delta}{2}$ that
\begin{align*}
x^\T M_{12} w  & \leq \sqrt{2 x^\T M_{12} \Sigma M_{12}^\T x\logg{\tfrac{2}{\delta}}} \\
& \leq \sqrt{2}\norm{x}\norm{M_{12}}\norm{\Sigma}^\frac{1}{2}\sqrt{\logg{\tfrac{2}{\delta}}} \\
& = \sqrt{2}\norm{x}\sigma\norm{M_{12}} \sqrt{\logg{\tfrac{2}{\delta}}}.
\end{align*}

On the other hand, by the Hanson-Wright inequality~\cite{rudelson13}, we have that with probability at least $1 - \tfrac{\delta}{2}$ that
\begin{align*}
w^\T M_{22} w & \leq \Tr M_{22}\Sigma + c \sqrt{\norm{\Sigma^\frac{1}{2}M_{22}\Sigma^\frac{1}{2}}^2_F\logg{\tfrac{2}{\delta}}} + c\norm{\Sigma^\frac{1}{2}M_{22}\Sigma^\frac{1}{2}}^2\logg{\tfrac{2}{\delta}} \\
& \leq \Tr M_{22}\Sigma  + c \sigma^2\norm{M_{22}}_F\sqrt{\logg{\tfrac{2}{\delta}}} + c\sigma^2 \norm{M_{22}}\logg{\tfrac{2}{\delta}}.
\end{align*}

\end{proof}

\begin{lemma}
\label{lem:lambda_min_block_matrix}
Let $\Sigma$ be a $\statedim \times \statedim$ positive-definite matrix and let $K$ be a real $\inputdim \times \statedim $ matrix.
Then, for any $\sigma_u \in \R$ we have that
\begin{align*}
  \lambda_{\min}\left(\bmattwo{\Sigma}{\Sigma K^\T}{K \Sigma}{K \Sigma K^\T + \sigma_u^2 I}\right) &\geq
  \sigma_u^2 \min\left(\frac{1}{2}, \frac{\lambda_{\min}(\Sigma)}{2 \ltwonorm{K \Sigma K^\T} + \sigma_u^2} \right) \:.
\end{align*}
\end{lemma}
\begin{proof}
We find $0 < \gamma_1 < 1$ and $\gamma_2 > 0$ such that the following condition holds
\begin{align*}
  \bmattwo{\Sigma}{\Sigma K^\T}{K \Sigma}{K \Sigma K^\T + \sigma_u^2 I} \succeq \bmattwo{\gamma_1 \Sigma}{0}{0}{\gamma_2 I} \:.
\end{align*}
By Schur complements, this condition is equivalent to
\begin{align*}
  0 &\preceq K \Sigma K^\T + (\sigma_u^2 - \gamma_2) I - K \Sigma ((1-\gamma_1)\Sigma)^{-1} \Sigma K^\T \\
  &= - \frac{\gamma_1}{1-\gamma_1} K\Sigma K^\T + (\sigma_u^2 - \gamma_2) I \:.
\end{align*}
Now set $\gamma_2 = \sigma_u^2/2$ and $\gamma_1 = \frac{\sigma_u^2}{2 \ltwonorm{K \Sigma K^\T} + \sigma_u^2}$.
\end{proof}


\section{Implementation of Adaptive Methods}
We consider several adaptive methods for numerical comparison. This section described the relevant implementation details.

\label{sec:appendix:implementation}

\subsection{Optimism in the Face of Uncertainty} \label{sec:app:ofu}

At the start of each epoch, the OFU method computes a confidence set around the dynamics and then finds the $(A,B)$ that would achieve the smallest LQR cost. The method then plays the associated optimal controller.

The confidence sets at epoch $i$ are of the form
\begin{align} \label{eq:app:ofu_c}
  \begin{split}
  C_i(\varepsilon) = \{ \Theta \in \R^{n \times (n+p)} : \Tr( (\Theta - \widehat{\Theta}_i) Z_{T_i} (\Theta - \widehat{\Theta}_i)^\T ) \leq \varepsilon \} \:,\\
   Z_{T_i} = \lambda I + \sum_{i=1}^{T_i} \cvectwo{x_t}{u_t} \cvectwo{x_t}{u_t}^\T \:.
   \end{split}
\end{align}
Here, $\widehat{\Theta}_i$ denotes the (regularized) least squares estimate of the
true parameters $\Theta_\ast = (\trueA, \trueB)$.
For our experiments, we set $\lambda = 10^{-5}$
and $\varepsilon = \Tr((\widehat{\Theta}_i - \Theta_\ast) Z_{T_i} (\widehat{\Theta}_i - \Theta_\ast)^\T)$ using the true and estimation values of $(A,B)$.

Then controller is selected by finding the ``best'' dynamics. To be precise,
let $J(A, B) = \Tr(P(A, B))$, where $P(A, B)$ is the solution to the
discrete algebraic Riccati solution
\begin{align*}
  P = A^\T P A - A^\T P B(B^\T P B + R)^{-1} B^\T P A + Q \:.
\end{align*}
Then for every epoch of OFU, it is necessary to solve to the non-convex
optimization problem
\begin{align}
  [\widetilde{A}, \widetilde{B}] = \arg\min_{[A, B] \in C_i(\varepsilon)} J(A, B) \:. \label{eq:ofu_problem}
\end{align}
up to an absolute error of at most $O(1/\sqrt{T_i})$.

As in Section 5.4 of \cite{abbasi12}, we heuristically solve this optimization problem
using projected gradient descent (PGD).
An expression for the gradient of $\Theta \mapsto J(A, B)$ is derived in \cite{abbasi12}
(see also \cite{abeille17}) by use of the implicit function theorem.
Specifically, $\nabla_\Theta \Tr(P(A, B))$
evaluated at a point $\Theta = (A, B)$ is an $n \times (n+p)$ matrix $D$. The
$i,j$-th entry is given by $\Tr(E_{ij})$, where
$E_{ij}$ is the solution to the Lyapunov equation
\begin{align*}
  E_{ij} = A_c^\T E_{ij} A_c + 2\mathrm{Sym}\left(A_c^\T P(A, B) e_i e_j^\T \cvectwo{I}{K}\right) \:,
\end{align*}
with $K$ as the optimal LQR controller for $(A, B)$, $A_c = A + B K$, and $\mathrm{Sym}(A) = \frac{1}{2}(A + A^\T)$.
Finally, the projection of $\Theta$ onto the set $C_i(\varepsilon)$
can be solved by a eigendecomposition of $Z_{T_i}$ followed by a
scalar root-finding search. The details of this are also found in Section 5.4 of \cite{abbasi12}.

We determine the end of an epoch using a switching rule based on a slight modification of the determinant condition of \cite{abbasi2011regret}.
We switch an epoch when both (a)
$T - T_i \geq 10$ and (b) $\det(Z_{T}) > 2 \det(Z_{T_i})$ hold.
The first condition is to ensure that the switches are not too frequent in the beginning of the
algorithm.

\subsection{Thompson Sampling}\label{sec:app:ts}

The Thompson sampling algorithm is nearly identical to the OFU algorithm, except the
optimization problem \eqref{eq:ofu_problem} is replaced by sampling.
While the description of Thompson sampling in the Bayesian setting of \cite{abbasi15} and \cite{ouyang17}
requires sampling from the posterior distribution, we follow the
more frequentist setting of \cite{abeille17} and sample a point $\widetilde{\Theta}$
uniformly at random from the confidence set $C_i(\varepsilon)$ as in~\eqref{eq:app:ofu_c}.

We implement this uniform samping by first drawing
a $U \sim \mathrm{Unif}([0, 1])$ and a
$\eta \in \R^{n \times (n+p)}$ with each $\eta_{ij} \sim \calN(0, 1)$,
and setting
\begin{align*}
  \widetilde{\Theta} = \widehat{\Theta} + \sqrt{\varepsilon} \left(\frac{U^{1/(n(n+p))}}{\norm{\eta}_F} \eta \right) Z_{T_i}^{-1/2} \:.
\end{align*}

For the epoch switching rule, we follow the suggestion of \cite{abeille17}
to force exploration after $\tau$ iterations, where we set $\tau = 500$.
Specifically, we switch an epoch when the following predicate holds:
\begin{align*}
  (T - T_i \geq \tau) \text{ or } ((T - T_i \geq 10) \text{ and } (\det(Z_T) > 2 \det(Z_{T_i}))) \:.
\end{align*}

\subsection{Robust Adaptive Control with FIR truncation} \label{sec:app:sls_fir}

We now describe how to turn the infinite-dimensional optimization problem in Algorithm~\ref{alg:adaptive}
into a finite-dimensional problem. First, recall the problem we want to solve,
\begin{align}
  \mathrm{minimize}_{\gamma \in [0, 1)} &\frac{1}{1-\gamma} \min_{\tf \Phi_x, \tf \Phi_u, V} \bightwonorm{ \begin{bmatrix} Q^{1/2} & 0 \\ 0 & R^{1/2} \end{bmatrix} \begin{bmatrix} \tf \Phi_x \\ \tf \Phi_u \end{bmatrix} }  \nonumber \\
    \qquad \text{s.t.} &\rvectwo{zI - \Ah}{-\Bh} \cvectwo{\tf \Phi_x}{\tf \Phi_u} = I + \frac{1}{z^{F}}V \:, \:\: \frac{\sqrt{2}\varepsilon}{1-C_x\rho^{F+1}}\bighinfnorm{ \cvectwo{ \bmtx{\Phi_x}}{ \bmtx{\Phi_u}}} \leq  \gamma \:, \label{eq:opt_problem_fir} \\
    & \norm{V} \leq C_x\rho^{F+1} \:, \:\: \tf \Phi_x \in \frac{1}{z}\RHinf^{F}(C_x, \rho) \:, \:\: \tf \Phi_u \in \frac{1}{z}\RHinf^{F}(C_u, \rho) \nonumber \:.
\end{align}

Ignoring the outer minimization over $\gamma$ (which can be solved with bisection), the inner minimization is convex.
Truncating the system responses to be FIR of length $F$ means that
\begin{align*}
  \tf\Phi_x = \sum_{k=1}^{F} \Phi_x(k) z^{-k} \:, \:\: \tf\Phi_u = \sum_{k=1}^{F} \Phi_u(k) z^{-k} \:.
\end{align*}
All pieces of the infinite dimensional problem can be written in terms of these variables. First, consider the $\htwo$ cost in the objective.
By Parseval's identity, we can simply add the second order cone constraint
\begin{align} \label{eq:cost_constr_fir}
  \bignorm{\begin{bmatrix}
    Q^{1/2} \Phi_x(1) \\
    \vdots \\
    Q^{1/2} \Phi_x(F) \\
    R^{1/2} \Phi_u(1) \\
    \vdots \\
    R^{1/2} \Phi_u(F) \\
  \end{bmatrix}}_F \leq t \:,
\end{align}
and minimize $t$.
Next, we consider the constraints of the original optimization. The function space constraints reduce to the requirement that
\begin{align} \label{eq:fn_space_constr_fir}
  \norm{\Phi_x(k)} \leq C_x \rho^k \:, \:\: \norm{\Phi_u(k)} \leq C_u \rho^k \:, \:\: k=1, ..., F \:.
\end{align}
Next, to rewrite the subspace constraint, we first consider that
\begin{align*}
  z \tf\Phi_x = \sum_{k=0}^{F-1} \Phi_x(k+1) z^{-k} \:,
\end{align*}
then the subspace constraint yields the following equality constraints,
\begin{align} \label{eq:subspace_constr_fir}
\begin{split}
  \Phi_x(1) &= I \:, \\
   \Phi_x(k+1) &= \Ah \Phi_x(k) + \Bh \Phi_u(k) \:, ~~k=1, ..., F-1 \:, \\
   V &= \Ah \Phi_x(F) + \Bh\Phi_u(F) \:.
\end{split}
\end{align}

The only constraint that remains is the $\hinf$ constraint,
for which we use the following result.
\begin{theorem}[Theorem 5.8, \cite{dumitrescu2007positive}]
\label{thm:hinf_constraint}
Consider the $T$-length FIR filter
\begin{align*}
  \tf H(z) = \sum_{k=0}^{T} H_k z^{-k} \:, H_k \in \R^{p \times m} \:.
\end{align*}
Define the matrix
\begin{align*}
  \overline{H} = \begin{bmatrix} H_0 \\ \vdots \\ H_T \end{bmatrix} \in \R^{p(T+1) \times m} \:.
\end{align*}
We have that $\norm{\tf H(z)}_{\hinf} \leq \gamma$ iff there exists $Q = Q^\T \succeq 0$
with $Q \in \R^{p(T+1) \times p(T+1)}$ satisfying
\begin{align*}
  Q = \begin{bmatrix}
    Q_{00} & Q_{01} & ... & Q_{0T} \\
    \ast & Q_{11} & ... & Q_{1T} \\
    \ast & \ast & \ddots & \vdots \\
    \ast & \ast & \ast & Q_{TT}
  \end{bmatrix} \:, \:\: Q_{ij} \in \R^{p \times p} \:,
\end{align*}
\begin{align*}
  \sum_{t=0}^{T} Q_{tt} &= \gamma^2 I_p \:, \:\:
  \sum_{t=0}^{T-k} Q_{t(t+k)} = 0_{p \times p} \:, k=1, ..., T \:, \:\:
  \bmattwo{Q}{\overline{H}}{\overline{H}^\T}{I_{m}} \succeq 0 \:.
\end{align*}
\end{theorem}

For the SLS problem, the $\hinf$ constraint on is the filter
\begin{align*}
  \tf H(z) = \sum_{k=1}^{F} \cvectwo{ \Phi_x(k) }{ \Phi_u(k) } z^{-k} \:.
\end{align*}
The constraint can be rewritten using the LMI in Theorem~\ref{thm:hinf_constraint}.
To avoid a decision variable of size $(n+p)(F+1) \times (n+p)(F+1)$, we instead consider the transpose system $\tf H^\T$ which has the same $\hinf$ norm and coefficients of size $n \times (n+p)$.

Putting this together, we arrive at the following SDP, which
can be solved using an off the shelf solver,
\begin{align*}
  \min_{
  \substack{\tf\Phi_x[k] \in \R^{n \times n},~
            \tf\Phi_u[k] \in \R^{p \times n},~V \in \R^{n \times n}\\
             P \in \R^{n(F+1)\times n(F+1)},~
             t \in \R }}~~ t
  \\
  \text{s.t.} \qquad&
  \eqref{eq:fn_space_constr_fir}\:, ~
  \eqref{eq:subspace_constr_fir}\:, ~
  \eqref{eq:cost_constr_fir}\:, \\
   ~~
  &\sum_{t=0}^{F} P_{tt} = \gamma^2 I \:, \:\: \sum_{t=0}^{F-k} P_{t(t+k)} = 0 \:, \:\: k=1, ..., F \:, \\
  &\overline{H} = \frac{\sqrt{2}\varepsilon}{1 - C_x \rho^{F+1}}\begin{bmatrix}
    0_{n \times n} & 0_{n \times p} \\
    \Phi_x(1)^\T & \Phi_u(1)^\T \\
    \vdots & \vdots \\
    \Phi_x(F)^\T & \Phi_u(F)^\T
  \end{bmatrix} \:, \:\: \bmattwo{P}{\overline{H}}{\overline{H}^\T}{I_{m}} \succeq 0 \:, \\
  &\norm{V} \leq C_x \rho^{F+1} \:.
\end{align*}
For our experiments, we used the SCS solver~\cite{odonoghue16} via CVXPY~\cite{diamond16}.

Finally, once the FIR responses $\{ \Phi_x(k) \}_{k=1}^{F}$ and
$\{ \Phi_u(k) \}_{k=1}^{F}$ are found, we need a way to implement the system responses as a controller.
We represent the dynamic controller $\tf K = \tf \Phi_u \tf \Phi_x^{-1}$ by finding
an equivalent state-space realization $(A_K, B_K, C_K, D_K)$ via Theorem 2 of \cite{anderson17}.

As a final note, the adaptive method as described in Algorithm~\ref{alg:adaptive} requires several constants to be specified. For the numerical experiments, we set
$\sigma_{\eta,i} = C_\eta \sigma_w T_i^{-1/3}$ where we vary $C_\eta$ for different experiments, fix $\gamma=0.98$, and use a fixed FIR trunction length of $F=12$.
For the experiments in Section~\ref{sec:experiments}, we set $C_\eta = 0.1$.


\section{Additional Experiments}

\subsection{Large-Transient Dynamics} \label{sec:app:laplacian_exp}

We present the regret comparison results using another system
\begin{align} \label{eq:exampledynamics_unstable}
\trueA  = \begin{bmatrix} 2 & 0 & 0\\
4 & 2 & 0\\
0 & 4 & 2\end{bmatrix}, ~~ \trueB = I, ~~ Q = 10 I, ~~ R =I \:.
\end{align}
The system is both unstable and has large transients.
Each state receives direct input, and the cost is such that input size is penalized relatively less than state.
This problem setting is amenable to robust methods due to both the cost ratio and the large transients, which are factors that may hurt optimistic methods.
For this experiment, we ran all adaptive methods as described in
Appendix~\ref{sec:appendix:implementation}, and used an initialization with a
horizon of length $T_0=250$ and $C_\eta = 2$.

\begin{figure} 
\centering
\begin{subfigure}[b]{\basefigwidth\textwidth}
\caption{\small Regret}
\centerline{\includegraphics[width=\columnwidth]{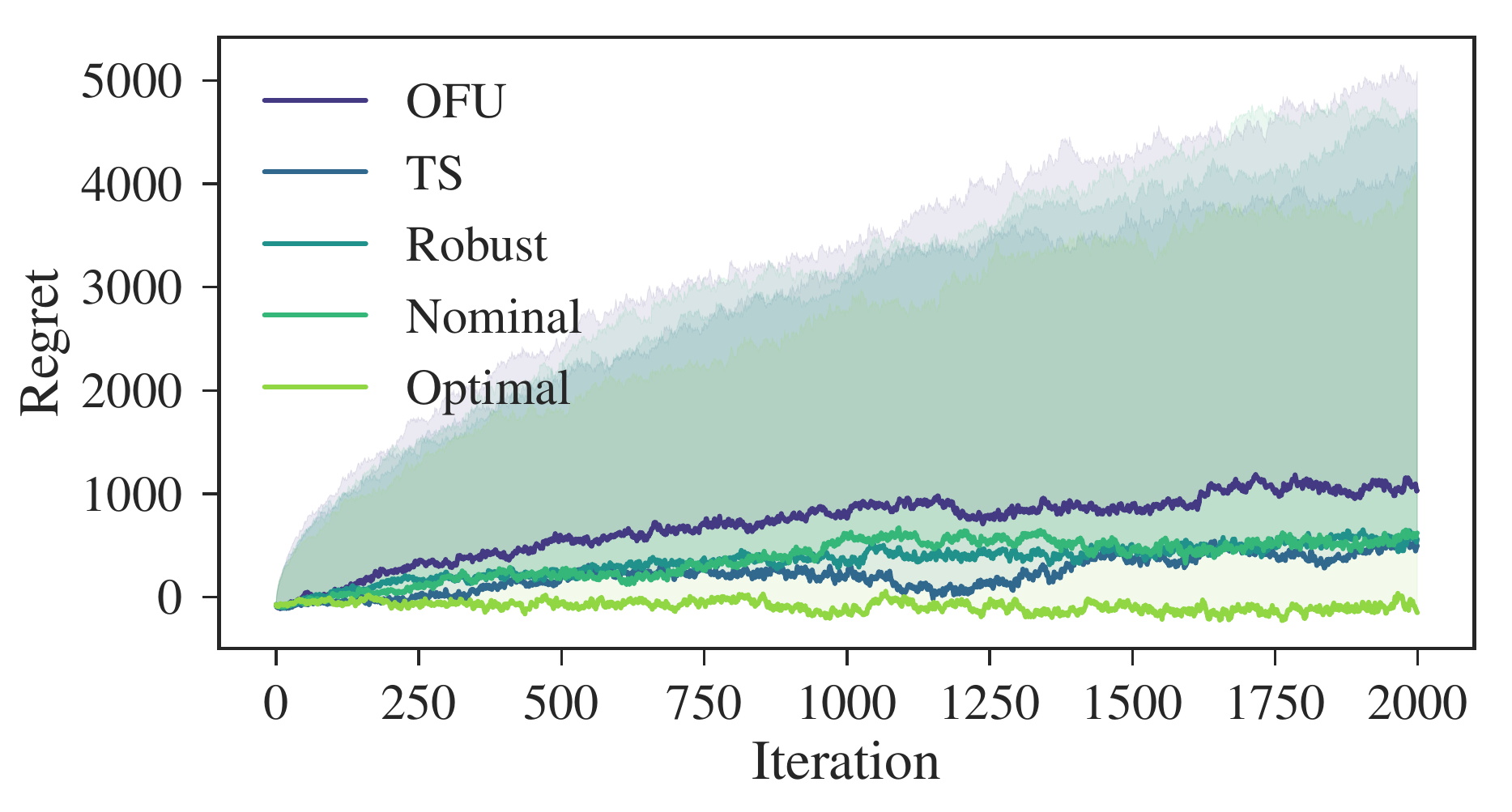}}
\end{subfigure}
\begin{subfigure}[b]{\basefigwidth\textwidth}
\caption{\small Infinite Horizon LQR Cost}
\centerline{\includegraphics[width=\columnwidth]{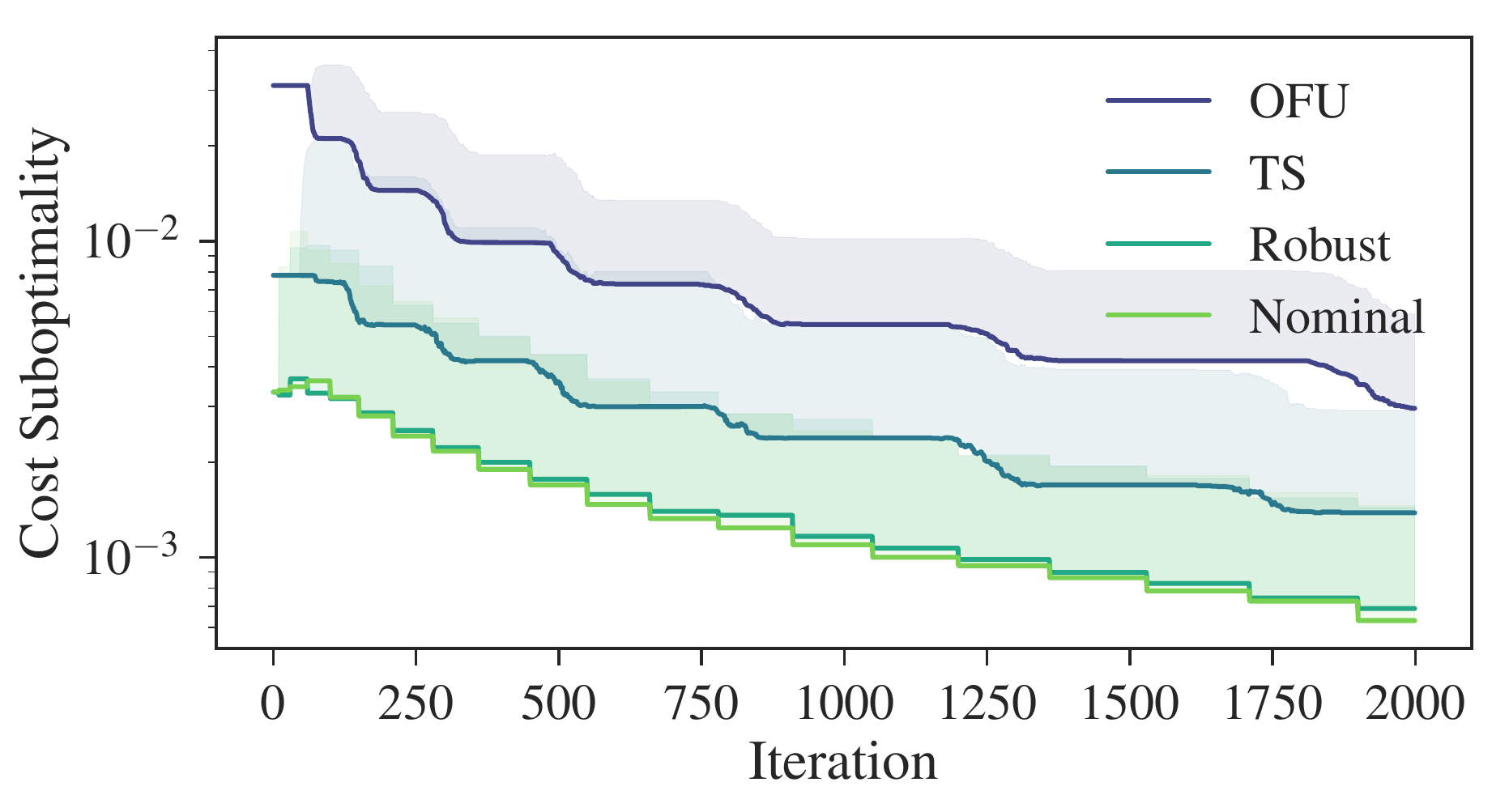}}
\end{subfigure}
\caption{\small A comparison of different adaptive methods on 500 experiments of the large-transient system example~\eqref{eq:exampledynamics_unstable}. In (a), the median and 90th percentile regret is plotted over time. In (b), the median and 90th percentile infinite-horizon LQR cost of the epoch's controller.}
\label{fig:regrets_laplace}
\end{figure}

The performance of the various adaptive methods is compared in Figure~\ref{fig:regrets_laplace}. The median and 90th percentile regret over 500 instances is displayed in  Figure~\ref{fig:regrets_laplace}a, which gives an idea of both ``average'' and worst-case behavior.
Overall, the methods have very similar performance. One benefit of robustness is the guaranteed stability and therefore bounded infinite-horizon cost at every point during operation.
In  Figure~\ref{fig:regrets_laplace}b, this infinite-horizon cost of the controller in each epoch is plotted.
This measures the cost of using each epoch's controller indefinitely, rather than continuing to update its parameters.
Especially for small numbers of iterations, the robust method performs relatively better than other adaptive algorithms, indicating that it is more amenable to early stopping.

\subsection{Error Scaling} \label{sec:app:error_scaling}

In our experiments, we use the actual estimation errors for controller synthesis. To examine the effect of this choice, we artificially inflate the estimation errors by various multipliers, and plot the regret for various methods in Figure~\ref{fig:regrets_err_mult}. These experiments were run on the the graph Laplacian example in~\eqref{eq:exampledynamics_laplacian} with an initialization with a horizon of length $T_0=300$ and $C_\eta = 1$.

The adaptive methods were run as described in Appendix~\ref{sec:appendix:implementation}.
The error term $\varepsilon$ for OFU and TS appears in the computation of the uncertainty set as in \eqref{eq:app:ofu_c}. The errors $\varepsilon_A$ and $\varepsilon_B$ for the robust adaptive method appear in \eqref{eq:opt_problem_fir}. The plot shows a modest degradation in regret as these terms are increased.

\begin{figure} 
\centering
\begin{subfigure}[b]{\basethreefigwidth\textwidth}
\caption{\small OFU  }
\centerline{\includegraphics[width=\columnwidth]{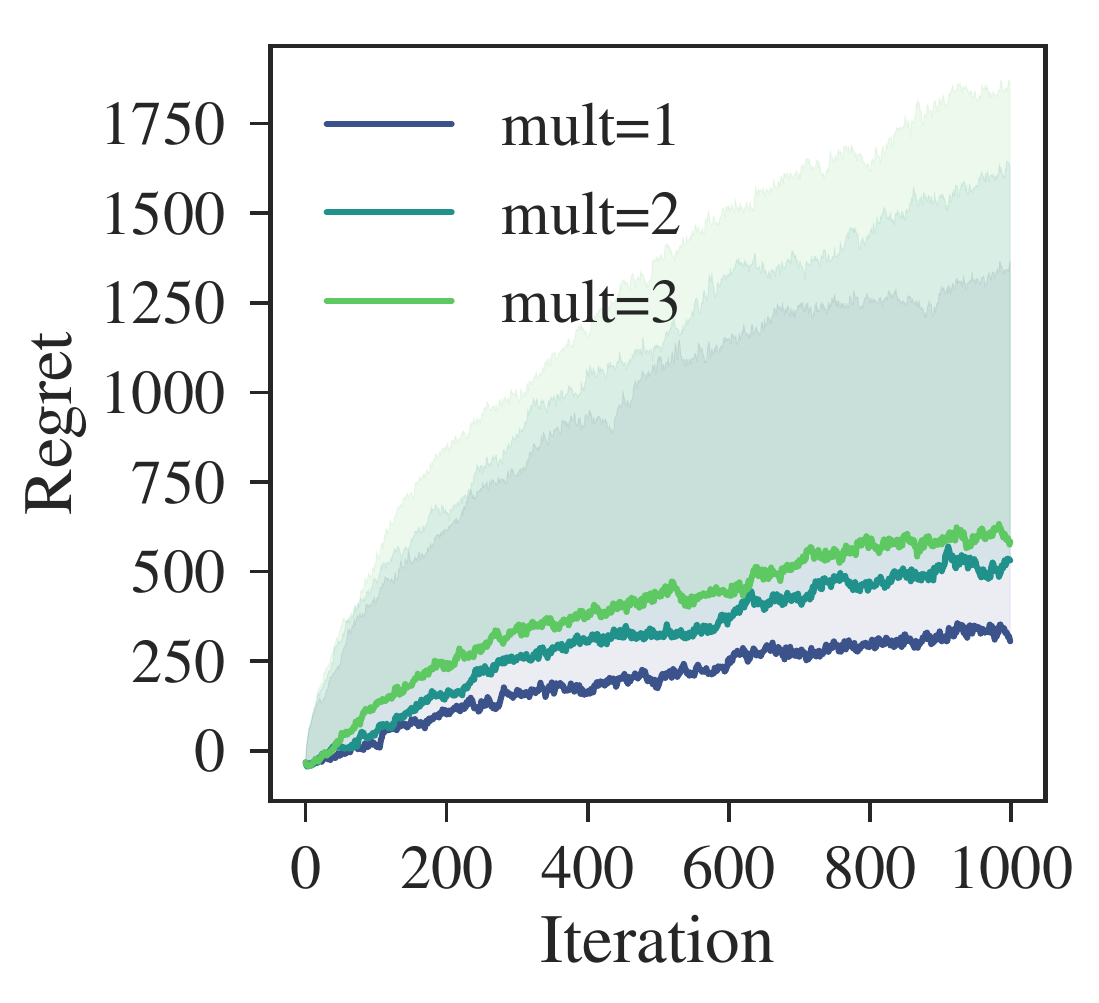}}
\end{subfigure}
\begin{subfigure}[b]{\basethreefigwidth\textwidth}
\caption{\small TS }
\centerline{\includegraphics[width=\columnwidth]{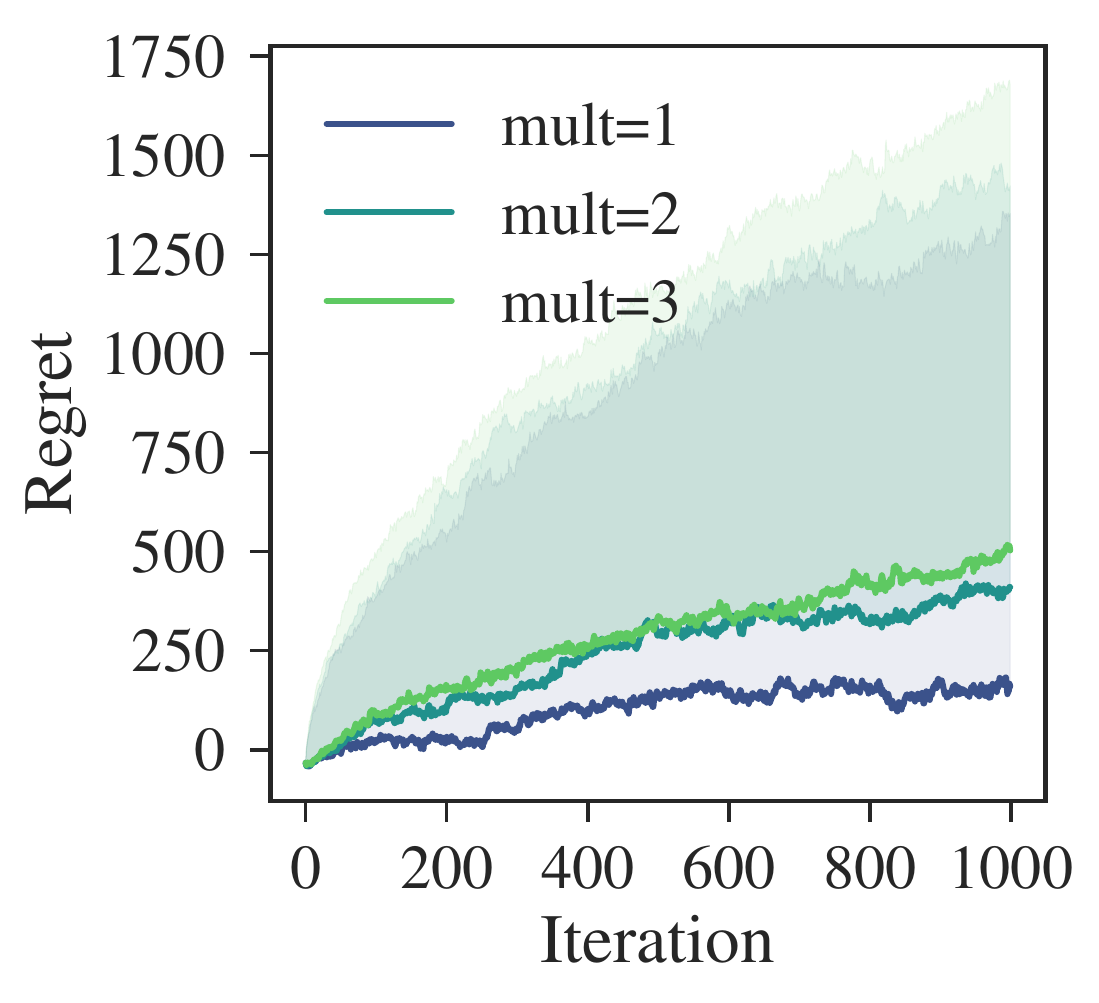}}
\end{subfigure}
\begin{subfigure}[b]{\basethreefigwidth\textwidth}
\caption{\small Robust  }
\centerline{\includegraphics[width=\columnwidth]{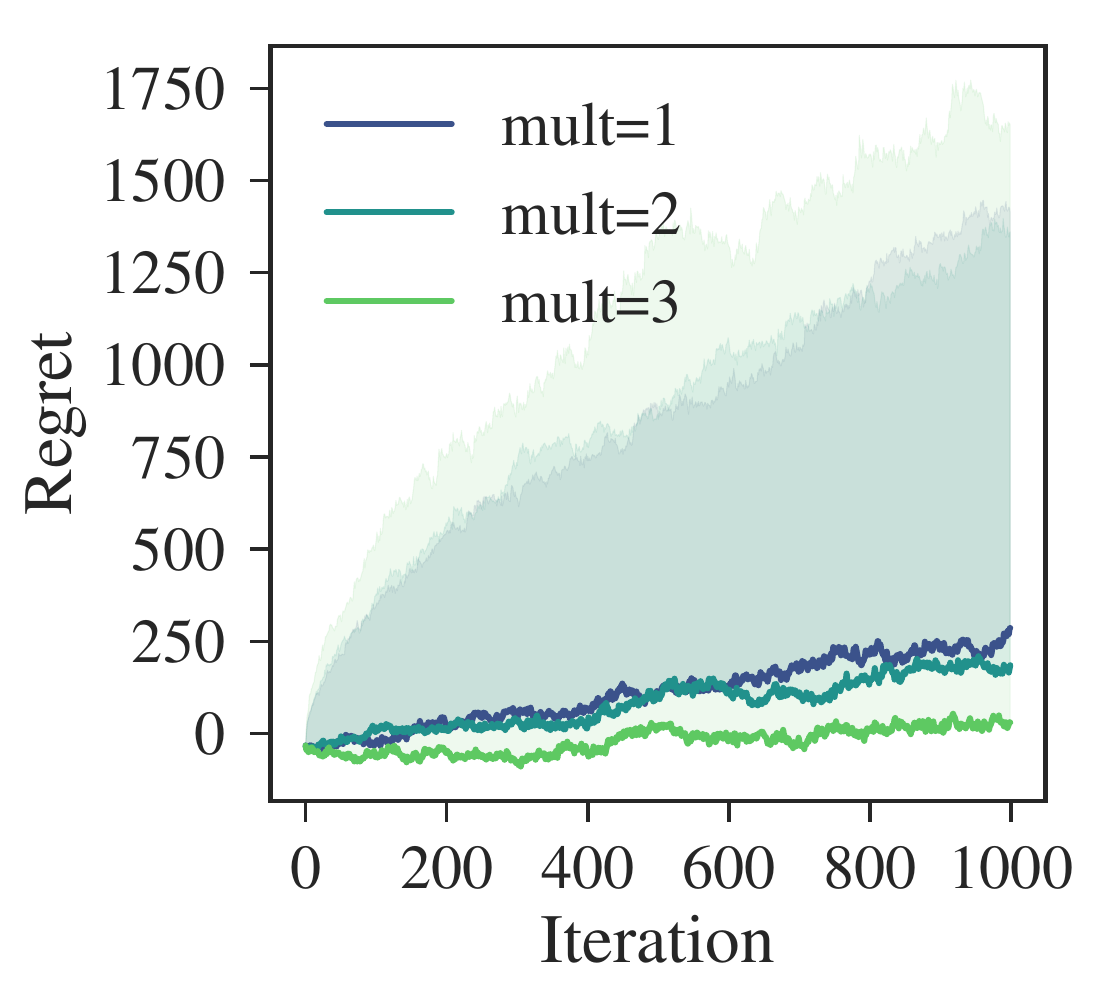}}
\end{subfigure}
\caption{\small A comparison of regret when enlarged error bounds are used for synthesis, rather than the true errors. Both the median over 500 trials and the 90th percentile regret are plotted. In (a) is OFU, in (b) is TS, in (c) is robust. The plots show modest if any degradation in performance.}
\label{fig:regrets_err_mult}
\end{figure}

\subsection{Learning the Disturbance Process} \label{sec:app:demand_ex}

We consider the problem of regulating a known system which is subject to disturbances correlated in time. These disturbances are modeled as the output of a LTI filter driven by white noise. In other words,
\[x_{k+1} = \trueA x_k + \trueB u_k + d_k \:, \qquad d_{k+1} = A_d d_k + w_k, \]
where $x_k$ is the state to drive to zero, and $d_k$ are the disturbances. We will take $(\trueA,\trueB)$ to be known and $A_d$ unknown.
This setting models many phenomenon related to demand forecasting, in which the dynamics of e.g. a server farm is known, and the changes in demand are stochastic but correlated in time, and can thus be approximated by the output of an LTI filter.

The plant inputs $u_k$ are designed for regulation. The controller design problem can be formulated as an optimization problem by defining the augmented system as
\begin{align} \label{eq:augmented_system}
\begin{bmatrix} x_{k+1} \\ d_{k+1} \end{bmatrix} =
\begin{bmatrix} \trueA & I \\ 0  & A_d \end{bmatrix} \begin{bmatrix} x_{k} \\ d_{k} \end{bmatrix} +
\begin{bmatrix} \trueB \\ 0   \end{bmatrix}  u_{k}
+ \begin{bmatrix} 0 \\ I \end{bmatrix} w_k\:.
\end{align}
We will denote the augmented state $z_k= \begin{bmatrix}x_k; d_k\end{bmatrix}$. Then the control actions can be designed using an adaptive LQR strategy.
In many situations, inputs are relatively more costly, corresponding for example to energy usage. Defining an LQR cost directly related to the economics of the system can be unwise, due to the resulting tendency for states to become large, which may correspond to unsafe execution. While tuning the quadratic cost to represent a mixture of economic and safety considerations can often achieve good behavior in practice, the method is heuristic and lacks guarantees. Instead, consider the explicit addition of a constraint on the state, $\|x_k\|_\infty \leq a$ for $0\leq k \leq H$ for some horizon (which may be infinite).

To state the necessary modification to the controller synthesis problem, we define the norm
\begin{align*}
  \lonenorm{\tf M} = \sup_{\infnorm{\tf w}=1} \: \infnorm{\tf M\tf w} \:,
\end{align*}
for both system responses and state matrices. This norm corresponds to the $\ell_\infty \mapsto \ell_\infty$ operator norm.

\begin{proposition}
For the system described in~\eqref{eq:augmented_system}, let $\tf\Phi_z$ denote a closed-loop state response. Then consider constraints
\begin{align}\label{eq:sys_response_constraint}
\begin{split}
 \|(\tf\Phi_z)_{22}\|_{\mathcal{L}_1} &\leq \gamma / \tilde\varepsilon_A\:,\\
 \|(\tf\Phi_z)_{12}\|_{\mathcal{L}_1} &\leq \frac{a}{b}\cdot (1-\gamma) := c
 \end{split}
\end{align}
where $(\tf\Phi_z)_{ij}$ denotes the blocks defined by the partition of $z_t$ into $x_t$ and $d_t$, and
 $\|\Ahat_d-A_d\|_{\mathcal{L}_1}\leq\tilde\varepsilon_A$.
The addition of these constraints to the synthesis problem in~\eqref{eq:opt_problem_fir} ensures that the resulting closed loop system has $\|x_k\|_\infty\leq a$ for $0\leq k \leq H$ as long as $\|w_k\|_\infty\leq b$ for $0\leq k \leq H$.
\end{proposition}
\begin{proof}
In transfer function notation, the state of the plant can be described by
\begin{align*}
\tf x = \begin{bmatrix} I & 0 \end{bmatrix} \tf z = \begin{bmatrix} I & 0 \end{bmatrix} \tf \Phi_z (I+\tf {\widehat\Delta})^{-1} \begin{bmatrix} 0 \\ I \end{bmatrix} \tf w\:.
\end{align*}
Furthermore, due to the known structure of the dynamics,
\begin{align*}
(I+\tf {\widehat\Delta})^{-1} = \left(I + \begin{bmatrix} 0 & 0 \\ 0 & \Delta_A \end{bmatrix}\tf\Phi_z\right)^{-1} = \begin{bmatrix} I & 0 \\ X & (I+\Delta_A(\tf\Phi_z)_{22})^{-1} \end{bmatrix}\:,
\end{align*}
where $X=(I+\Delta_A(\tf\Phi_z)_{22})^{-1}\Delta_A(\tf\Phi_z)_{21}$.
Then we have, letting $\tf {\widehat\Delta}_{22}=\Delta_A(\tf\Phi_z)_{22}$,
\begin{align*}
\tf x = \begin{bmatrix} I & 0 \end{bmatrix} \tf \Phi_z \begin{bmatrix} 0 \\ (I+\tf {\widehat\Delta}_{22})^{-1} \end{bmatrix}  \tf w = (\tf \Phi_z)_{12} (I+\tf {\widehat\Delta}_{22})^{-1} \tf w\:.
\end{align*}
Finally, to bound the size of the state,
\begin{align*}
\|\tf x\|_\infty &\leq \|(\tf \Phi_z)_{12} (I+\tf {\widehat\Delta}_{22})^{-1}\|_{\mathcal{L}_1} \|\tf w\|_\infty \leq \frac{1}{1- \|\tf {\widehat\Delta}_{22}\|_{\mathcal{L}_1}} \|(\tf \Phi_z)_{12}\|_{\mathcal{L}_1} \|\tf w\|_\infty\:.
\end{align*}
Then we have that $\|\tf {\widehat\Delta}_{22}\|_{\mathcal{L}_1} \leq \tilde\varepsilon_A\|(\tf\Phi_z)_{22}\|_{\mathcal{L}_1}$, so the result follows from the constraints and the assumption on $w_k$.
\end{proof}
Therefore, with either a bounded noise assumption on $w_k$ or a high-probability bound over a finite time horizon, we can apply the previous result to synthesize safe controllers. In the example displayed in Figure~\ref{fig:constraints}, the constraint as in~\eqref{eq:sys_response_constraint} is added to the controller synthesis procedure with $c=0.1$ and $\gamma=0.98$.


\end{document}